\documentclass[lettersize,journal]{IEEEtran}
\usepackage{amsmath,amsfonts}
\usepackage{algorithmic}
\usepackage{algorithm}
\usepackage{array}
\usepackage[caption=false,font=normalsize,labelfont=sf,textfont=sf]{subfig}
\usepackage{textcomp}
\usepackage{stfloats}
\usepackage{hyperref}
\usepackage{verbatim}
\usepackage{graphicx}
\usepackage{cite}
\usepackage{tikz}
\usepackage{makecell}
\usepackage{color}
\newtheorem{lemma}{Lemma}
\newtheorem{proof}{Proof}
\begin{document}

\title{GMOR: A Lightweight Robust Point Cloud Registration Framework via Geometric Maximum Overlapping}

\author{Zhao Zheng, Jingfan Fan, Long Shao, Hong Song, Danni Ai, Tianyu Fu, Deqiang Xiao, Yongtian Wang, and Jian Yang
  \thanks{This work was supported by the National Science Foundation Program of China (62025104, 62422102, 62331005, U22A2052), the National Key R\&D Program of China (2023YFC2415300), and Haidian Original Innovation Joint Fund of Beijing Natural Science Foundation (L242091). (Corresponding Authors: Jingfan Fan; Long Shao; Jian Yang.)}
  \thanks{Zhao Zheng, Jingfan Fan, Long Shao, Danni Ai, Deqiang Xiao, Yongtian Wang, and Jian Yang are with Beijing Engineering Research Center of Mixed Reality and Advanced Display, School of Optics and Photonics, Beijing Institute of Technology, Beijing 100081, China, and Zhengzhou Research Institute, Beijing Institute of Technology, Zhengzhou 450003, China (e-mail:fjf@bit.edu.cn; lshao@bit.edu.cn; jyang@bit.edu.cn).}
  \thanks{Hong Song is with the School of Computer Science and Technology, Beijing Institute of Technology, Beijing 100081, China.}
  \thanks{Tianyu Fu is with the School of Medical Technology, Beijing Institute of Technology, Beijing 100081, China.}
  \thanks{Code released at \href{https://github.com/Bitzhaozheng/GMOR}{https://github.com/Bitzhaozheng/GMOR}.}}

\markboth{Journal of \LaTeX\ Class Files,~Vol.~14, No.~8, August~2021}%
{Shell \MakeLowercase{\textit{et al.}}: A Sample Article Using IEEEtran.cls for IEEE Journals}

\IEEEpubid{0000--0000/00\$00.00~\copyright~2021 IEEE}

\maketitle

\begin{abstract}
  Point cloud registration based on correspondences computes the rigid transformation that maximizes the number of inliers constrained within the noise threshold.
  Current state-of-the-art (SOTA) methods employing spatial compatibility graphs or branch-and-bound (BnB) search mainly focus on registration under high outlier ratios.
  However, graph-based methods require at least quadratic space and time complexity for graph construction, while multi-stage BnB search methods often suffer from inaccuracy due to local optima between decomposed stages.
  This paper proposes a geometric maximum overlapping registration framework via rotation-only BnB search.
  The rigid transformation is decomposed using Chasles' theorem into a translation along rotation axis and a 2D rigid transformation.
  The optimal rotation axis and angle are searched via BnB, with residual parameters formulated as range maximum query (RMQ) problems.
  Firstly, the top-k candidate rotation axes are searched within a hemisphere parameterized by cube mapping, and the translation along each axis is estimated through interval stabbing of the correspondences projected onto that axis.
  Secondly, the 2D registration is relaxed to 1D rotation angle search with 2D RMQ of geometric overlapping for axis-aligned rectangles, which is solved deterministically in polynomial time using sweep line algorithm with segment tree.
  Experimental results on indoor 3DMatch/3DLoMatch scanning and outdoor KITTI LiDAR datasets demonstrate superior accuracy and efficiency over SOTA methods, while the time complexity is polynomial and the space complexity increases linearly with the number of points, even in the worst case.
\end{abstract}


\section{Introduction}
\IEEEPARstart{P}{oint} cloud registration estimates the rigid transformation with 6 degrees of freedom (DoF) between two point clouds, which is a fundamental problem in simultaneous localization and mapping (SLAM)~\cite{baiFasterLIO2022,ruanSLAMesh2023,wuVoxelMap2024,mengFast2024}, computer assisted surgery~\cite{liincremental2023,zhuDevelopment2025} and augmented reality-based navigation~\cite{shaoAugmented2022,qinCollaborative2024,wangRobotic2024,yaoWearable2026}.
Although iterative closest point (ICP)~\cite{beslmethod1992} and its variants~\cite{pavlovAAICP2018,koideVoxelized2021,vizzoKISSICP2023,tunaXICP2024} are widely used for fine registration with given initial alignment, they require robust initialization by employing a global registration method.

\IEEEpubidadjcol
Global registration methods estimate the rigid transformation between point clouds by establishing correspondences without initial guesses.
Conventional correspondence-based approaches rely on handcrafted geometric primitives, such as 4-points congruent sets~\cite{aiger4points2008,melladoSuper2014}, convex hull~\cite{fan3Points2016,fanConvex2017} and dense local feature descriptors~\cite{rusuFast2009,sipiranHarris2011}.
Although these methods can generate plausible correspondence sets, they usually suffer from a high outlier ratio, which requires robust filtering.
To address this challenge, various robust global registration methods have been developed.
Random sample consensus (RANSAC)~\cite{fischlerRandom1981} and its variants~\cite{liPoint2021,shiRANSAC2024} remain widely used due to their simplicity, while graph-based matching methods~\cite{baiPointDSC2021,yangSACCOT2022,chenSC2PCR2022} apply spatial compatibility (SC) constraints to achieve strong outlier rejection.
Deterministic approaches based on branch-and-bound (BnB) search~\cite{yangGoICP2016,parrabustosGuaranteed2018} further provide global optimality guarantees in the searched solution space while suffer from rapidly growing search complexity in high-dimensional transformation space.

Within this context, recent deterministic BnB methods reduce the search complexity of rigid registration by decomposing the original 6 DoF into lower-dimensional subproblems~\cite{chenDeterministic2022,liEfficient2024,huangScalable2024,huangEfficient2024}.
For example, TR-DE~\cite{chenDeterministic2022} adopts a two-stage pure BnB strategy with 3+3 DoF decomposition.
Meanwhile, interval stabbing algorithms, which originally designed to maximize a uniform inlier count, are commonly used as 1-DoF reductions within the stages~\cite{parrabustosGuaranteed2018,liEfficient2024,huangScalable2024,huangEfficient2024}.
Chasles' decomposition theorem~\cite{chaslesNote1830} offers another practical way to shrink the search space and was first applied to gravity-constrained 4-DoF registration by Li et al.~\cite{liTransformation2024}.
This method considers a 1+2+1 DoF decomposition and solves the resulting subproblems by interval stabbing, 2D rotation center BnB search, and 1D rotation angle voting.
Despite their effectiveness, these multi-stage methods introduce two limitations:
the uniform inlier count treats all correspondences equally and do not incorporate correspondence weights, which reduces discrimination when correspondence confidence is highly non-uniform;
and sequentially decomposed stages may propagate local optima between stages.
Moreover, since interval stabbing collapses only a single DoF, it remains challenging to design a unified reduction that systematically guides BnB toward further dimension reduction.

To mitigate these limitations while maintaining efficient runtime and low memory consumption, we propose a Geometric Maximum Overlapping Registration (GMOR) framework based on a two-stage rotation-only BnB search.
The method leverages Chasles' decomposition theorem~\cite{chaslesNote1830,josephAlternative2020} to decouple rigid motion in the general 6-DoF case, as illustrated in Fig.~\ref{fig_chasles}.
Although the resulting DoF decomposition is similar in spirit to TR-DE~\cite{chenDeterministic2022} and the gravity-constrained formulation in~\cite{liTransformation2024}, our framework differs in its rotation-only search space and algorithmic structure.
Constrained settings, such as gravity-constrained 4-DoF registration~\cite{caiPractical2019,liTransformation2024}, naturally arise as special cases within this general framework.
In the first stage, the rotation axis is searched directly on the unit sphere using a cube-mapping hemispherical parameterization, and the translation along the axis is estimated via a 1D range maximum query (RMQ) over projected correspondences.
This RMQ adapts interval stabbing algorithms from prior works~\cite{parrabustosGuaranteed2018,huangScalable2024,huangEfficient2024,zhangAccelerating2024,liEfficient2024,liTransformation2024}, but is used here to maximize a weighted correspondence overlap objective rather than uniform inlier counting.
In the second stage, the residual rotation angle is searched in 1D space and the rotation axes with top-k weights in the first stage are selected to mitigate the local optima between stages.
Rotation centers induced by candidate angles are relaxed into axis-aligned bounding boxes, and maximum overlap estimation is formulated as a 2D RMQ problem over pairwise intersecting rectangles.
We further reformulate this geometric maximum overlapping as a stabbing counting query of Klee's rectangle problem~\cite{debergComputational2008,bentleySolution1977}, which can be solved deterministically in linear-logarithmic time using sweep line algorithm inspired by Bentley-Ottmann algorithm~\cite{bentleyAlgorithms1979} augmented with a segment tree.

This viewpoint provides a unified interpretation across both stages: casting interval stabbing as a solver of 1D RMQ problem also motivates the 2D RMQ formulation in the second stage, further reducing the BnB search dimension.
The proposed framework combines rotation-only BnB search over the Special Orthogonal Group ($\mathrm{SO}(3)$) with computational geometry-based RMQ solvers to achieve effective dimension reduction while preserving polynomial time complexity and linear space complexity.
The main contributions of this work are summarized as follows:

\begin{itemize}
  \item{We propose a novel two-stage rotation-only BnB framework for general 6-DoF rigid registration, which optimizes a weighted correspondence overlap objective rather than uniform inlier counting.}
  \item{A cube-mapping hemispherical parameterization is introduced for global rotation axis search, coupled with adaptive interval stabbing to estimate translation along the axis. The correspondence projections progressively tighten the search bounds, enabling efficient overlap maximization.}
  \item{We derive a deterministic polynomial-time solution for the residual 2D rigid registration by reformulating it as a 1D rotation angle search with 2D RMQ. The resulting geometric overlap maximization is cast as a Klee's rectangle problem~\cite{debergComputational2008} and solved via a sweep line algorithm with segment tree in $O(N\log N)$ time per iteration.}
\end{itemize}

\section{Related Work}

\subsection{Feature extraction and learning-based registration}

A common approach for global registration is extracting the corresponding features between two point clouds to estimate the transformation between them.
Rusu et al.~\cite{rusuFast2009} propose a Fast Point Feature Histograms (FPFH) descriptor for global registration,
which is a handcrafted feature capturing the local geometry of each point invariant to scene context.
Learning-based feature descriptors have been widely used for specific scenes, and 3DMatch proposed by Zeng et al.~\cite{zeng3DMatch2017} is a pioneering work which provides the first benchmark dataset for learned feature representations.
Choy et al.~\cite{choyFully2019} propose Fully Convolutional Geometric Features (FCGF),
which is trained separately on indoor/outdoor datasets and achieves higher inlier ratio.
Huang et al.~\cite{huangPREDATOR2021} demonstrate that using pairwise point clouds as joint input extracts richer features, proposing Predator for low-overlap cases.
Qin et al.~\cite{qinGeometric2022} further employ Transformer networks to model global geometric relationships of pairwise point clouds.
To address the positional deviation problem in predicted correspondences, Zhao et al.~\cite{zhaoRobust2025} propose a dual-branch network for matching and predicting patch-based features respectively.

Based on the correspondence matching in feature space, the rigid transformation is estimated by either RANSAC~\cite{fischlerRandom1981} or chunk-based methods~\cite{qinGeometric2022}.
Recently, end-to-end learning approaches, e.g. DGR~\cite{choyDeep2020} and PointDSC~\cite{baiPointDSC2021}, are progressive for solving the rigid transformation of correspondences.
While effective in specific scenarios~\cite{wangCCAG2024,yangLandmark2025}, they generalize poorly to practical applications.
Crucially, traditional methods~\cite{shiRANSAC2024,yangTEASER2021,chenSC2PCR2022,zhang3D2023,huangEfficient2024} without learning now surpass them in accuracy and efficiency through graph or solution space exploration.
These observations motivate a trend towards front-end and back-end decoupling~\cite{qiaoG3Reg2024,caoDMS2025,liuSGreg2025}, where the front end focuses on feature descriptor and correspondence generation, while the back end is responsible for reliable robust estimation.
Following this paradigm, this section concentrates on back-end global registration techniques.

\subsection{Iterative Random Sampling}

RANSAC~\cite{fischlerRandom1981} is a widely used algorithm for robust estimation with outliers, has inspired numerous variants~\cite{barathGraphCut2018,liPoint2021,shiRANSAC2024}.
Rusu et al.~\cite{rusuFast2009} leverage pairwise point distances and Fast Point Feature Histograms (FPFH) descriptor similarities to formulate SAmple Consensus Initial Alignment (SAC-IA), which suppresses outlier influence through finite iterations.
Barath and Matas~\cite{barathGraphCut2018} improve sampling efficiency with Graph-Cut RANSAC, removing outliers by graph cutting with local optimization.
Unlike local sampling approaches, Li et al.~\cite{liPoint2021} develop a one-point RANSAC variant using global pairwise distance and a modified Tukey's biweight function for scaled transformation estimation.
Shi et al.~\cite{shiRANSAC2024} extend this framework and propose TCF, implementing a three-stage RANSAC process for robust outlier rejection.

While these methods demonstrate strong performance under moderate outlier ratios and sufficient correspondences, their computational complexity becomes exponential when outlier ratio is extremely high.
To balance robustness and efficiency, RANSAC~\cite{fischlerRandom1981} is commonly integrated with other constraints, e.g. SC graph, as discussed in subsequent sections.

\subsection{Spatial Compatibility Graph Matching}

SC~\cite{baiPointDSC2021,yangSACCOT2022,chenSC2PCR2022} is a practical measure of the geometric consistency between point clouds after rigid transformation.
SC graph is constructed by pair-wise distances between correspondences, and the vertices are connected if the distances difference is less than a threshold.
Yang et al.~\cite{yangSACCOT2022} propose SAC-COT by randomly sampling three-point correspondences guided by the SC graph,
which generates more reliable correspondences than RANSAC~\cite{fischlerRandom1981}.
Yan et al.~\cite{yannew2022} optimize the candidate selection by the reliability degree of graph nodes and edges for outlier removal.
Based on reliable seeds selection of SC graph in PointDSC~\cite{baiPointDSC2021}, Chen et al.~\cite{chenSC2PCR2022} eliminate the training procedure and propose SC2-PCR with a second order graph (SOG) to sample the correspondences,
later extended to SC2-PCR++~\cite{chenSC2PCR2023} via consensus reselection with one-to-many matching and mixed feature-spatial consistency metric.
Yan et al.~\cite{yanTurboReg2025} consider the GPU-based parallel searching and propose pivot-guided search algorithm for SOG.
These works demonstrate traditional SOG matching's superiority over end-to-end approaches.

For outlier pruning, Yang et al.~\cite{yangTEASER2021} propose TEASER++ that combines truncated least squares estimation with practical maximum clique algorithm~\cite{parrapractical2020}.
Zhang et al.~\cite{zhang3D2023} relax the constraint of the maximum clique in TEASER++~\cite{yangTEASER2021} and propose MAC to find the maximal cliques of the SOG.
This approach employs modified Bron-Kerbosch algorithm implemented in igraph library~\cite{eppsteinListing2010} with complexity $O(d_v(N\!-\!d_v)3^{d_v/3})$, where $d_v$ is the graph degeneracy.
Laserna~\cite{lasernaCliReg2025} proposes CliReg for sparse graphs, which adopts BnB algorithm in maximal clique search of sparse graph.
While efficient for sparse graphs, maximal clique-based methods approach exponential space and time complexity on recursively searching in dense graphs, which leads to memory limitations reported in~\cite{huangScalable2024}.

In summary, graph-based methods achieve high accuracy particularly with sparse correspondence sets.
However, their applicability is constrained by following primary factors:
1) quadratic space complexity in graph construction limits scalability for large-scale point clouds;
2) clique search exhibits exponential time complexity in the worst case of dense graphs;
3) SC's distance invariance under Euclidean Group ($\mathrm{E}(3)$) rather than Special Euclidean Group ($\mathrm{SE}(3)$) may yield inaccurate registrations under improper rotations.

\subsection{BnB Search in Solution Space}

BnB search provides globally optimal solutions for combinatorial optimization problems, but the time complexity increases exponentially in the dimensionality of the solution space.
Recent works~\cite{yangGoICP2016,parrabustosGuaranteed2018,chenDeterministic2022,huangScalable2024} focus on the search in the 6-DoF solution space of rigid transformation.
Hartley and Kahl~\cite{hartleyGlobalOptimizationRotation2009} pioneered this approach, formulating rotation search as BnB over the closed ball with radius $\pi$ in $\mathrm{R}^3$ using Rodrigues' parameterization.
Yang et al.~\cite{yangGoICP2016} propose Go-ICP with nested BnB search of $\mathrm{SE}(3)$ to estimate the 6-DoF rigid transformation between two point clouds.
To reduce the dimensions of search space, Bustos et al.~\cite{parrabustosFast2016} propose a fast rotation search with circular R-tree indexing,
but it still requires the nested $\mathbb{R}^3$ translation search.
In further research, Bustos and Chin~\cite{parrabustosGuaranteed2018} propose GORE, combining BnB with 1D interval stabbing for outlier removal while leaving the nested translation search problem unresolved.

Although $\mathrm{SE}(3)$ BnB search guarantees global optimality, its computational cost remains prohibitive due to 6-DoF search space complexity.
Stage-wise decomposition provides a practical compromise, approximating global optima through sequential subproblems.
Chen et al.~\cite{chenDeterministic2022} propose TR-DE via two-stage 3-DoF searches:
first solving the 2-DoF rotation axis and 1-DoF translation along the axis through 3D BnB,
and then solving the 1-DoF rotation angle and 2-DoF planar translation via nested BnB.
This reduces the computational complexity from 6 DoF to 3 DoF, while maintaining near-optimal registration accuracy.

Similarly, the 1D interval stabbing algorithm can be utilized in these decomposed stages to reduce the search dimensions~\cite{zhangAccelerating2024}.
Li et al.~\cite{liEfficient2024} decompose registration into three 2-DoF rotation matrix row vector estimation subproblems, and estimate the 1-DoF translation projection via interval stabbing.
Based on the row-vector decomposition, Huang et al.~\cite{huangScalable2024} propose TEAR to solve in two stages:
first resolving an initial row vector through 2-DoF BnB in spherical space with estimating translation via interval stabbing,
then adopting rotation matrix orthogonality constraints between rows to simplify subsequent solution.
Huang et al.~\cite{huangEfficient2024} further enhance this approach with HERE, integrating the SC graph with BnB search.
It preprocesses the input correspondences by valid sampling of SC graph and then performs a progressive outlier removal in three-stage BnB search with interval stabbing.
Nevertheless, staged BnB search improves efficiency but sacrifices the global optimality:
once one of the stages gets trapped in a local optima, it is challenging for the subsequent stages to yield the correct result.

\subsection{Practical LiDAR registration with Gravity Priors}

In contrast to the general 6-DoF registration problem, several studies~\cite{caiPractical2019,liFast2023,liTransformation2024,limQuatro2024} have investigated a practical 4-DoF setting enabled by gravity priors in LiDAR systems, where the rotation is constrained to the yaw angle.
With roll and pitch approximately aligned with the gravity direction, the solution space is effectively reduced to a 2.5D problem, consisting of 3 translational and 1 rotational DoF.

Cai et al.~\cite{caiPractical2019} study this setting by adopting Intrinsic Shape Signatures (ISS)~\cite{zhongIntrinsic2009} for correspondence generation and propose Fast Match Pruning (FMP) to further suppress outliers.
The resulting transformation is estimated via 3-DoF BnB search over the translation space with 1-DoF rotation angle interval stabbing.
Li et al.~\cite{liFast2023} also perform the translation space BnB search but estimate the remaining rotation angle through global voting.
To further improve efficiency based on~\cite{liFast2023}, Li et al.~\cite{liTransformation2024} decompose the 4-DoF problem into multiple stages using Chasles' theorem~\cite{chaslesNote1830}.
Specifically, the translation along the known rotation axis is first estimated via interval stabbing, followed by a 2D BnB search for the rotation center, and a final global voting step to recover the rotation angle.
Lim et al.~\cite{limQuatro2024} consider the 4-DoF case of TEASER++~\cite{yangTEASER2021} and propose Quatro, which adopts maximal clique in SC and graduated non-convexity based truncated least square estimation of Quasi-$\mathrm{SO}(3)$ solution.
Recently, Aoki et al.~\cite{aoki3DBBS2024} propose a correspondence-free formulation that performs global BnB search jointly over translation and yaw angle, assuming roll and pitch lie within a small neighborhood around zero.
This approach avoids explicit correspondence matching in feature space but relies on the voxel map-based nearest neighbor search at the cost of increased computational complexity.
Nevertheless, the above BnB-based methods treat translation as the search domain, making their search efficiency and resolution sensitive to the spatial scale of the point clouds.

Building upon the 2.5D formulation, an even more aggressive dimensionality reduction has been adopted by representing LiDAR point clouds in a 2D bird's-eye-view (BEV) space~\cite{xuRING2023,yuanBTC2024,luoBEVPlace2025},
which enables efficient point cloud retrieval for loop detection.
By projecting point clouds onto the ground plane and discarding the dimension along gravity direction, the registration problem is further simplified to a purely 2D formulation with 3 DoF, significantly reducing computational complexity while retaining sufficient geometric structure for localization and alignment~\cite{hessRealtime2016}.
Overall, the 4-DoF and 3-DoF registration problems can be viewed as special cases of the general 6-DoF formulation, which can be flexibly addressed within a unified BnB framework by appropriately redefining the search domain.

\section{Problem Formulation}
\label{sec:problem}

Given two point clouds with $N$ correspondences $\{\mathbf{P}_i\}_{i=1}^N$ and $\{\mathbf{Q}_i\}_{i=1}^N$, the purpose is to estimate the rotation matrix $\hat{\mathbf{R}}$ and translation vector $\hat{\mathbf{t}}$ that maximize the weighted sum of inliers:
\begin{equation}
  \hat{\mathbf{R}}, \hat{\mathbf{t}} = \arg\max_{\mathbf{R}, \mathbf{t}} \sum_{i=1}^N w_i \cdot \mathbb{I}
  \left( \left\| \mathbf{Q}_i - (\mathbf{R}\mathbf{P}_i + \mathbf{t}) \right\| \le \xi \right) ,
  \label{eq_problem}
\end{equation}
where $\mathbb{I}(\cdot)$ is the indicator function of inliers, and $w_i$ represents the $i^{th}$ correspondence weight in $\mathbf{w}=\{w_i\}_{i=1}^N$.
When $\mathbf{w}=\{1\}^N$, the reduces to the maximizing the inliers count. $\xi$ is the noise threshold, following $\xi^2 = 2\sigma^2 \chi_{3,0.95}^2$ from Gaussian noise $\mathcal{N}(0, \sigma^2\mathbf{I})$ per dimension of a single point,
with $\chi_{3,0.95} \approx 2.796 $ being the $95\%$ quantile of the 3-DoF chi-square distribution.

\begin{figure}
  \centering
  \centering
  \subfloat[]{\includegraphics[height=0.32\columnwidth]{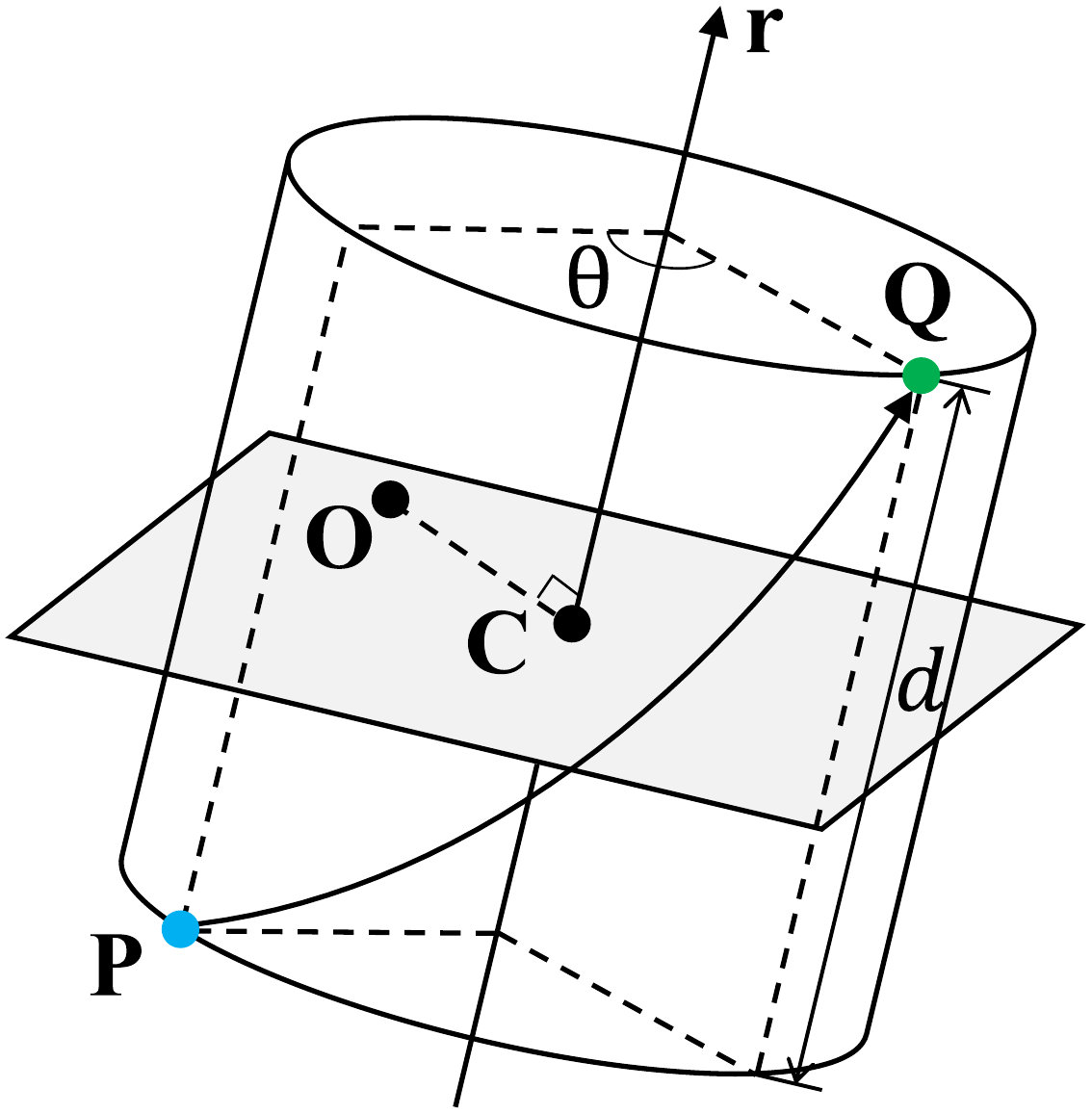}}
  \centering
  \subfloat[]{\includegraphics[height=0.32\columnwidth]{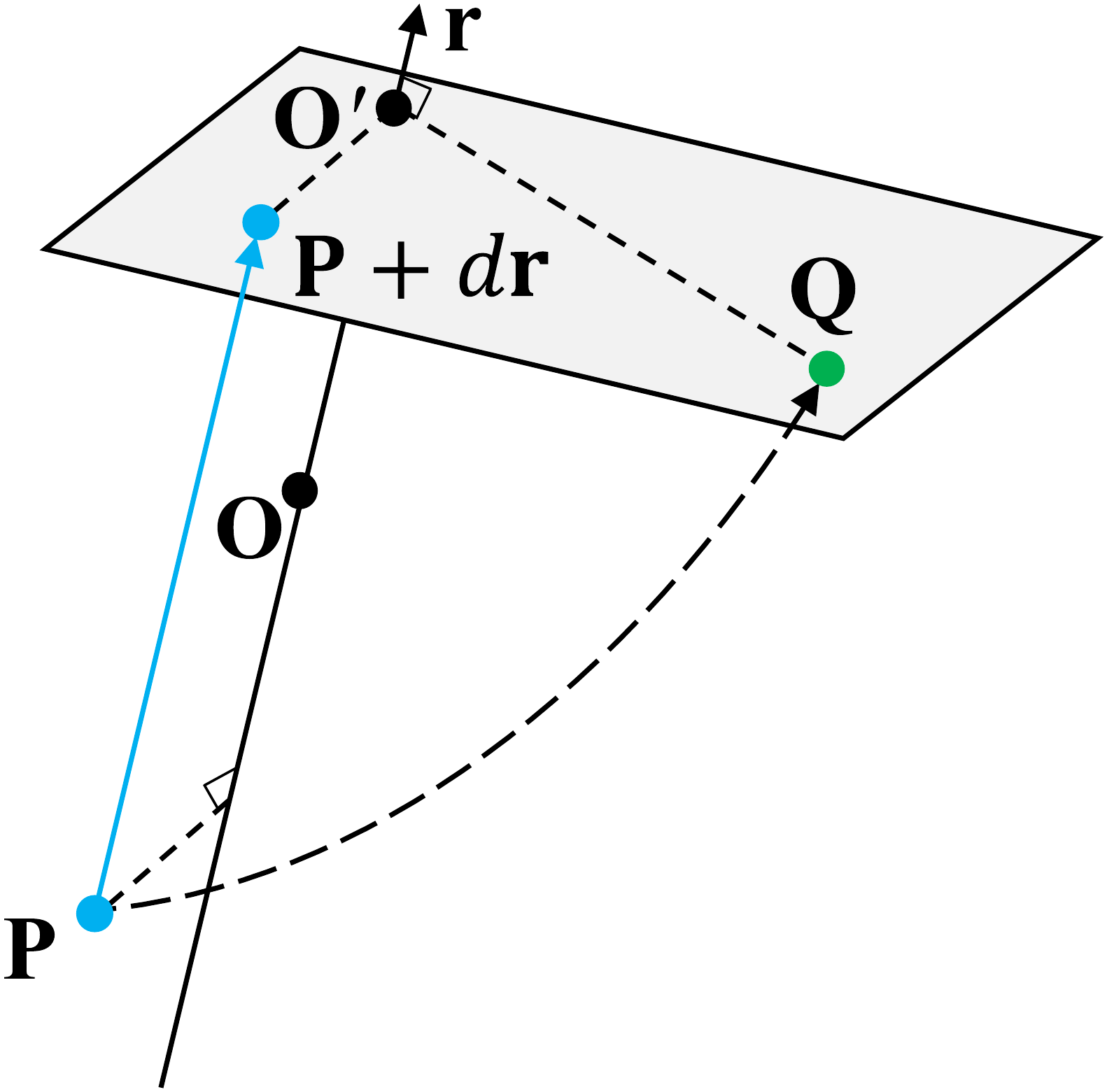}}
  \subfloat[]{\includegraphics[height=0.32\columnwidth]{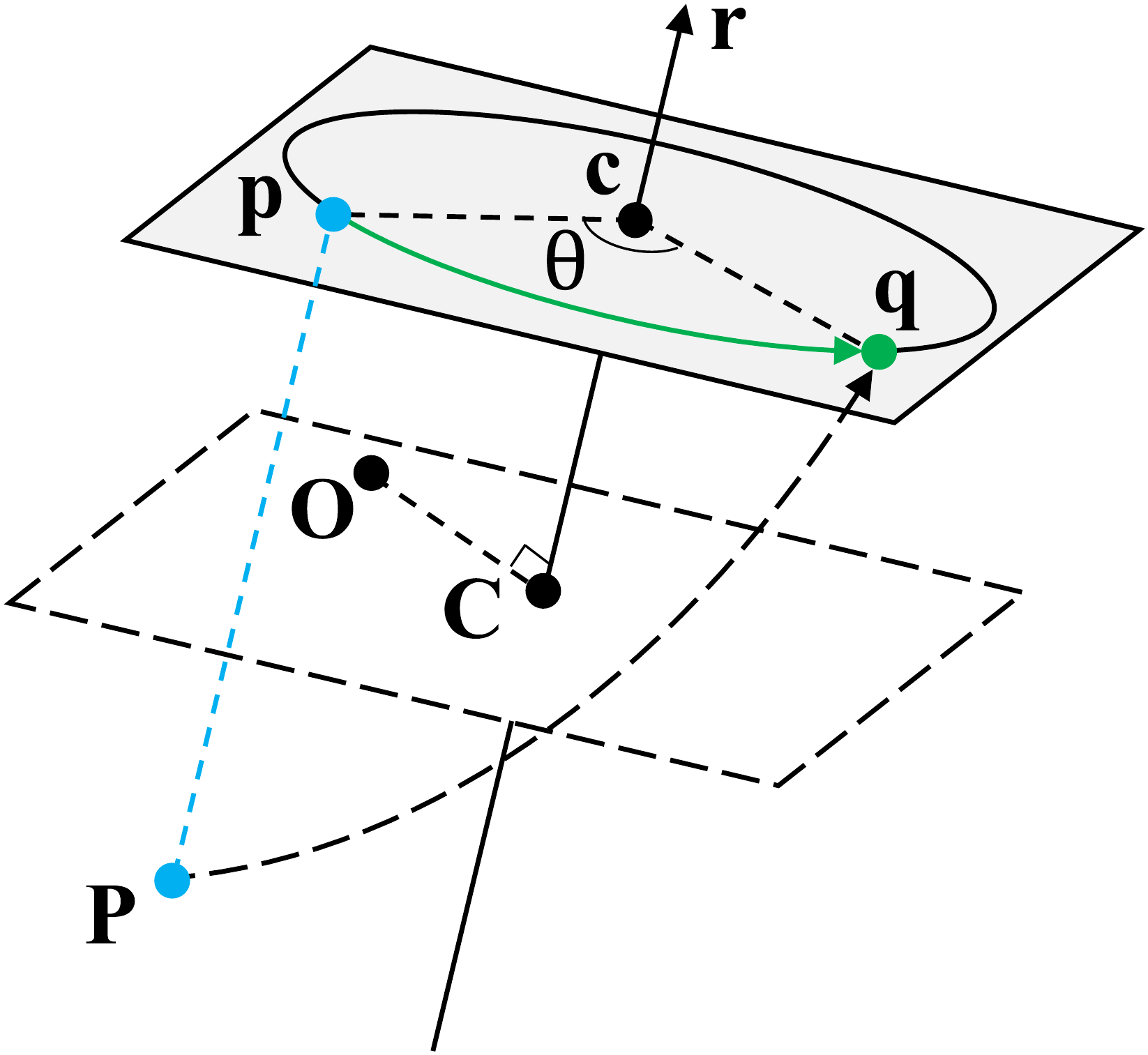}}
  \caption{Decomposition of rigid transformation by Chasles' theorem~\cite{chaslesNote1830}.
    (a) Equivalent screw motion from $\mathbf{P}$ to $\mathbf{Q}$;
    (b) Decomposed translation along the rotation axis direction $\mathbf{r}$;
    (c) Residual 2D rotation about the axis through a determined center $\mathbf{C}$.}
  \label{fig_chasles}
\end{figure}

As shown in Fig.~\ref{fig_chasles}, the rigid transformation is equivalent to a screw motion,
which can be decomposed to a translation along the unit rotation axis direction $\mathbf{r}$ and a rotation about the axis through a fixed point $\mathbf{C}$ satisfying $\mathbf{r} \times (\mathbf{R}\mathbf{C} + \mathbf{t} - \mathbf{C}) = \mathbf{0}$.
Applying the Rodrigues' formula, we reformulate~\eqref{eq_problem} as:
\begin{equation}
  \begin{aligned}
    \hat{\mathbf{R}}, \hat{\mathbf{t}} = & \arg\max_{\mathbf{R}, \mathbf{t}}
    \sum_{i=1}^N w_i \cdot \mathbb{I}
    \left( \left\| (\mathbf{r} \cdot \mathbf{S}_i - d) \mathbf{r} \right. \right. \\
                                         & \left. \left. + \mathbf{r} \times
    \left( \mathbf{Q}_i - \mathbf{C} - \mathbf{R} (\mathbf{P}_i - \mathbf{C}) \right)
    \right\| \le \xi \right),
  \end{aligned}
  \label{eq_chasles}
\end{equation}
with parameters:
\begin{subequations} \label{eq_chasles_params}
  \begin{align}
    \mathbf{S}_i & = \mathbf{Q}_i - \mathbf{P}_i, \label{eq_chasles_params_a} \\
    \mathbf{R}   & = \mathbf{I}\cos\theta
    + (1 - \cos\theta) \mathbf{r}\mathbf{r}^\top
    + \mathbf{r}_\times \sin\theta, \label{eq_chasles_params_b}               \\
    d            & = \mathbf{r} \cdot \mathbf{t}, \label{eq_chasles_params_c} \\
    \mathbf{C}   & = \frac{1}{2} \left(
    \mathbf{I} - \mathbf{r} \mathbf{r}^\top
    + \mathbf{r}_\times  \cot \left( \frac{\theta}{2} \right)
    \right) \mathbf{t}, \label{eq_chasles_params_d}
  \end{align}
\end{subequations}
where $\mathbf{S}_i$ is the difference vector between $\mathbf{Q}_i$ and $\mathbf{P}_i$, $\mathbf{I}$ is the identity matrix, $\mathbf{r}_\times$ is the skew-symmetric matrix of unit rotation axis $\mathbf{r}$, and $\theta$ is the rotation angle in the Rodrigues' formula~\eqref{eq_chasles_params_b}.
Particularly, $\mathbf{C}$ represents an arbitrary point on the rotation axis with direction $\mathbf{r}$, and the selected $\mathbf{C}$ in~\eqref{eq_chasles_params_d} is the origin point projected onto the axis.
The derivation from~\eqref{eq_problem} to~\eqref{eq_chasles} is presented in Appendix~\ref{appd:chasles}.

Based on the decomposition in~\eqref{eq_chasles}, we set minimum branch width $\epsilon$ and search for the globally optimal rigid transformation through two stages:
\begin{enumerate}
  \item{Stage I: BnB search over the hemisphere for rotation axis direction, estimating the translation via interval stabbing of axis-projected correspondences.}
  \item{Stage II: BnB search for the rotation angle, estimating the rotation center using sweep line algorithm accelerated by segment tree.}
\end{enumerate}

\begin{figure*}[!t]
  \centering
  \includegraphics[width=0.9\textwidth]{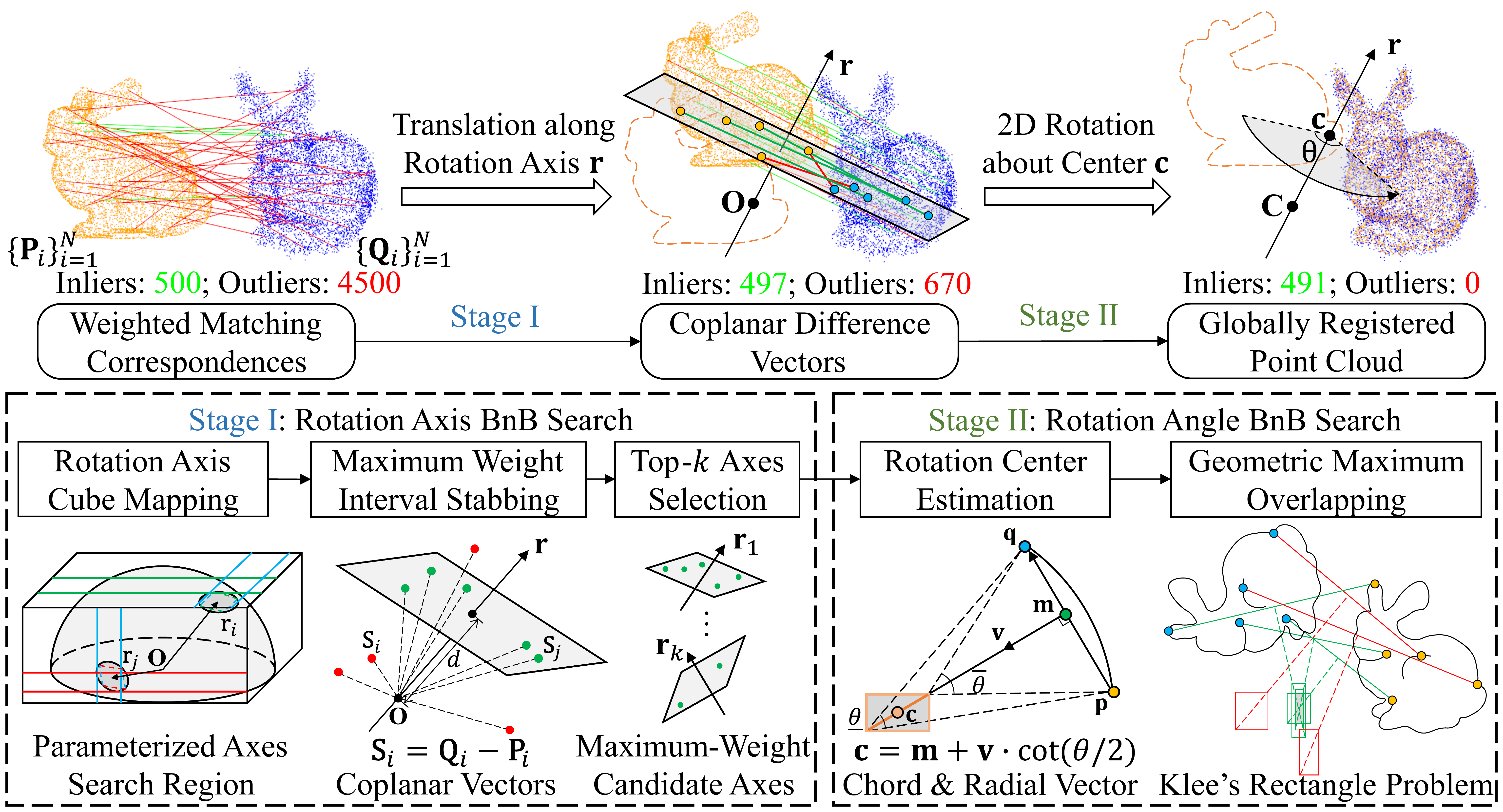}
  \caption{Two-stage BnB search framework of GMOR.}
  \label{fig_framework}
\end{figure*}

The framework of our GMOR method is illustrated in Fig.~\ref{fig_framework}.
With preprocessed weighted matching, the correspondences are filtered in two stages to maximize the weighted sum of inliers.
Furthermore, the subproblem in stage I is selecting maximal-weight coplanar points in 3D space,
and the subproblem in stage II is deterministically estimating the rigid transformation on 2D plane.

\section{Preprocessing: Weighted Cross-Matching in Feature Space}
\label{sec:preprocess}

To generate the weighted correspondences between two original point clouds $\{\mathbf{P}_i^o\}_{i=1}^{N_p}$ and $\{\mathbf{Q}_i^o\}_{i=1}^{N_q}$,
we perform softmax weighted k-nearest neighbor (kNN) search with dustbin in feature space.
The correspondence weights are computed as:
\begin{equation}
  w_i = \frac{\exp \left( -\dfrac{ \| \mathbf{f}^{p}_{i} - \mathbf{f}^{q}_1 \| ^2}
  {2d_\mathrm{f}^2 \|\mathbf{f}^{p}_{i}\|^2} \right)}
  {\exp \left( -\dfrac{\delta^2}{2d_\mathrm{f}^2} \right)
  + \sum\limits_{j=1}^{k_\mathrm{f}} \exp \left( -\dfrac{ \| \mathbf{f}^{p}_i - \mathbf{f}^{q}_j \| ^2}
  {2d_\mathrm{f}^2 \|\mathbf{f}^{p}_{i}\|^2} \right)} ,
  \label{eq_weight}
\end{equation}
where $\mathbf{f}^{p}_i$ and $\mathbf{f}^{q}_j$ are feature vectors of the $i^{th}$ point in $\{\mathbf{P}_i^o\}_{i=1}^{N_p}$ and $j^{\mathrm{th}}$ nearest neighbor in $\{\mathbf{Q}_i^o\}_{i=1}^{N_q}$ respectively.
$d_\mathrm{f}$ is a distance factor in feature space, and $\delta=\sqrt{2}$ is the dustbin threshold to reduce weights for distant features, and $k_\mathrm{f}$ is the number of nearest neighbors.
We implement cross-matching by computing the bidirectional kNN from $\{\mathbf{P}_i^o\}_{i=1}^{N_p}$ to $\{\mathbf{Q}_i^o\}_{i=1}^{N_q}$ and vice versa.
The combined results form correspondences $\{\mathbf{P}_i\}_{i=1}^N \rightarrow \{\mathbf{Q}_i\}_{i=1}^N$ with weights $\{w_i\}_{i=1}^N$.

\section{Stage I: Rotation Axis Search}
\label{sec:stage1}

The stage I objective maximizes inlier weights by finding optimal rotation axis $\mathbf{r}$ and translation distance $d$:
\begin{equation}
  \begin{aligned}
    \hat{\mathbf{r}}, \hat{d} & = \arg\max_{\mathbf{r}, d} W_1                               \\
                              & = \arg\max_{\mathbf{r}, d} \sum_{i=1}^N w_i \cdot \mathbb{I}
    \left( \left| \mathbf{r} \cdot \mathbf{S}_i - d \right| \le \xi_1 \right) ,
  \end{aligned}
  \label{eq_stage1}
\end{equation}
where $\xi_1^2 = 2\sigma^2 \chi_{1, 0.95}^2$ is the 1D noise threshold ($\chi_{1, 0.95} \approx 1.96$).

The objective function~\eqref{eq_stage1} identifies the largest coplanar subset of $\{\mathbf{S}_i\}_{i=1}^N$ by fitting a plane with normal $\mathbf{r}$ at distance $d$ from the origin.
Unlike TR-DE~\cite{chenDeterministic2022}, we parameterize rotation axes via cube mapping~\cite{snyderMap1987} on the hemisphere, as shown in Fig.~\ref{fig_cube}(a).
The transformation from cube-face coordinates $(x_f, y_f)^\top$ (e.g., XY-face) to Cartesian coordinate of rotation axis is
\begin{equation}
  \begin{aligned}
    \mathbf{r}_\mathrm{p} & = (x_f, y_f, 1)^\top ,                                \\
    \mathbf{r}            & = \mathbf{r}_\mathrm{p} / \|\mathbf{r}_\mathrm{p}\| ,
  \end{aligned}
  \label{eq_cube}
\end{equation}
where $\mathbf{r}_\mathrm{p}$ is the unnormalized rotation axis $\mathbf{r}$ on the cube face.
The initial search space per cube face is $[-1, 1] \times [-1, 1]$.
Moreover, a comparative study of different spherical parameterizations (e.g. Miller projection used in TR-DE~\cite{chenDeterministic2022}) is presented in Appendix~\ref{appd:proj}.

\begin{figure}
  \centering
  \subfloat[]{\includegraphics[width=0.5\columnwidth]{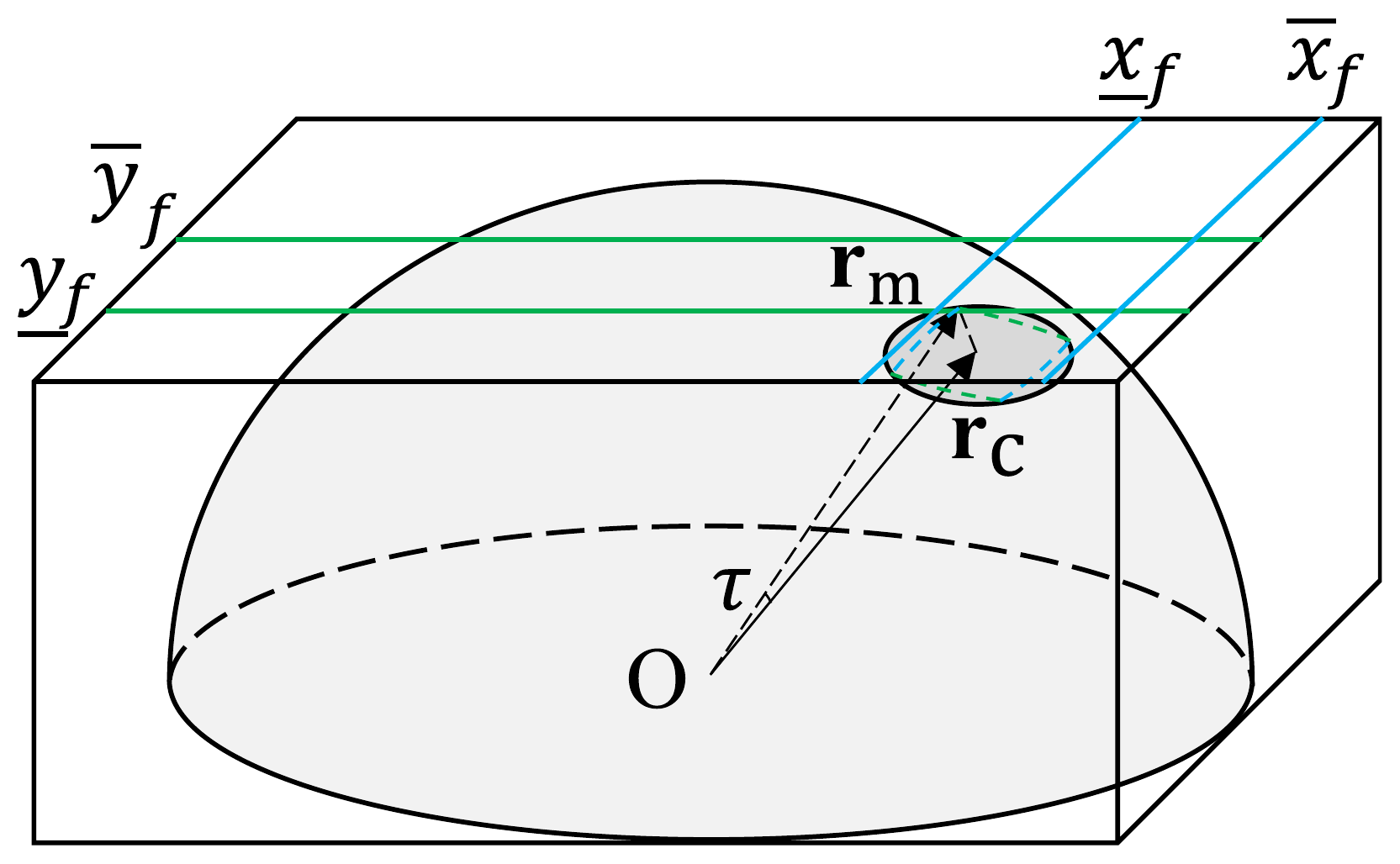}}
  \subfloat[]{\includegraphics[width=0.5\columnwidth]{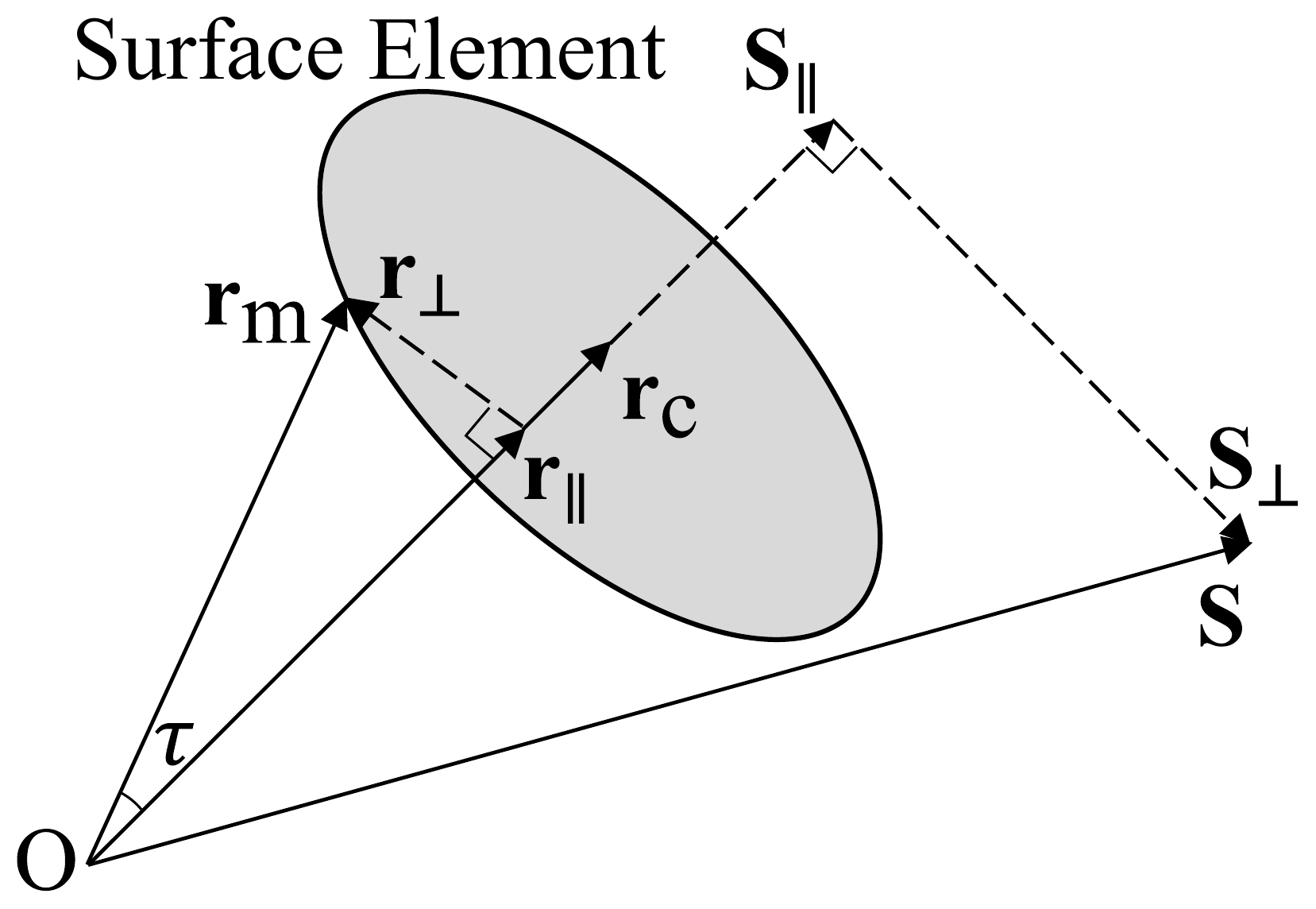}}
  \caption{Rotation axis parameterization for BnB search.
    (a) Cube mapping with surface element bounding circle.
    (b) Decomposition of projection onto rotation axis within the surface element.}
  \label{fig_cube}
\end{figure}

Using hemisphere discretization via cube mapping, we perform 2D BnB search for rotation axis.
The rotation axis search employs best-first BnB with space bisection, guided by the upper bound $\overline{W}_1$.

\subsection{Projection onto Rotation Axis}

The rotation axis $\mathbf{r}_\mathrm{c}$ for a surface element is represented by its center $\left( (\underline{x}_f + \overline{x}_f) / 2, (\underline{y}_f + \overline{y}_f) / 2 \right)^\top$.
The farthest possible axis $\mathbf{r}_\mathrm{m}$ from $\mathbf{r}_\mathrm{c}$ lies at the corner closest to the cube face center.
The translation distance $d$ for correspondence $\mathbf{S}$ projected onto axis $\mathbf{r}$ is:
\begin{equation}
  \begin{aligned}
    d & = \mathbf{r}\cdot\mathbf{S}                                                                 \\
      & = (\mathbf{r}_\parallel + \mathbf{r}_\perp) \cdot (\mathbf{S}_\parallel + \mathbf{S}_\perp) \\
      & = \mathbf{r}_\parallel \cdot \mathbf{S}_\parallel + \mathbf{r}_\perp \cdot \mathbf{S}_\perp
  \end{aligned}
  \label{eq_d}
\end{equation}
where $\mathbf{r}$ is any axis within the surface element,
$\mathbf{r}_\parallel$ and $\mathbf{r}_\perp$ are components parallel and perpendicular to $\mathbf{r}_\mathrm{c}$,
and $\mathbf{S}_\parallel$, $\mathbf{S}_\perp$ are corresponding components of $\mathbf{S}$ in Fig.~\ref{fig_cube}(b).
The norms of the parallel and perpendicular components satisfy:
\begin{equation}
  \begin{aligned}
    \cos\tau & \le \|\mathbf{r}_\parallel\| \le 1    \\
    0        & \le \|\mathbf{r}_\perp\| \le \sin\tau
  \end{aligned}
  \label{eq_tau}
\end{equation}
where $\tau$ is the maximum angle between $\mathbf{r}_\mathrm{m}$ and $\mathbf{r}_\mathrm{c}$.
For each correspondence $\mathbf{S}$, the projection boundaries on $\mathbf{r}_\mathrm{c}$ are:
\begin{equation}
  \begin{aligned}
    \underline{d}_l & = \mathbf{r}_\mathrm{c} \cdot \mathbf{S} - \xi_1 , \\
    \overline{d}_l  & = \mathbf{r}_\mathrm{c} \cdot \mathbf{S} + \xi_1 ,
  \end{aligned}
  \label{eq_d_l}
\end{equation}

Combining~\eqref{eq_d} and~\eqref{eq_tau} and considering the two cases where $\mathbf{S}$ crosses or does not cross the surface element,
we derive tighter boundaries than TR-DE~\cite{chenDeterministic2022}:
\begin{equation}
  \begin{aligned}
    \underline{d}_u & = d_\parallel - d_\perp - \xi_1,                                                  \\
    \overline{d}_u  & =
    \begin{cases}
      d_\parallel + d_\perp + \xi_1, & \|\mathbf{S}\| \cos\tau > \mathbf{r}_\mathrm{c} \cdot \mathbf{S}   \\
      \|\mathbf{S}\| + \xi_1,        & \|\mathbf{S}\| \cos\tau \le \mathbf{r}_\mathrm{c} \cdot \mathbf{S}
    \end{cases} \\
                    & s.t. \quad \mathbf{r}_\mathrm{c} \cdot \mathbf{S} \ge 0
  \end{aligned}
  \label{eq_d_ub}
\end{equation}
where $d_\parallel = \mathbf{r}_\mathrm{c} \cdot \mathbf{S} \cos\tau$ and $d_\perp = \left\| \mathbf{S}_\perp \right\| \sin\tau$.

When $\mathbf{r}_\mathrm{c} \cdot \mathbf{S} < 0$, the boundaries in~\eqref{eq_d_ub} are inverted:
\begin{equation}
  \begin{aligned}
    \overline{d}_u  & = d_\parallel + d_\perp + \xi_1,                                                   \\
    \underline{d}_u & =
    \begin{cases}
      d_\parallel - d_\perp - \xi_1, & \|\mathbf{S}\| \cos\tau > -\mathbf{r}_\mathrm{c} \cdot \mathbf{S}   \\
      -\|\mathbf{S}\| - \xi_1,       & \|\mathbf{S}\| \cos\tau \le -\mathbf{r}_\mathrm{c} \cdot \mathbf{S}
    \end{cases} \\
                    & s.t. \quad \mathbf{r}_\mathrm{c} \cdot \mathbf{S} < 0
  \end{aligned}
  \label{eq_d_ubn}
\end{equation}

\subsection{Gradually Converged Interval Stabbing}

The boundaries of translation distance form $N$ intervals $\left\{ [\underline{d}_{li}, \overline{d}_{li}] \right\}_{i=1}^N$ and $\left\{ [\underline{d}_{ui}, \overline{d}_{ui}] \right\}_{i=1}^N$ respectively.
The maximum accumulated inlier weights correspond to the peak interval overlap, as shown in Fig.~\ref{fig_intstab}(a) and (b).
It can be estimated by interval stabbing algorithm~\cite{parrabustosGuaranteed2018} modified by weights:
1) Sort all interval boundaries by $d$, and assign the $i^{th}$ weight of left interval boundary to $+w_i$,
right boundary to $-w_i$ for integrating from left to right, as shown in Fig.~\ref{fig_intstab}(c);
2) compute cumulative weights along $d$-axis, and identify the global maximum in Fig.~\ref{fig_intstab}(d).

The upper and lower weights' bounds of the current branch in~\eqref{eq_stage1} are estimated using two sets of intervals:
\begin{equation}
  \begin{aligned}
    \overline{W}_1  & = \mathrm{IntervalStabbing} \left( \left\{ [\underline{d}_{ui}, \overline{d}_{ui}] \right\}_{i=1}^N \right) , \\
    \underline{W}_1 & = \mathrm{IntervalStabbing} \left( \left\{ [\underline{d}_{li}, \overline{d}_{li}] \right\}_{i=1}^N \right) ,
  \end{aligned}
  \label{eq_intstab}
\end{equation}
where $\mathrm{IntervalStabbing}(\cdot)$ returns RMQ of the integrated weights.
The inlier indices contributing to the peak are not stored during BnB search, but will be reselected based on the optimal rotation axis after search completion.

To accelerate BnB search and prefilter outliers, we implement a gradually converged strategy per branch node:
1) Compute threshold $W_\rho = \rho \overline{W}_1 + (1 - \rho)\underline{W}_1^{\ast}$ where $0 \le \rho < 1$,
$\overline{W}_1$ is the current upper bound, and $\underline{W}_1^{\ast}$ the global best lower bound;
2) Determine the left interval boundary $\underline{\hat{d}}_u$ and right boundary $\overline{\hat{d}}_u$ where accumulated weights exceed $W_\rho$, as shown in Fig.~\ref{fig_intstab}(d);
3) The possible translation distance range $[\underline{\hat{d}}_u, \overline{\hat{d}}_u]$ of current branch is inherited to the subbranches,
and prefilter the points by $\underline{\hat{d}}_u \le \mathbf{r}_\mathrm{c} \cdot \mathbf{S} \le \overline{\hat{d}}_u$.

\begin{figure}
  \centering
  \hfill
  \subfloat[]{\includegraphics[width=0.35\columnwidth]{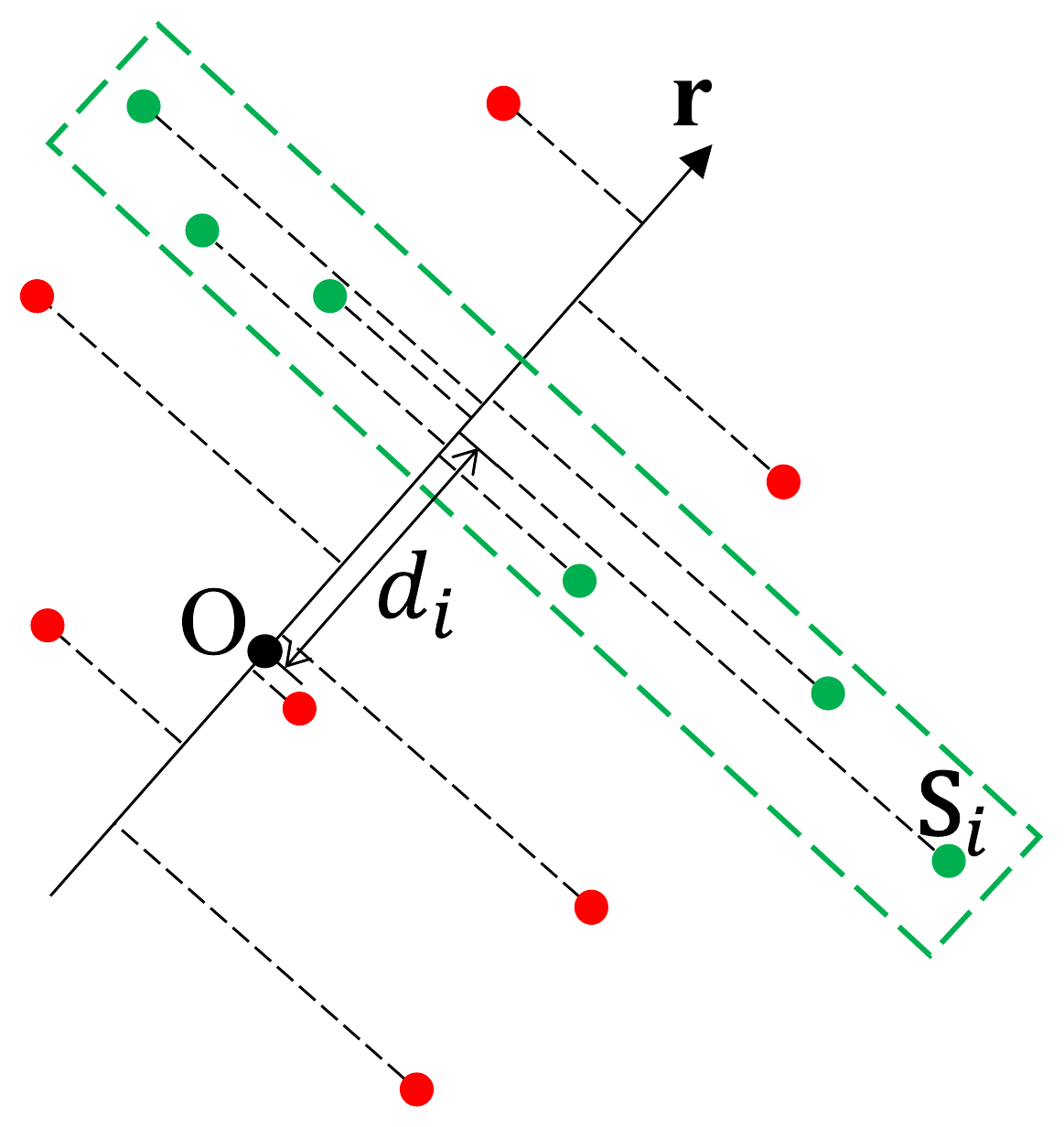}}
  \hfill
  \subfloat[]{\includegraphics[width=0.5\columnwidth]{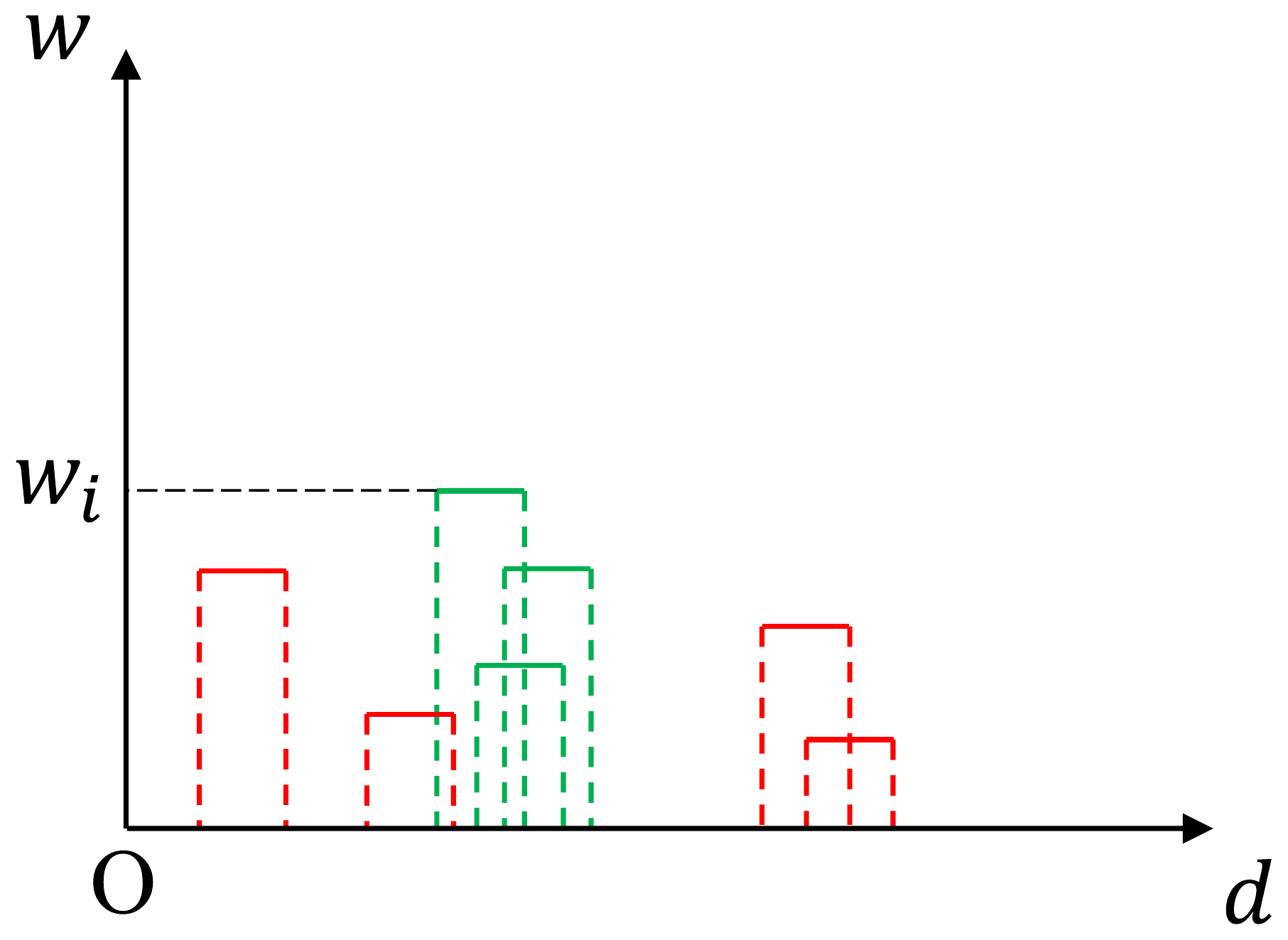}}
  \\
  \subfloat[]{\includegraphics[width=0.5\columnwidth]{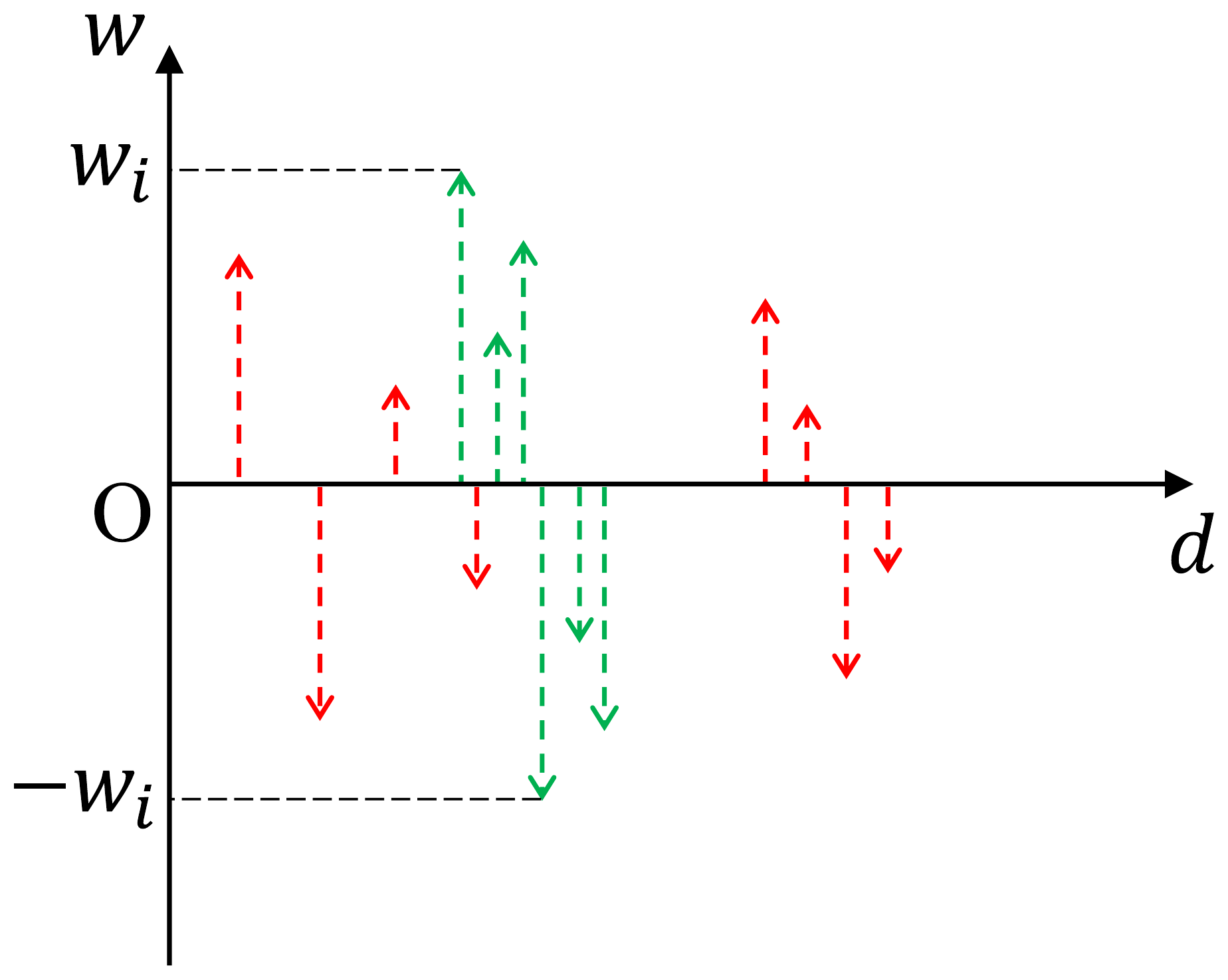}}
  \subfloat[]{\includegraphics[width=0.5\columnwidth]{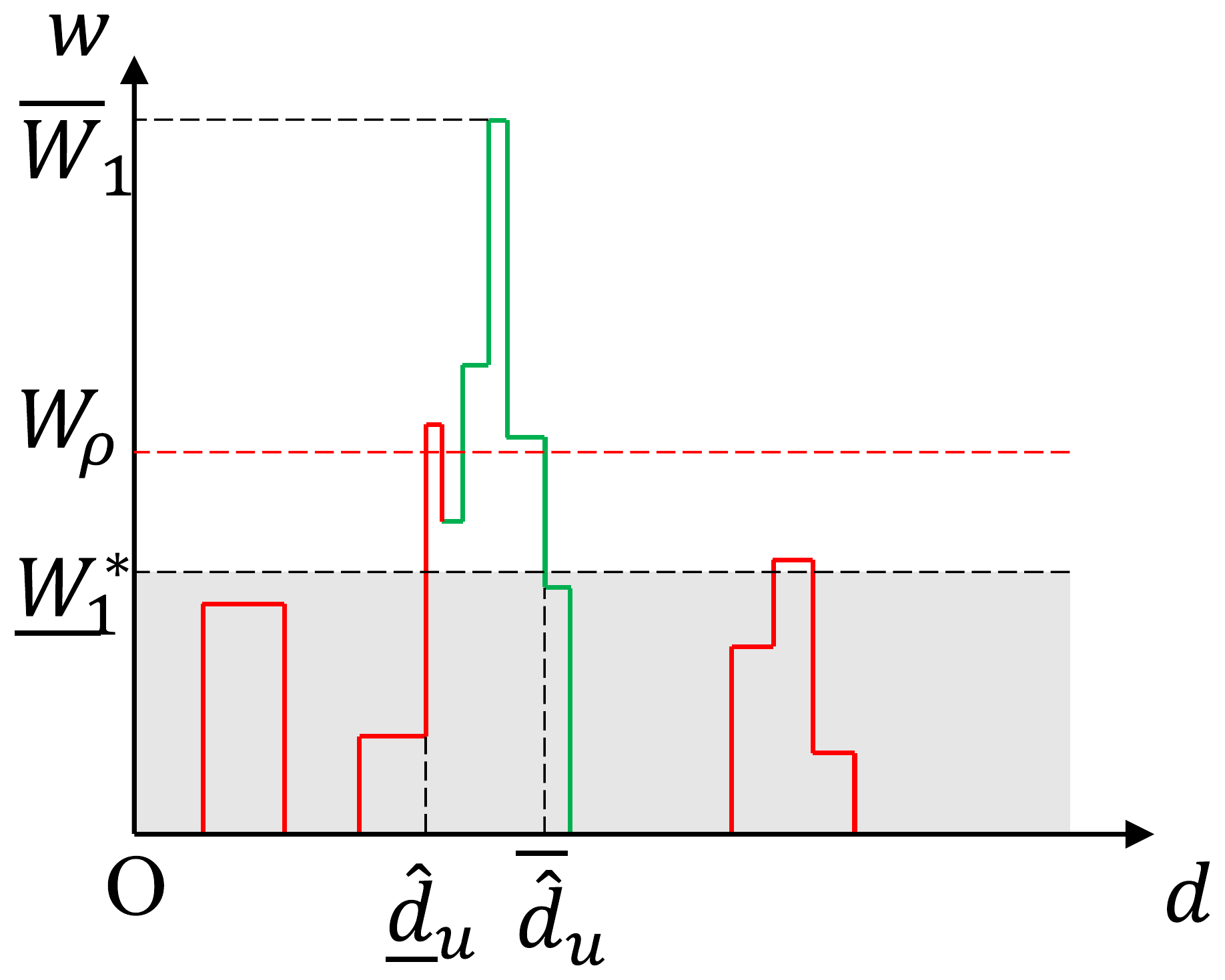}}
  \caption{1D RMQ solved by interval stabbing.
    (a) Difference vectors of correspondences projected onto rotation axis.
    (b) Estimated intervals $[\underline{d}, \overline{d}]$ and weights.
    (c) Integrating discretized left interval boundary with $+w_i$ weight and right with $-w_i$ weight.
    (d) Accumulated weights and maximum overlapping.}
  \label{fig_intstab}
\end{figure}

In the BnB search, the branch will be pruned in the following cases:
1) Current upper bound $\overline{W}_1$ is less than the best lower bound $\underline{W}_{1}^{\ast}$;
2) The current branch width $\overline{x}_f - \underline{x}_f$ or $\overline{y}_f - \underline{y}_f$ is less than the minimum branch width $\epsilon$;
3) The current width of upper lower bounds $\overline{W}_1 - \underline{W}_{1}$ is less than a bound threshold $\epsilon_b$.

\subsection{Post Top-k Plane Fitting}

Specifically, while the BnB search result for the rotation axis on the hemisphere yields the optimal solution for stage I,
this axis may not correspond to the global optimum for the rigid transformation.
We initialize the search by partitioning each cube face into four branches defined by corner coordinates $\left( 0, 0 \right)^\top $ and $\left( \pm 1, \pm 1 \right)^\top $.
BnB search is then executed independently for each branch.
The top-$k$ candidate rotation axes with highest $\underline{W}_1$ are selected as initial candidates for stage II.

These rotation axes are refined via plane fitting using filtered inliers.
The $j^{th}$ fitting plane of inliers $\{\mathbf{S}_i\}_{i=1}^{N_j}$ is estimated through eigenvalue decomposition of the symmetric covariance matrix $\mathbf{\Gamma}_S$:
\begin{subequations}
  \begin{align}
    \mathbf{S}_c      & = \frac{1}{N_j} \sum_{i=1}^{N_j} \mathbf{S}_i , \label{eq_planeA}               \\
    \mathbf{\Gamma}_S & = \frac{1}{N_j} \sum_{i=1}^{N_j}
    \left( \mathbf{S}_i - \mathbf{S}_c \right) \left( \mathbf{S}_i - \mathbf{S}_c \right)^\top
    = \mathbf{VDV}^\top , \label{eq_planeB}                                                             \\
    \mathbf{r}_j      & = \mathbf{v}_3 , \quad d_j = \mathbf{r}_j \cdot \mathbf{S}_c, \label{eq_planeC}
  \end{align}
\end{subequations}
where $\mathbf{S}_c$ is the centroid of inliers,
$\mathbf{V}=[\mathbf{v}_1, \mathbf{v}_2, \mathbf{v}_3]$ contains orthonormal eigenvectors,
and $\mathbf{D}=\mathrm{diag}(\lambda_1, \lambda_2, \lambda_3)$ with $\lambda_1 \ge \lambda_2 \ge \lambda_3$.
The rotation axis $\mathbf{r}_j$ corresponds to the minimal-variance direction $\mathbf{v}_3$, and $d_j$ is the projected centroid distance.

\section{Stage II: Rotation Angle Search}
\label{sec:stage2}

For each top-$k$ candidate rotation axis from stage I, we project the filtered correspondences onto the plane perpendicular to $\mathbf{r}$.
The problem in stage II then becomes a 2D rigid registration with 3 DoF (1 DoF for the rotation angle and 2 DoF for the translation vector).
We efficiently solve this by reducing the search to 1D rotation angle:

\begin{enumerate}
  \item{The 2D rigid transformation is a special case where the axial translation is zero in Chasles' theorem, which is equivalent to a rotation about a fixed point.}
  \item{The global optimum of 2D rigid transformation can be found in at most $O(\frac{1}{\epsilon} \cdot N \log N)$ time.}
\end{enumerate}

Time complexity analysis is detailed in Section~\ref{sec:time}.
As defined previously, the objective function of rotation about a point in stage II is
\begin{equation}
  \begin{aligned}
    \hat{\theta}, \hat{\mathbf{c}} & = \arg\max_{\theta, \mathbf{c}} W_2 \\
                                   & = \arg\max_{\theta, \mathbf{c}}
    \sum_{i=1}^N w_i \cdot \mathbb{I} \left(
    \left\| (\mathbf{q}_i - \mathbf{c}) - \mathbf{R}_\theta (\mathbf{p}_i - \mathbf{c}) \right\| \le \xi_2
    \right),
  \end{aligned}
  \label{eq_stage2}
\end{equation}
where $\mathbf{p}_i$ and $\mathbf{q}_i$ denote the 2D projections of $\mathbf{P}_i$ and $\mathbf{Q}_i$ onto the plane perpendicular to the rotation axis,
$\mathbf{R}_\theta$ represents the 2D rotation matrix parameterized by angle $\theta$,
$\xi_2^2 = 2\sigma^2 \chi_{2, 0.95}^2$ defines the noise threshold based on the distribution of distances between two independent 2D points under Gaussian noise ($\chi_{2, 0.95} \approx 2.448$),
and $\mathbf{c}$ is the rotation center.

Directly BnB searching over the 3-DoF parameter space in~\eqref{eq_stage2} is computationally inefficient.
To address this issue, we leverage geometric properties of 2D rigid transformations:
We reduce the search to the 1-DoF rotation angle $\theta$, while efficiently estimating the corresponding optimal rotation center $\mathbf{c}$ for each candidate angle using a sweep line algorithm accelerated by a segment tree.

\subsection{Rotation Center Determined by Angle}

\begin{figure}
  \centering
  \subfloat[]{\includegraphics[width=0.54\columnwidth]{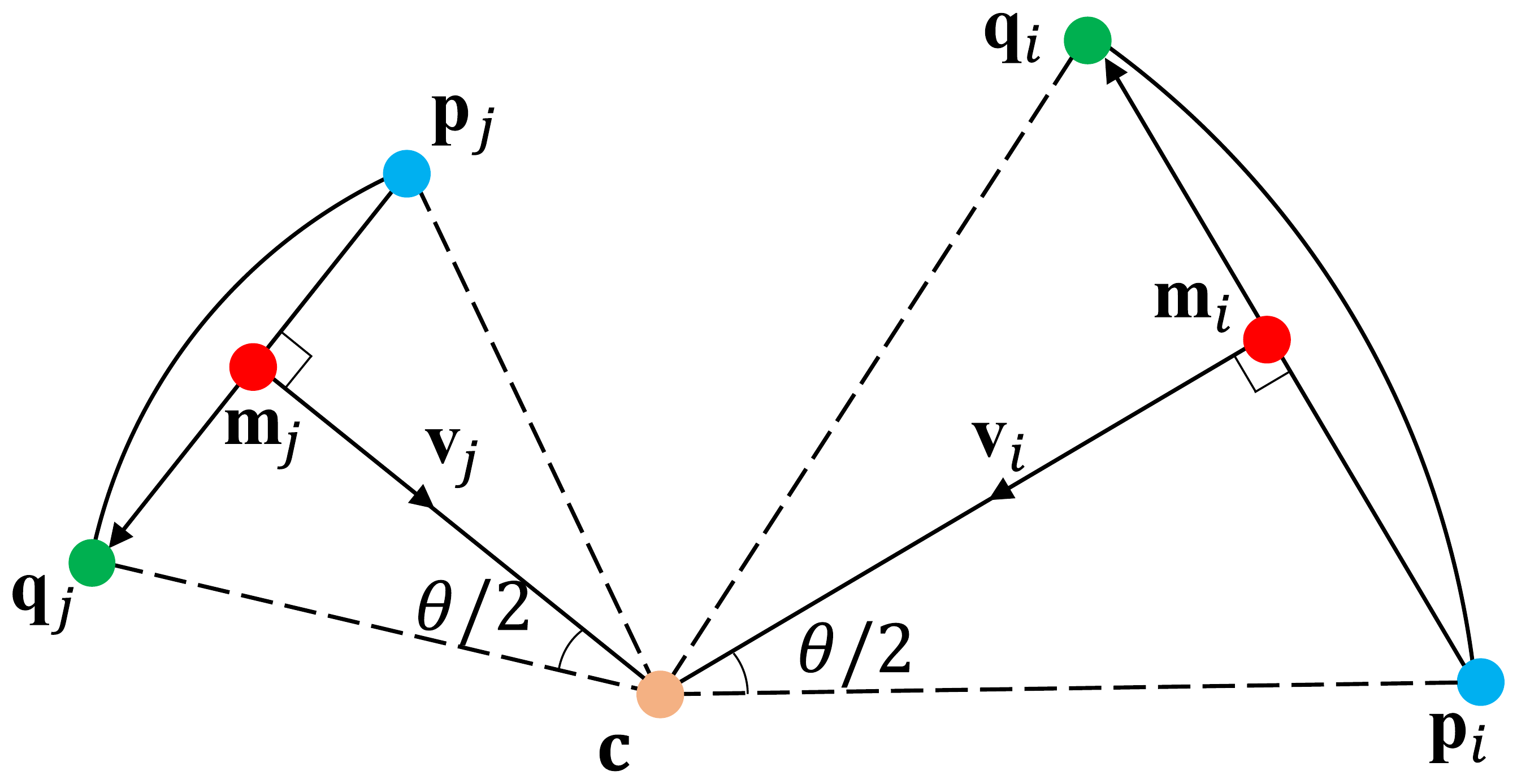}}\hfill
  \subfloat[]{\includegraphics[width=0.42\columnwidth]{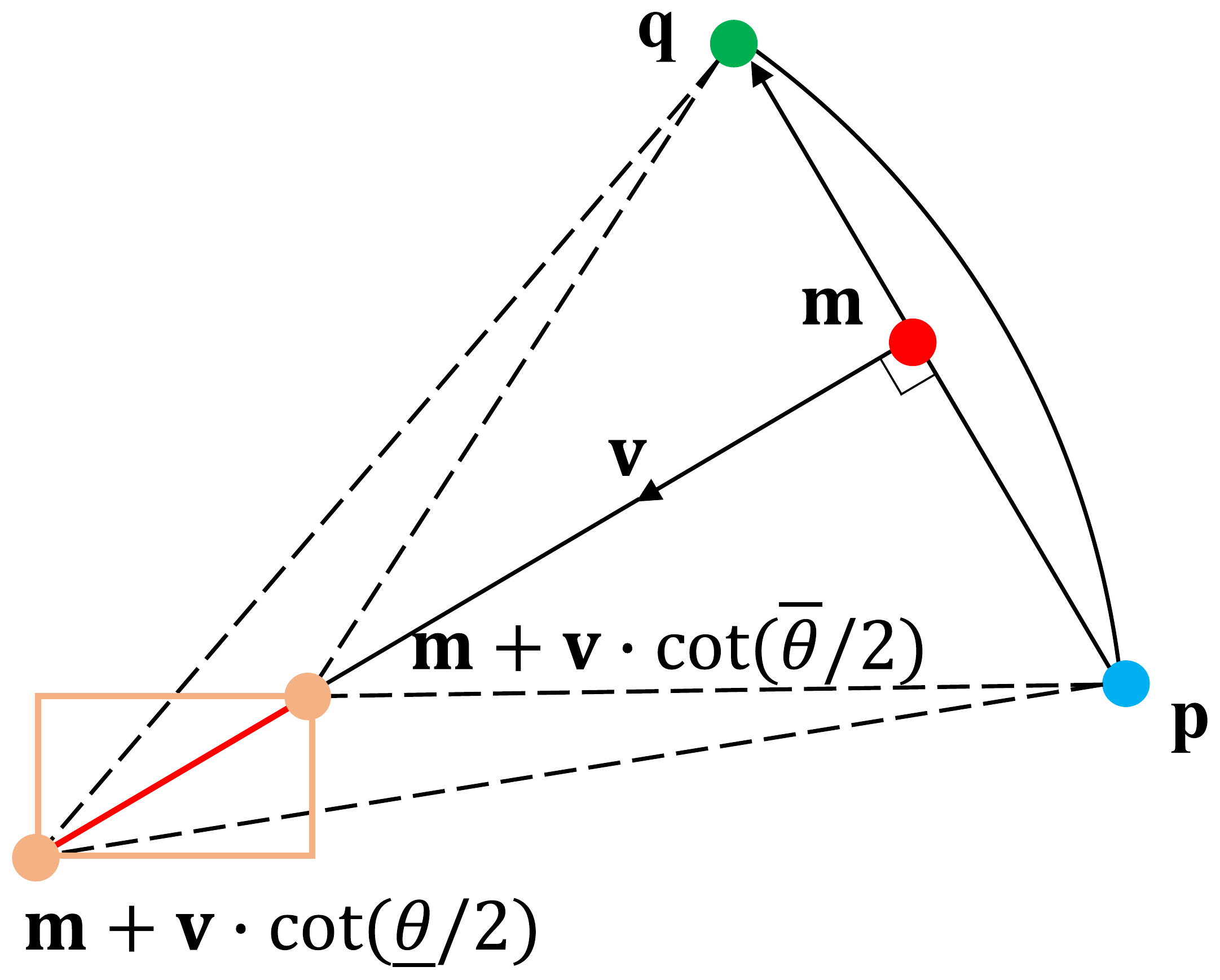}}
  \caption{Rigid transformation on the 2D plane.
    (a) Chords and arcs for rotation of $\theta$ angle about $\mathbf{c}$.
    (b) Rotation center determined by $[\underline{\theta}, \overline{\theta}]$ and correspondence.}
  \label{fig_rot2d}
\end{figure}

Given rotation angle $\theta$ and center $\mathbf{c}$, an inlier $\mathbf{p}$ moves along the arc to $\mathbf{q}$, as shown in Fig.~\ref{fig_rot2d}(a).
The chord midpoint $\mathbf{m}$ and the perpendicular vector $\mathbf{v}$ toward $\mathbf{c}$ are:
\begin{equation}
  \begin{aligned}
    \mathbf{m} & = \frac{1}{2} (\mathbf{q} + \mathbf{p}) ,                            \\
    \mathbf{v} & = \frac{1}{2} \mathbf{R}_{\frac{\pi}{2}} (\mathbf{q} - \mathbf{p}) ,
  \end{aligned}
  \label{eq_vecs2d}
\end{equation}
where $\mathbf{R}_{\frac{\pi}{2}}$ is the 2D rotation matrix of $90^\circ$.
The rotation center $\mathbf{c}$ for given angle $\theta$ and chord $\mathbf{p} \rightarrow \mathbf{q}$ is:
\begin{equation}
  \begin{aligned}
    \mathbf{c}(\theta) & = (x_\mathbf{c}(\theta), y_\mathbf{c}(\theta))^\top            \\
                       & = \mathbf{m} + \mathbf{v} \cot \left( \frac{\theta}{2} \right) \\
                       & = \frac{1}{2}
    \left[
      \begin{matrix}
        x_q + x_p - (y_q - y_p)\cot \left( \frac{\theta}{2} \right) \\
        y_q + y_p + (x_q - x_p)\cot \left( \frac{\theta}{2} \right)
      \end{matrix}
      \right] ,
  \end{aligned}
  \label{eq_rotcenter}
\end{equation}
where $\mathbf{p} = (x_p, y_p)^\top$ and $\mathbf{q} = (x_q, y_q)^\top$ are 2D coordinates.
Therefore, the range of rotation center $\mathbf{c}$ is determined by the range of angle $[\underline{\theta}, \overline{\theta}]$ and~\eqref{eq_rotcenter}, as shown in Fig.~\ref{fig_rot2d}(b).
Given the prior Gaussian noise $\sigma$ in each dimension of correspondences, the covariance matrix of $\mathbf{c}$ is derived as:
\begin{equation}
  \mathbf{\Gamma}_\mathbf{c} = \frac{1}{2} \csc^2 \left( \frac{\theta}{2} \right) \sigma^2 \mathbf{I},
  \label{eq_centerCov}
\end{equation}
where $\mathbf{I}$ is the identity matrix.
As a result that the noises of $\mathbf{c}$ in $x$ and $y$ directions are independent, and the prior noise threshold of the rotation center is
\begin{equation}
  \xi_\mathbf{c}^2(\theta) = \frac{1}{2} \csc^2 \left( \frac{\theta}{2} \right) \sigma^2 \chi_{2,0.95}^2.
  \label{eq_centerxi}
\end{equation}

The noise threshold constrains the possible rotation center $\mathbf{c}$ to a circular region centered at $\mathbf{c}(\theta)$ with radius $\xi_\mathbf{c}(\theta)$.
This geometric constraint is equivalent to the inlier condition in~\eqref{eq_stage2}, formalized by:
\begin{lemma}
  \label{lem_center}
  The rotation center $\mathbf{c}$ within the bounding area determined by $\mathbf{p}, \mathbf{q}, \theta$ and $\sigma$ satisfies the inequality in the indicator function of~\eqref{eq_stage2}.
\end{lemma}

\begin{proof}
  By converting the noise threshold $\xi_\mathbf{c}(\theta)$ in x and y dimensions to polar coordinates, the bounding area of $\mathbf{c}$ is rewritten as
  \begin{equation}
    \begin{aligned}
      \mathbf{c} & = \mathbf{c}(\theta) + \xi_\mathbf{c}(\theta)\mathbf{u}, \\
      s.t. \quad & \left\| \mathbf{u} \right\| \le 1 ,                      \\
    \end{aligned}
    \label{eq_centerbound}
  \end{equation}
  where $\mathbf{u}$ is an arbitrary 2D vector within the unit circle.
  The residual vector in~\eqref{eq_stage2} is
  \begin{equation}
    \mathbf{e} = (\mathbf{q} - \mathbf{c})  - \mathbf{R}_\theta (\mathbf{p} - \mathbf{c}) .
    \label{eq_res}
  \end{equation}

  The rotation center $\mathbf{c}(\theta)$ derived by geometric constraint in~\eqref{eq_rotcenter} follows that $(\mathbf{q} - \mathbf{c}(\theta))  - \mathbf{R}_\theta (\mathbf{p} - \mathbf{c}(\theta)) = \mathbf{0}$, and the norm of residual vector is
  \begin{subequations}
    \begin{align}
      \left\|\mathbf{e}\right\|^2 & = \left\| \xi_\mathbf{c}(\theta)
      \left( \mathbf{I} - \mathbf{R}_\theta \right)
      \mathbf{u} \right\|^2 \label{eq_resnorm_a}                                         \\
                                  & = \left\| \xi_\mathbf{c}(\theta)
      \left( \mathbf{R}_{\frac{\theta}{2}}^\top - \mathbf{R}_{\frac{\theta}{2}} \right)
      \mathbf{u} \right\|^2 \label{eq_resnorm_b}                                         \\
                                  & = \left\| \xi_\mathbf{c}(\theta)
      \left[
        \begin{matrix}
          0                                      & 2\sin \left( \frac{\theta}{2} \right) \\
          -2\sin \left( \frac{\theta}{2} \right) & 0
        \end{matrix}
        \right]
      \mathbf{u}\right\|^2 \label{eq_resnorm_c}                                          \\
                                  & = 2\sigma^2 \chi_{2, 0.95}^2 \left\|
      \left[
        \begin{matrix}
          0  & 1 \\
          -1 & 0
        \end{matrix}
        \right]
      \mathbf{u} \right\|^2 \label{eq_resnorm_d}                                         \\
                                  & = 2\sigma^2 \chi_{2, 0.95}^2
      \left\| \mathbf{R}_{\frac{\pi}{2}}^\top \mathbf{u} \right\|^2 \label{eq_resnorm_e} \\
                                  & \le 2\sigma^2 \chi_{2, 0.95}^2, \label{eq_resnorm_f}
    \end{align}
  \end{subequations}
  which is equivalent to the inequality with threshold $\xi_2$ in~\eqref{eq_stage2}.
  Therefore, we can solve the problem of rotation center $\mathbf{c}$ with maximum overlapping instead of~\eqref{eq_stage2}.
\end{proof}

For computational convenience in the sweep line algorithm in subsequent sections, we approximate the bounding circle as a square of equivalent area.
The half-edge length $\hat{\xi}_\mathbf{c}(\theta)$ satisfies:
\begin{equation}
  \hat{\xi}_\mathbf{c}^2(\theta) = \frac{\pi}{4} \xi_\mathbf{c}^2(\theta).
  \label{eq_centersqr}
\end{equation}

Combining~\eqref{eq_rotcenter} and~\eqref{eq_centerxi}, the bounding box corners for angle interval $[\underline{\theta}, \overline{\theta}]$ and correspondence $\{\mathbf{p}, \mathbf{q}\}$ are:
\begin{subequations} \label{eq_boxub}
  \begin{align}
    \underline{x}_{\mathbf{c}u} & = \min \left( x_\mathbf{c}(\underline{\theta}) - \hat{\xi}_\mathbf{c}(\underline{\theta}),
    x_\mathbf{c}(\overline{\theta}) - \hat{\xi}_\mathbf{c}(\overline{\theta}) \right) , \label{eq_boxub_a}                   \\
    \overline{x}_{\mathbf{c}u}  & = \max \left( x_\mathbf{c}(\underline{\theta}) + \hat{\xi}_\mathbf{c}(\underline{\theta}),
    x_\mathbf{c}(\overline{\theta}) + \hat{\xi}_\mathbf{c}(\overline{\theta}) \right) , \label{eq_boxub_b}                   \\
    \underline{y}_{\mathbf{c}u} & = \min \left( y_\mathbf{c}(\underline{\theta}) - \hat{\xi}_\mathbf{c}(\underline{\theta}),
    y_\mathbf{c}(\overline{\theta}) - \hat{\xi}_\mathbf{c}(\overline{\theta}) \right) , \label{eq_boxub_c}                   \\
    \overline{y}_{\mathbf{c}u}  & = \max \left( y_\mathbf{c}(\underline{\theta}) + \hat{\xi}_\mathbf{c}(\underline{\theta}),
    y_\mathbf{c}(\overline{\theta}) + \hat{\xi}_\mathbf{c}(\overline{\theta}) \right) . \label{eq_boxub_d}
  \end{align}
\end{subequations}

Similar to stage I, we compute a bounding box centered at $\theta_c = (\underline{\theta} + \overline{\theta}) / 2$:
\begin{subequations} \label{eq_boxlb}
  \begin{align}
    \underline{x}_{\mathbf{c}l} & = x_\mathbf{c}(\theta_c) - \hat{\xi}_\mathbf{c}(\theta_c) , \label{eq_boxlb_a} \\
    \overline{x}_{\mathbf{c}l}  & = x_\mathbf{c}(\theta_c) + \hat{\xi}_\mathbf{c}(\theta_c) , \label{eq_boxlb_b} \\
    \underline{y}_{\mathbf{c}l} & = y_\mathbf{c}(\theta_c) - \hat{\xi}_\mathbf{c}(\theta_c) , \label{eq_boxlb_c} \\
    \overline{y}_{\mathbf{c}l}  & = y_\mathbf{c}(\theta_c) + \hat{\xi}_\mathbf{c}(\theta_c) . \label{eq_boxlb_d}
  \end{align}
\end{subequations}

For each angle interval $[\underline{\theta}, \overline{\theta}]$, the rectangles with $(\underline{x}_{\mathbf{c}u}, \underline{y}_{\mathbf{c}u}, \overline{x}_{\mathbf{c}u}, \overline{y}_{\mathbf{c}u})$ are used to estimate the upper bound $\overline{W}_2$,
and the rectangles with $(\underline{x}_{\mathbf{c}l}, \underline{y}_{\mathbf{c}l}, \overline{x}_{\mathbf{c}l}, \overline{y}_{\mathbf{c}l})$ are used to estimate the lower bound $\underline{W}_2$.

\subsection{Maximum Overlapping via Sweep Line Algorithm}

At each rotation angle interval branch $[\underline{\theta}, \overline{\theta}]$, these axis-aligned rectangles form geometric intersections of Klee's rectangle problems.
We compute maximum weighted overlap using sweep line algorithm~\cite{bentleyAlgorithms1979} with segment tree, as shown in Fig.~\ref{fig_sweepline}.

\subsubsection{Segments discretization}

The rectangles in~\eqref{eq_boxub} and~\eqref{eq_boxlb} are discretized into vertical segments of left and right edges $(\underline{x}, \underline{y}, \overline{y}, w)$ and $(\overline{x}, \underline{y}, \overline{y}, -w)$, respectively, where $w$ is the rectangle weight.
The endpoints of segments are sorted by $y$-coordinate, and we discretize them by sorted indices from bottom to top:
\begin{equation}
  \begin{aligned}
    \mathbf{s}_l & = \left(\underline{x}, y_{l}, y_{r}, w\right) , \\
    \mathbf{s}_r & = \left(\overline{x}, y_{l}, y_{r}, -w\right) , \\
    s.t. \quad 1 & \le y_{l} < y_{r} \le 2N,
  \end{aligned}
  \label{eq_seg}
\end{equation}
where $\mathbf{s}_l$ and $\mathbf{s}_r$ are the left and right edges of rectangles, $y_l$ and $y_r$ are sorted integer indices of the $y$-coordinate.
These segments are then sorted by $x$-coordinate for traversing from left to right in sweep line algorithm.

\begin{algorithm}[h]
  \caption{Segment tree example for RMQ.}\label{alg:segtree}
  \textbf{Require}: Node $\mathbf{n}$ in the segment tree, segment endpoints $(y_l, y_r)$, and weight $w$
  \begin{algorithmic}[1]
    \STATE {\textsc{BUILD}}$(\mathbf{n}, y_l, y_r)$
    \STATE \hspace{0.5cm}$\mathbf{n}.y_l \leftarrow y_l$, $\mathbf{n}.y_r \leftarrow y_r$
    \STATE \hspace{0.5cm}$\mathbf{n}\mathrm{.val} \leftarrow 0$, $\mathbf{n}\mathrm{.lazy} \leftarrow 0$
    \STATE \hspace{0.5cm}\textbf{if} $y_l = y_r$ \textbf{then}
    \STATE \hspace{1.0cm}\textbf{return}
    \STATE \hspace{0.5cm}\textbf{end if}
    \STATE \hspace{0.5cm}$\mathrm{mid} \leftarrow \left\lfloor (y_l + y_r)/2 \right\rfloor$
    \STATE \hspace{0.5cm}{\textsc{BUILD}}$(\mathbf{n}\mathrm{.left}, y_l, \mathrm{mid})$
    \STATE \hspace{0.5cm}{\textsc{BUILD}}$(\mathbf{n}\mathrm{.right}, \mathrm{mid} + 1, y_r)$
    \STATE
    \STATE {\textsc{PUSHDOWN}}$(\mathbf{n})$
    \STATE \hspace{0.5cm}$ \mathbf{n}\mathrm{.left.val} \leftarrow \mathbf{n}\mathrm{.left.val} + \mathbf{n}\mathrm{.lazy}  $
    \STATE \hspace{0.5cm}$ \mathbf{n}\mathrm{.left.lazy} \leftarrow \mathbf{n}\mathrm{.left.lazy} + \mathbf{n}\mathrm{.lazy}  $
    \STATE \hspace{0.5cm}$ \mathbf{n}\mathrm{.right.val} \leftarrow \mathbf{n}\mathrm{.right.val} + \mathbf{n}\mathrm{.lazy}  $
    \STATE \hspace{0.5cm}$ \mathbf{n}\mathrm{.right.lazy} \leftarrow \mathbf{n}\mathrm{.right.lazy} + \mathbf{n}\mathrm{.lazy}  $
    \STATE \hspace{0.5cm}$ \mathbf{n}\mathrm{.lazy} \leftarrow 0  $
    \STATE
    \STATE {\textsc{PUSHUP}}$(\mathbf{n})$
    \STATE \hspace{0.5cm}$ \mathbf{n}\mathrm{.val} \leftarrow \max(\mathbf{n}\mathrm{.left.val} , \mathbf{n}\mathrm{.right.val})  $
    \STATE
    \STATE {\textsc{UPDATE}}$(\mathbf{n}, y_l, y_r, w)$
    \STATE \hspace{0.5cm}\textbf{if} $\mathbf{n}.y_l \ge y_l$ \textbf{and} $\mathbf{n}.y_r \le y_r$ \textbf{then}
    \STATE \hspace{1.0cm} $ \mathbf{n}\mathrm{.val} \leftarrow \mathbf{n}\mathrm{.val} + w $
    \STATE \hspace{1.0cm} $ \mathbf{n}\mathrm{.lazy} \leftarrow \mathbf{n}\mathrm{.lazy} + w $
    \STATE \hspace{1.0cm} \textbf{return}
    \STATE \hspace{0.5cm}\textsc{PUSHDOWN}$(\mathbf{n})$
    \STATE \hspace{0.5cm}$\mathrm{mid} \leftarrow \left\lfloor (\mathbf{n}.y_l + \mathbf{n}.y_r) / 2 \right\rfloor $
    \STATE \hspace{0.5cm}\textbf{if} $(y_l \le \mathrm{mid})$ \textbf{then}
    \STATE \hspace{1.0cm}\textsc{UPDATE}$(\mathbf{n}\mathrm{.left}, y_l, y_r, w)$
    \STATE \hspace{0.5cm}\textbf{end if}
    \STATE \hspace{0.5cm}\textbf{if} $(y_r > \mathrm{mid})$ \textbf{then}
    \STATE \hspace{1.0cm}\textsc{UPDATE}$(\mathbf{n}\mathrm{.right}, y_l, y_r, w)$
    \STATE \hspace{0.5cm}\textbf{end if}
    \STATE \hspace{0.5cm}\textsc{PUSHUP}$(\mathbf{n})$
  \end{algorithmic}
\end{algorithm}

\subsubsection{Segment tree with lazy propagation}

We present an example implementation of a segment tree spanning integer indices $[1, 2N]$ that maintains the maximum overlapping of weighted rectangles, as shown in Algorithm~\ref{alg:segtree}.
Each node of stores the sorted indices of segment endpoints $[y_l, y_r]$ in~\eqref{eq_seg}, current maximum weighted overlap ``$\mathrm{val}$'', and lazy tag ``$\mathrm{lazy}$''.
During sweep line traversal, lazy propagation accelerates overlap updates by deferring child node modifications until necessary.
This reduces update complexity to $O(\log N)$ per segment, as shown in Fig.~\ref{fig_sweepline}(b).

\begin{figure}
  \centering
  \subfloat[]{\includegraphics[width=0.5\columnwidth]{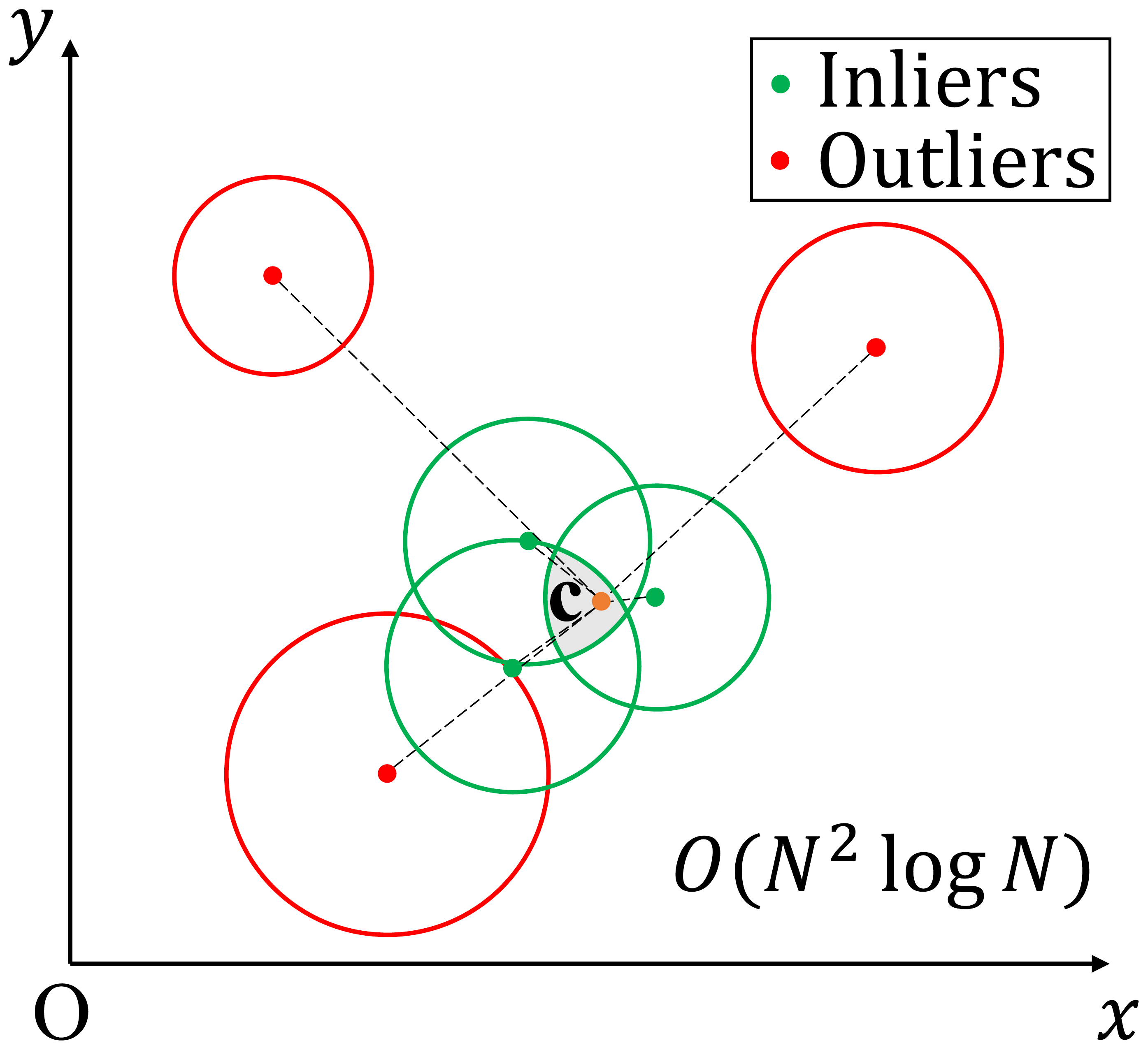}}
  \subfloat[]{\includegraphics[width=0.5\columnwidth]{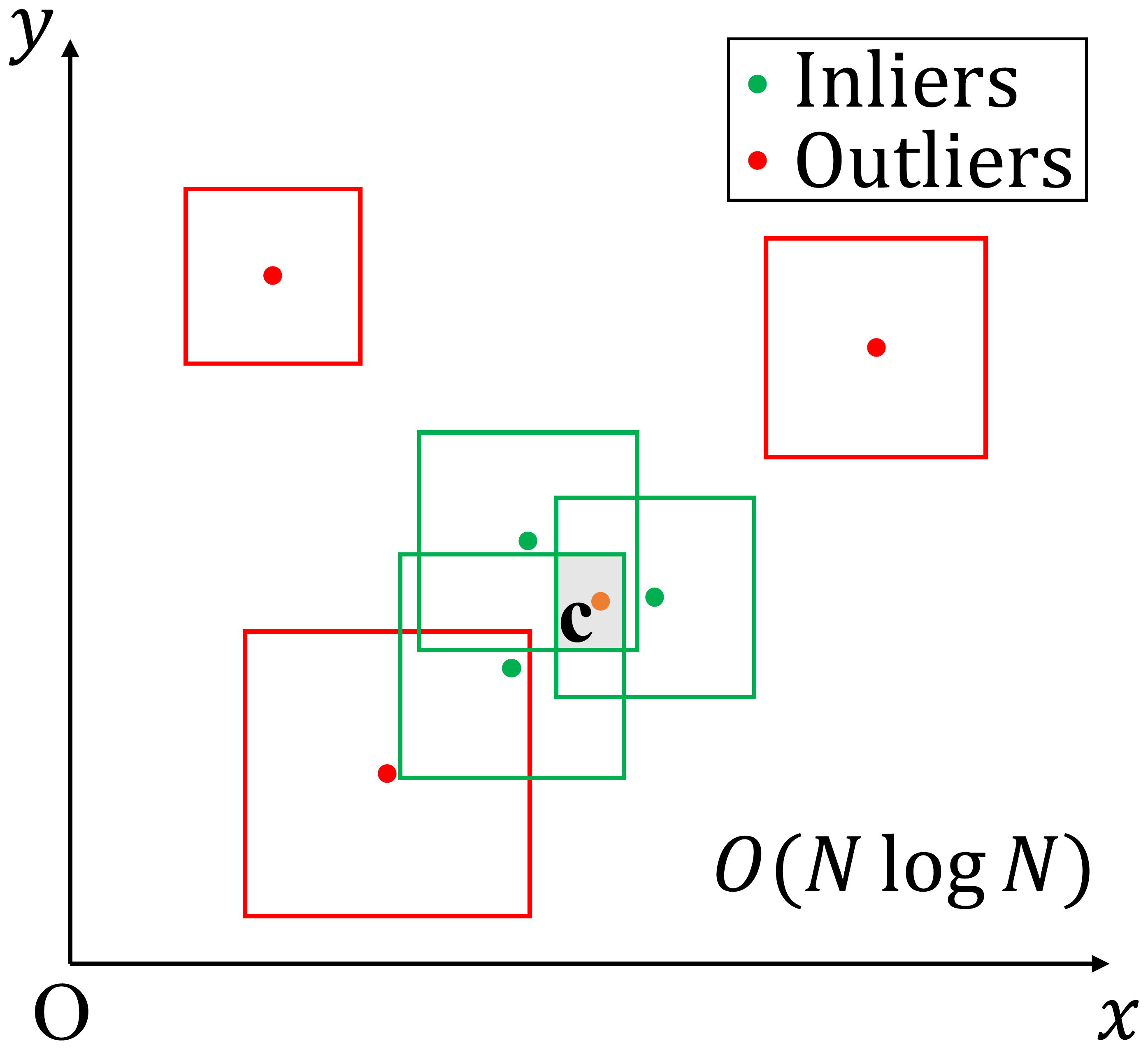}}
  \\
  \subfloat[]{\includegraphics[width=0.5\columnwidth]{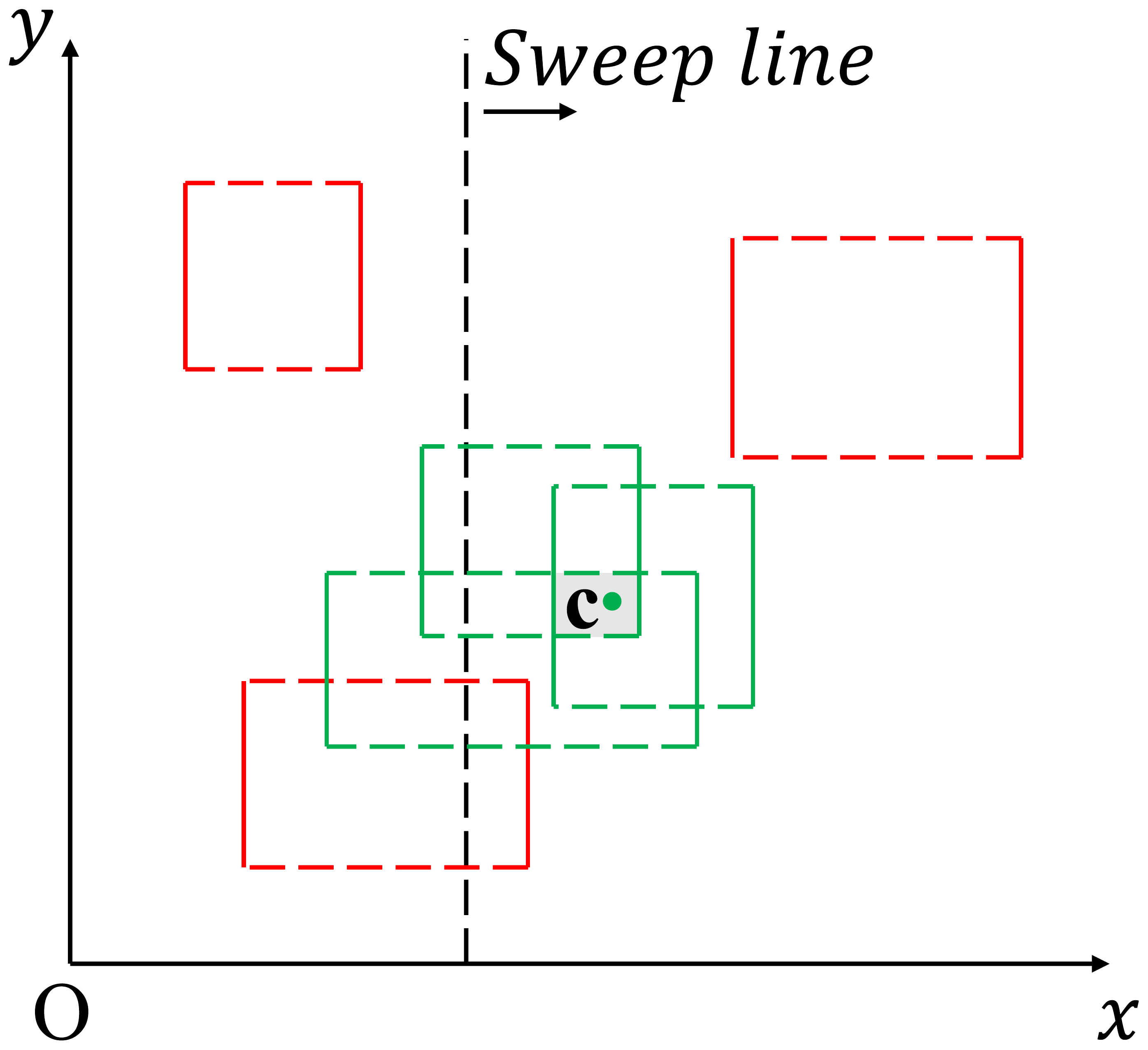}}
  \subfloat[]{\includegraphics[width=0.5\columnwidth]{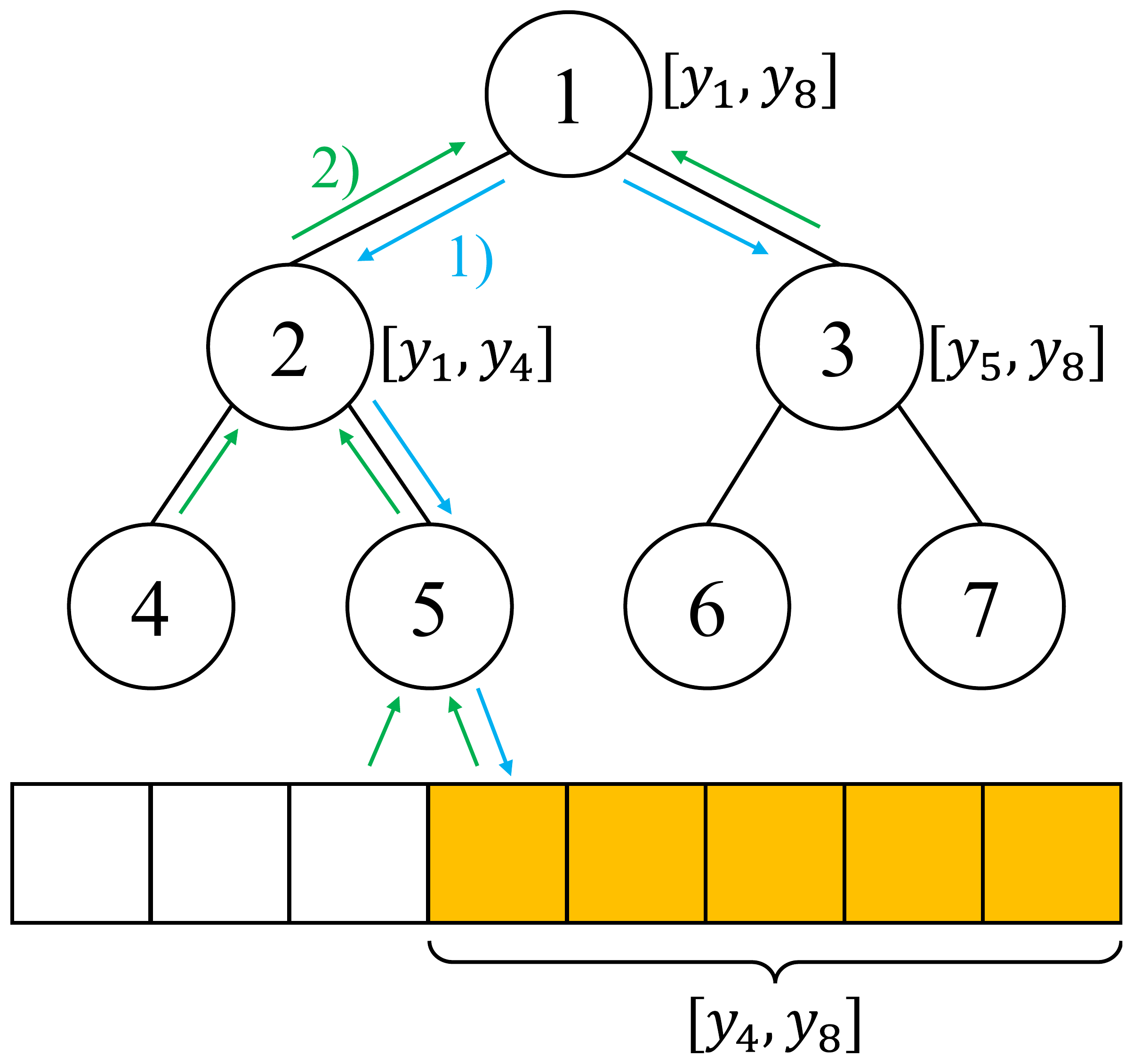}}
  \caption{Rotation center search using sweep line algorithm with segment tree.
    (a) 2D RMQ with circular bounding circles.
    (b) 2D RMQ with axis-aligned rectangles can be solved in $O(N \log N)$ time.
    (c) Sweep line from left to right for traversing the rectangles.
    (d) Add a new segment to the tree and update the maximum overlapping within $O(\log N)$ time
    (e.g. a segment $[y_4, y_8]$ is added and only the $1 \rightarrow 2 \rightarrow 5 \rightarrow y_4$ and $1 \rightarrow 3$ paths are updated).}
  \label{fig_sweepline}
\end{figure}

The root node of the segment tree maintains the maximum weighted overlap at the current sweep line position.
Querying this root node provides the maximum swept overlap in $O(1)$ time per iteration.
The complete sweep line procedure for maximum overlap of rectangles is detailed in Algorithm~\ref{alg:sweepline}.

For the current branch in~\eqref{eq_stage2}, we estimate the upper and lower weights' bounds via the two rectangle sets:
\begin{equation}
  \begin{aligned}
    \overline{W}_2  & = \mathrm{SweepLine} \left( \left\{
    \left( \underline{x}_{\mathbf{c}u},
    \underline{y}_{\mathbf{c}u},
    \overline{x}_{\mathbf{c}u},
    \overline{y}_{\mathbf{c}u} \right)
    \right\}_{i=1}^N \right) ,                            \\
    \underline{W}_2 & = \mathrm{SweepLine} \left( \left\{
    \left( \underline{x}_{\mathbf{c}l},
    \underline{y}_{\mathbf{c}l},
    \overline{x}_{\mathbf{c}l},
    \overline{y}_{\mathbf{c}l} \right)
    \right\}_{i=1}^N \right),
  \end{aligned}
  \label{eq_sweep}
\end{equation}
where $\mathrm{SweepLine}(\cdot)$ computes the maximum rectangle overlap in Algorithm~\ref{alg:sweepline}.
The pruning strategy of BnB search follows stage I in Sec.~\ref{sec:stage1}.
The searching dimension in stage II is compressed to 1D by the proposed RMQ equivalency, as shown in Fig.~\ref{fig_anglebnb}.

\begin{figure}
  \centering
  \includegraphics[width=\columnwidth]{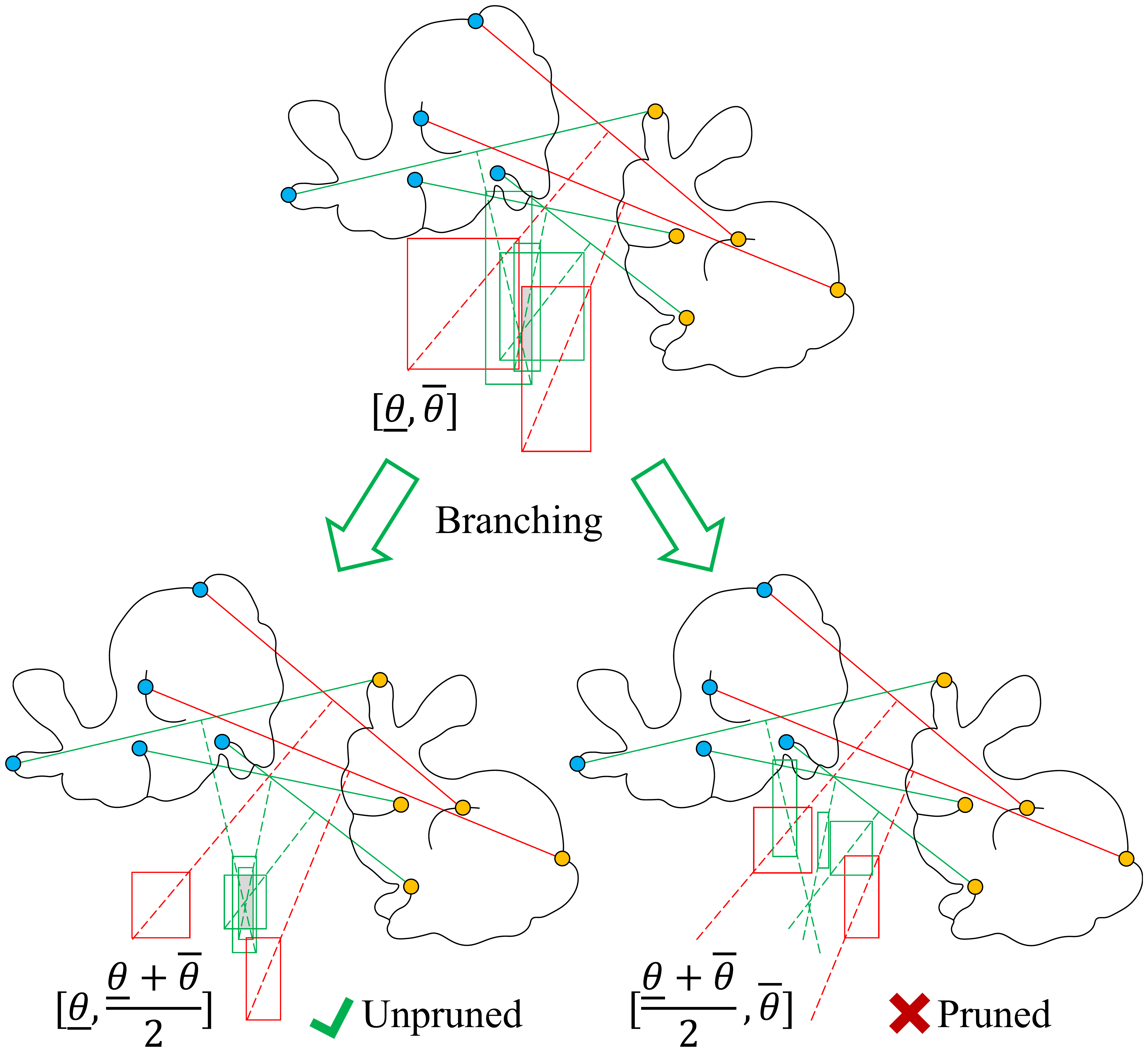}
  \caption{Branching strategy for the 1D rotation angle interval $[\underline{\theta}, \overline{\theta}]$.
    Green lines denote inlier correspondences, while red lines indicate outliers.
    After branching of the current node, sub-branches with low upper bounds are pruned.}
  \label{fig_anglebnb}
\end{figure}

\begin{algorithm}[h]
  \caption{Solving 2D RMQ by sweep line algorithm.}\label{alg:sweepline}
  \textbf{Input}: Bounding boxes $\{ (\underline{x}_i, \underline{y}_i, \overline{x}_i, \overline{y}_i) \}_{i=1}^N $, weights of correspondences $\{w_i\}_{i=1}^N$    \\
  \textbf{Output}: Maximum overlapping of weighted rectangles $W_2$
  \begin{algorithmic}[1]
    \STATE Sort $\{\underline{y}_i, \overline{y}_i\}_{i=1}^N$ to indices $\{y_{li}, y_{ri}\}_{i=1}^N$
    \STATE Initialize segments $\left\{\mathbf{s}_i\right\}_{i=1}^{2N}$
    \FOR{$i \leftarrow 1$ \textbf{to} $N$}
    \STATE $\mathbf{s}_{2i-1} \leftarrow (\underline{x}_i, y_{li}, y_{ri}, w_i)$
    \STATE $\mathbf{s}_{2i} \leftarrow (\overline{x}_i, y_{li}, y_{ri}, -w_i)$
    \ENDFOR
    \STATE Sort $\left\{\mathbf{s}_i\right\}_{i=1}^{2N}$ by $x$-coordinate
    \STATE Initialize segment tree $\mathbf{T}$, $W_2 \leftarrow 0$
    \STATE BUILD($\mathbf{T.root}, 1, 2N$)
    \FOR{$i \leftarrow 1$ \textbf{to} $2N$}
    \STATE \textsc{UPDATE}$(\mathbf{T.root}, \mathbf{s}_i.y_l, \mathbf{s}_i.y_r, \mathbf{s}_i.w)$
    \STATE \textbf{if} {$\mathbf{T.root}\mathrm{.val} > W_2$} \textbf{then}
    \STATE \hspace{0.5cm} $W_2 \leftarrow \mathbf{T.root}\mathrm{.val}$
    \STATE \textbf{end if}
    \ENDFOR
  \end{algorithmic}
\end{algorithm}

\subsection{Degeneration Induced by Translation-only Problem}

When rotation angle $\theta \in [-\epsilon/2, \epsilon/2]$, the rotation center $\mathbf{C}$ in~\eqref{eq_chasles_params_d} and noise threshold $\xi_\mathbf{c}(\theta)$ diverge, requiring specialized handling.

Under approximation of small rotation angle that $\mathbf{R} \approx \mathbf{I}$, all rotation axes in stage I are correct.
The points are filtered by difference vectors $\mathbf{q}_i - \mathbf{p}_i$ on the 2D plane perpendicular to the rotation axes,
where inliers should have similar difference vectors, as shown in Fig.~\ref{fig_rottrans}.
Due to the minimum branch width $\epsilon$, the arc can be approximated as a straight line perpendicular to $\mathbf{p}_i$.
Therefore, the corners of the bounding boxes are approximated using $\xi_2$ from~\eqref{eq_stage2} and~\eqref{eq_centersqr}.

\begin{figure}
  \centering
  \includegraphics[width=0.6\columnwidth]{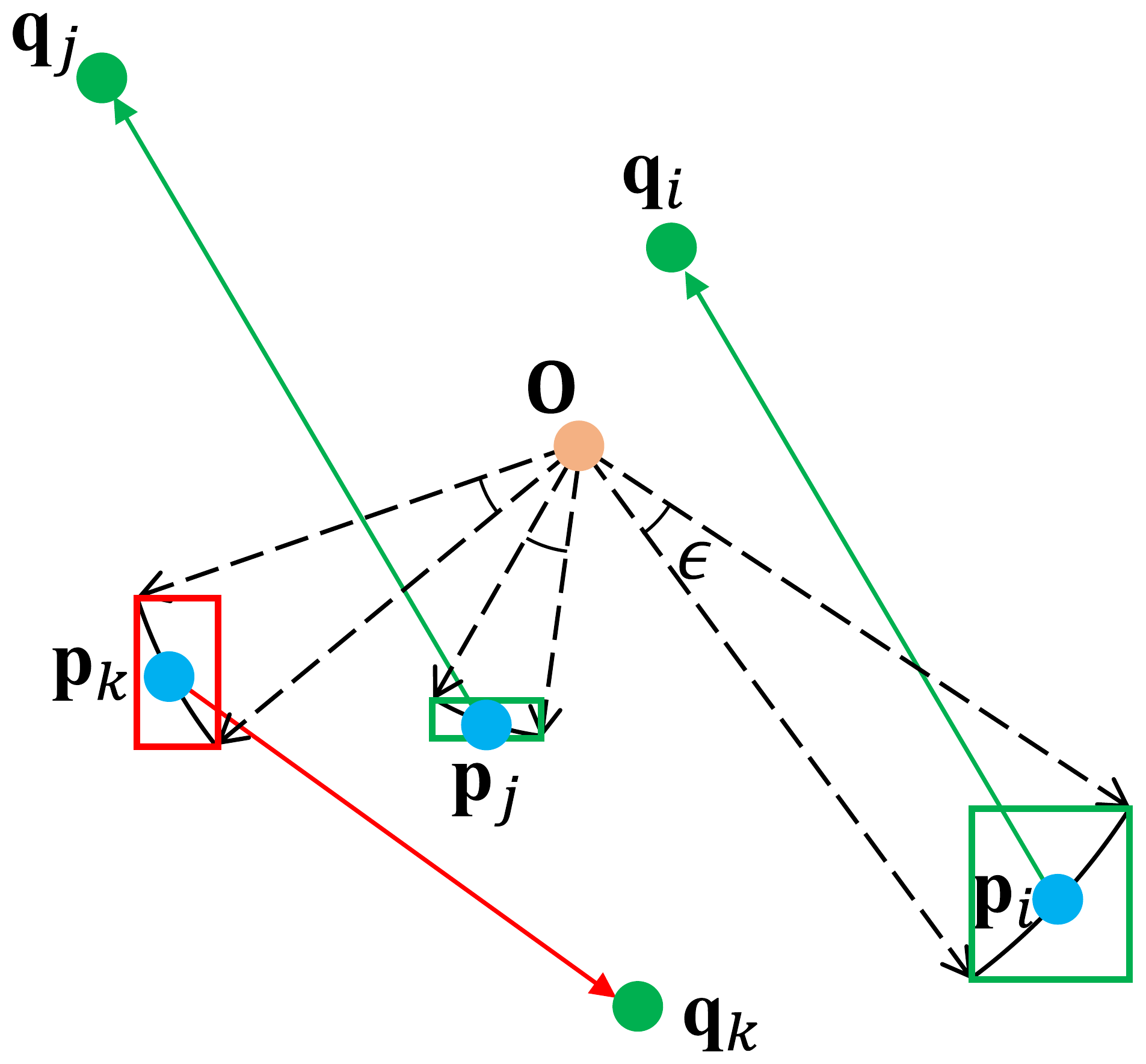}
  \caption{Translation-only problem with zero rotation angle and the bounding boxes.}
  \label{fig_rottrans}
\end{figure}

Before rotation search, Algorithm~\ref{alg:sweepline} solves this case as well, yielding maximum overlap as initial lower bound $\underline{W}_2^*$ for stage II BnB.

\subsection{Post-Refinement with IRLS}

After the two-stage BnB search, we refine the rigid transformation using weighted SVD on inliers\cite{choyDeep2020} via iterative reweighted least squares (IRLS).
The weights $\hat{w}$ are assigned using Tukey's biweight kernel\cite{demenezesreview2021}:
\begin{equation}
  \hat{w} =
  \begin{cases}
    \left( 1 - \dfrac{e_r^2}{c^2} \right)^2, & e_r < c    \\
    0,                                       & e_r \ge c,
  \end{cases}
  \label{eq_tukey}
\end{equation}
where $e_r$ is the residual error, $c = 4.685 \times \mathrm{MAD}$, and $\mathrm{MAD}$ is the median absolute deviation of inlier residual errors.

Adopting the scale-annealing estimator\cite{liPoint2021, shiRANSAC2024}, we filter inliers by $|e_r| \le \mu_i$ per iteration.
The threshold updates with annealing factor $\lambda = 0.5$ in the IRLS iteration as:
\begin{equation}
  \mu_{i+1} = (1 - \lambda) \mathrm{MAD} + \lambda \mu_i .
  \label{eq_anneal}
\end{equation}

The initial threshold $\mu_1 = \xi$ corresponds to~\eqref{eq_problem}, with maximum number of iterations set to 5.

\subsection{Subproblem of 4-DoF Registration}
\label{sec:4dof}

According to Fig.~\ref{fig_framework}, 4-DoF registration with a known rotation axis can be viewed as a constrained subproblem of the original 6-DoF rigid registration.
Under this setting, the proposed framework admits two alternative decomposition strategies, depending on the order in which the remaining DoF are resolved:

\begin{enumerate}
  \item \emph{1+3 DoF decomposition:}
        The translation along the known rotation axis is firstly estimated via interval stabbing to filter out outliers, similar in spirit to the approach of Li et al.~\cite{liTransformation2024}.
        Given the estimated translation, the remaining 2D rigid transformation is then solved using the proposed method.
  \item \emph{3+1 DoF decomposition:}
        The 2D rigid transformation on the plane perpendicular to the rotation axis is firstly estimated, together with outlier filtering on the projected plane.
        The remaining 1D translation along the rotation axis is subsequently recovered via interval stabbing.
\end{enumerate}

The two decomposition strategies exhibit different robustness characteristics:
(1) The 1+3 DoF decomposition relies on the consistency of 1D projected correspondences and is therefore sensitive to outlier distributions along the axis;
(2) The 3+1 DoF decomposition firstly constrains 3 DoF, leveraging stronger geometric structure in the projected plane, resulting in more informative residuals and a reduced ambiguity for the subsequent 1D translation estimation.
This ordering is expected to provide improved robustness under high outlier ratios or noisy correspondences, which is further experimentally examined in Sec.~\ref{sec:4dof_exp}.

\section{Space and Time Complexity}
\label{sec:time}

In stage I, BnB searches the 2D hemisphere from the best case with $O(\log{\frac{1}{\epsilon}})$ to the worst case with $O(\frac{1}{\epsilon^2})$ time, and the worst space complexity of nodes is also $O(\frac{1}{\epsilon^2})$.
In each branch node of rotation axis, the 1D interval stabbing of correspondences projected onto the axis is sorted in $O(N \log N)$ time,
and the maximum inliers of sorted points are filtered in $O(N)$ time.
The top-$k$ plane fitting consumes $O(kN)$ time.
Note that $k \ll \frac{1}{\epsilon}, N$, the time complexity of stage I is from $O(\log{\frac{1}{\epsilon}} \cdot N \log N)$ to $O(\frac{1}{\epsilon^2} N \log N)$.

In stage II with the $k$ rotation axes, BnB searches the 1D rotation angle from the best case with $O(k\log{\frac{1}{\epsilon}})$ to the worst case with $O(\frac{k}{\epsilon})$.
The segments are sorted in $O(N \log N)$ time, and maintained by a segment tree containing $O(N)$ nodes.
The sweep line algorithm traverses $O(N)$ segments by $x$, and the maximum overlapping of rectangles is updated in $O(\log N)$ time, leading to $O(N \log N)$ time complexity.
Therefore, the time complexity of stage II is from $O(k\log{\frac{1}{\epsilon}} \cdot N \log N)$ to $O(\frac{k}{\epsilon} N \log N)$.

In summary, the total time complexity with given number of correspondences $N$ and minimum branch width $\epsilon$ is from $O(k\log{\frac{1}{\epsilon}} \cdot N \log N)$ to $O(\frac{1}{\epsilon^2} N \log N)$, and the space complexity is $O(N + \frac{1}{\epsilon^2})$ in the worst case.

\section{Experimental Results}

\subsubsection{Datasets and Metrics}
We use real and synthetic datasets to compare our method with other baselines.
Specifically, the indoor 3DMatch and 3DLoMatch datasets~\cite{zeng3DMatch2017}, and outdoor KITTI dataset~\cite{geigerAre2012}, are employed to evaluate the real-data performance of the baseline methods.
Additionally, synthetic datasets generated from models of the Stanford 3D Scanning Repository~\cite{curlessvolumetric1996} are used in further experiments with arbitrary number of correspondences and outlier ratio.

Following~\cite{hartleyRotation2013}, the rotation error (RE) between the estimated rotation matrix $\hat{\mathbf{R}}$ and the ground truth $\mathbf{R}$ is computed using chordal distance.
The translation error (TE) is defined as the Euclidean distance between $\hat{\mathbf{t}}$ and $\mathbf{t}$:
\begin{equation}
  \begin{aligned}
    \mathrm{RE} & = \arccos \left(\frac{1}{2} \left(\mathrm{tr} \left( \hat{\mathbf{R}}^\top \mathbf{R} \right) - 1 \right) \right) \\
                & = 2 \arcsin \left(\frac{1}{2\sqrt{2}} \left\| \hat{\mathbf{R}} - \mathbf{R} \right\|_\mathrm{F} \right) ,         \\
    \mathrm{TE} & = \left\| \hat{\mathbf{t}} - \mathbf{t} \right\|_2 ,
  \end{aligned}
  \label{eq_RETE}
\end{equation}
where $\left\| \cdot \right\|_\mathrm{F}$ is Frobenius norm, and $\left\| \cdot \right\|_2$ is Euclidean distance.
A registration is successful if both RE and TE are below a certain threshold.
The ratio of successful registrations in all pairs is defined as the registration recall (RR).
In our experiments, we report the average RE and TE of successful pairs.

\subsubsection{Implementation details}

Our method is implemented in C++ using the Eigen library \cite{eigenweb} and accelerated by multi-threading.
The feature matching is performed using the Point Cloud Library (PCL)~\cite{rusu3D2011} and the nanoflann library~\cite{blancoNanoflann2014}.
The minimum branch width $\epsilon$ is set to 0.05 and the bound threshold $\epsilon_b$ is set to $10^{-3}$ for both our method and TR-DE~\cite{chenDeterministic2022}, while the convergence ratio $\rho$ is fixed at 0.25 for our method.
For all evaluated methods, we set the inlier distance threshold $\xi$ as twice the downsampling voxel size.
All experiments were executed on a laptop with an Intel Core i9-14900HX CPU and 16 GB RAM.

We compare our method with the following baselines on real and synthetic datasets:
1) RANSAC~\cite{fischlerRandom1981} and its variant TCF~\cite{shiRANSAC2024}, where RANSAC-1M and RANSAC-4M use a maximum of $10^6$ and $4 \times 10^6$ iterations respectively;
2) graph-based methods: TEASER++~\cite{yangTEASER2021}, SC2-PCR++~\cite{chenSC2PCR2023}, and MAC~\cite{zhang3D2023};
3) BnB-based methods: TEAR~\cite{huangScalable2024}, HERE~\cite{huangEfficient2024}, and TR-DE~\cite{chenDeterministic2022}.
Since TR-DE\cite{chenDeterministic2022} also adopts a two-stage BnB search but is closed-source, we re-implement it for fair comparisons.
The RANSAC implementation uses Open3D~\cite{zhouOpen3D2018} in Python with multi-threading acceleration, and convergence parameters match the official FCGF~\cite{choyFully2019} configuration.
As SC2-PCR++~\cite{chenSC2PCR2023} is GPU-optimized whereas others are CPU-based, we adapt its source code for CPU execution following \cite{huangEfficient2024}.
For multithreaded methods, the thread count is fixed at 12.

Moreover, implementations of RANSAC~\cite{fischlerRandom1981} in Open3D~\cite{zhouOpen3D2018}, TEASER++~\cite{yangTEASER2021}, SC2-PCR++~\cite{chenSC2PCR2023}, and MAC~\cite{zhang3D2023} provide their own feature matching pipelines.
We adopt the default parameter settings recommended by the respective authors to reproduce the reported results.
TEASER++~\cite{yangTEASER2021} optionally employs a cross-checker (CC) in feature space to reject ambiguous correspondences.
We observe that enabling this filter leads to a lower RR compared to configurations without CC.
Accordingly, we report the results of TEASER++~\cite{yangTEASER2021} both without and with the CC, denoted as TEASER++ and TEASER++(w/ CC), respectively.
For the remaining methods, feature correspondences are generated using the same preprocessing procedure described in Section~\ref{sec:preprocess} to ensure a fair comparison.

For the 4-DoF pose estimation subproblem under gravity priors, we evaluate both the proposed 1+3 and 3+1 decomposition strategies.
The comparisons include representative 4-DoF registration methods, namely BnB and FMP+BnB~\cite{caiPractical2019},  Quatro~\cite{limQuatro2024}, Li et al.~\cite{liTransformation2024}, and the correspondence-free 3D-BBS~\cite{aoki3DBBS2024}.
All 4-DoF methods are evaluated using the same ground-truth rotation axis.

\subsection{Indoor 3DMatch and 3DLoMatch Evaluation}

For the indoor 3DMatch and 3DLoMatch datasets, we evaluate the methods using feature descriptors extracted by:
traditional FPFH~\cite{rusuFast2009}, learning-based FCGF~\cite{choyFully2019}, and Predator~\cite{huangPREDATOR2021}.
Considering the distribution of different feature extraction methods in feature space, the preprocessing distance factor $d_\mathrm{f}$ is set to 0.01 for FPFH~\cite{rusuFast2009}, 0.07 for FCGF~\cite{choyFully2019} and 0.05 for Predator~\cite{huangPREDATOR2021},
and the number of kNN searched features $k_\mathrm{f}$ is set to 40 for FPFH~\cite{rusuFast2009}, 60 for FCGF~\cite{choyFully2019} and 30 for Predator~\cite{huangPREDATOR2021}, respectively.
The top-$k$ parameter defaults to 12 for our method.
The registration is successful if RE $<$ 15$^\circ$ and TE $<$ 30~cm.

\begin{figure*}[!htbp]
  \centering
  \footnotesize
  \begin{tabular}{cccccc}
    \normalsize Input pairs                   &
    \normalsize TCF\cite{shiRANSAC2024}       &
    \normalsize MAC\cite{zhang3D2023}         &
    \normalsize HERE\cite{huangEfficient2024} &
    \normalsize GMOR (Ours)                   &
    \normalsize Ground truth                                                            \\
    \makecell{
    \includegraphics[width=0.14\textwidth]{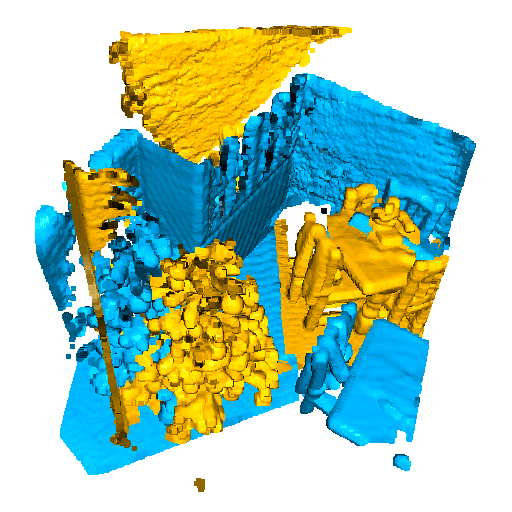}        \\
    Source: \#7536                                                                      \\
      Target: \#6456
    }                                         &
    \makecell{
    \includegraphics[width=0.14\textwidth]{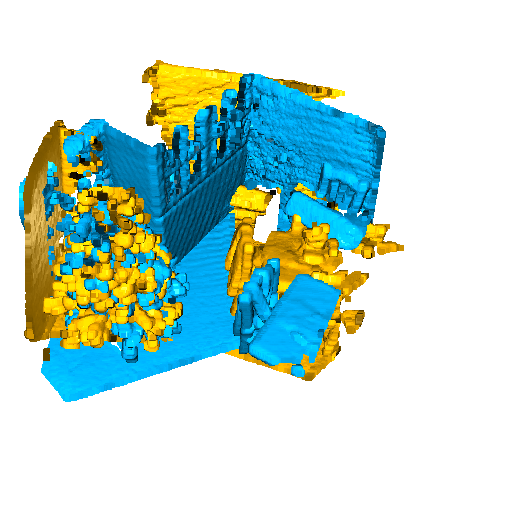}    \\
    \textcolor{red}{RE = 37.92$^{\circ}$}                                               \\
    \textcolor{red}{TE = 101.11cm}                                                      \\
    }                                         &
    \makecell{
    \includegraphics[width=0.14\textwidth]{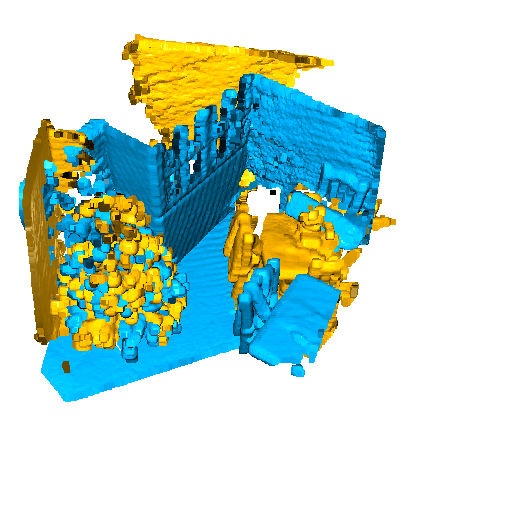}    \\
    \textcolor{red}{RE = 40.82$^{\circ}$}                                               \\
    \textcolor{red}{TE = 109.32cm}                                                      \\
    }                                         &
    \makecell{
    \includegraphics[width=0.14\textwidth]{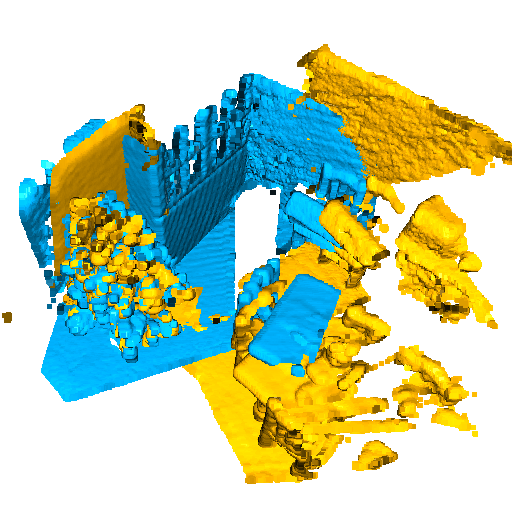}   \\
    RE = 10.25$^{\circ}$                                                                \\
    TE = 18.44cm                                                                        \\
    }                                         &
    \makecell{
    \includegraphics[width=0.14\textwidth]{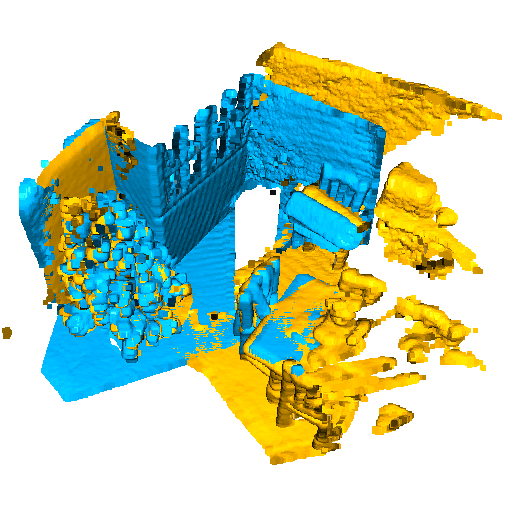}   \\
    RE = 1.22$^{\circ}$                                                                 \\
    TE = 5.04cm                                                                         \\
    }                                         &
    \makecell{
    \includegraphics[width=0.14\textwidth]{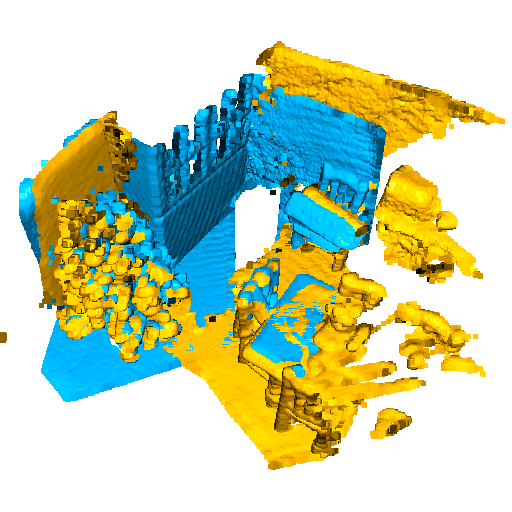}     \\
    \\
    \\
    }                                                                                   \\
    \makecell{
    \includegraphics[width=0.14\textwidth]{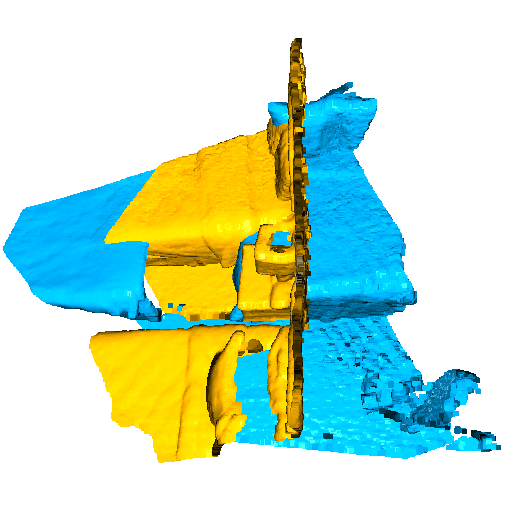}      \\
    Source: \#5024                                                                      \\
      Target: \#6095
    }                                         &
    \makecell{
    \includegraphics[width=0.14\textwidth]{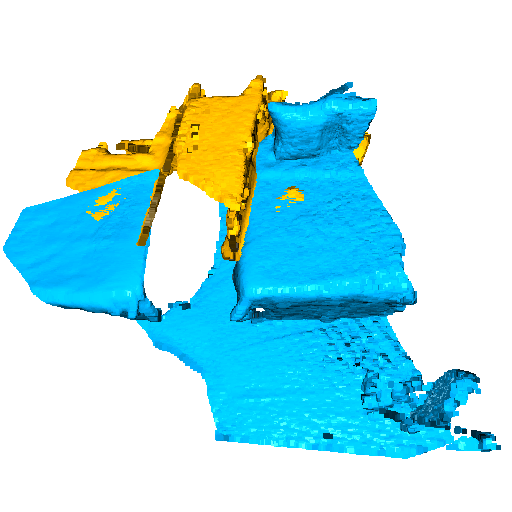}  \\
    \textcolor{red}{RE = 178.49$^{\circ}$}                                              \\
    \textcolor{red}{TE = 264.63cm}                                                      \\
    }                                         &
    \makecell{
    \includegraphics[width=0.14\textwidth]{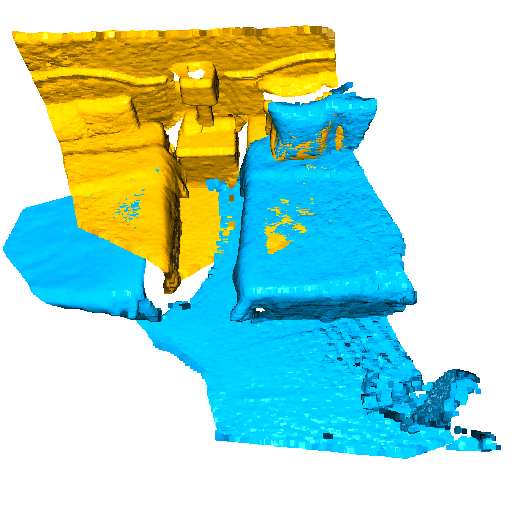}  \\
    RE = 2.67$^{\circ}$                                                                 \\
    TE = 15.43cm                                                                        \\
    }                                         &
    \makecell{
    \includegraphics[width=0.14\textwidth]{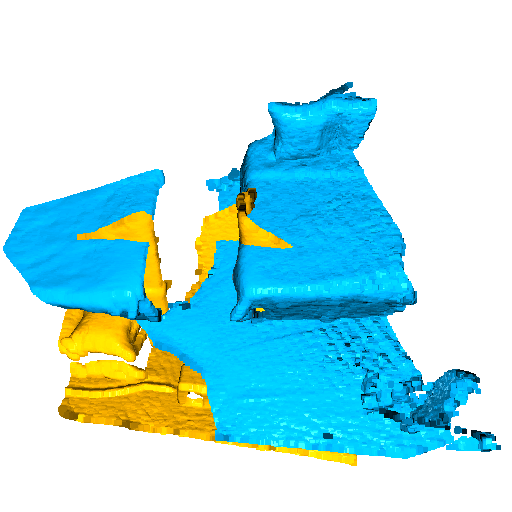} \\
    \textcolor{red}{RE = 173.72$^{\circ}$}                                              \\
    \textcolor{red}{TE = 185.32cm}                                                      \\
    }                                         &
    \makecell{
    \includegraphics[width=0.14\textwidth]{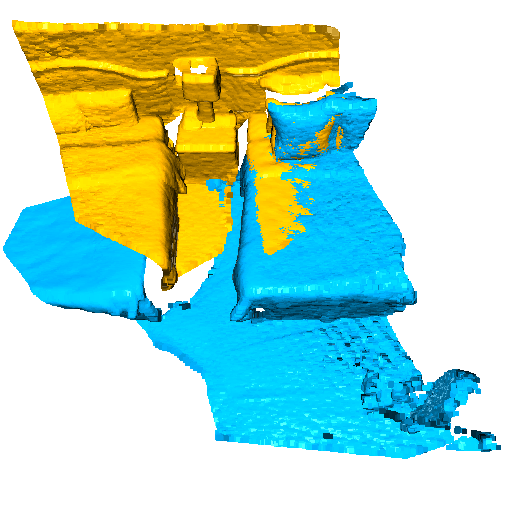} \\
    RE = 3.80$^{\circ}$                                                                 \\
    TE = 17.38cm                                                                        \\
    }                                         &
    \makecell{
    \includegraphics[width=0.14\textwidth]{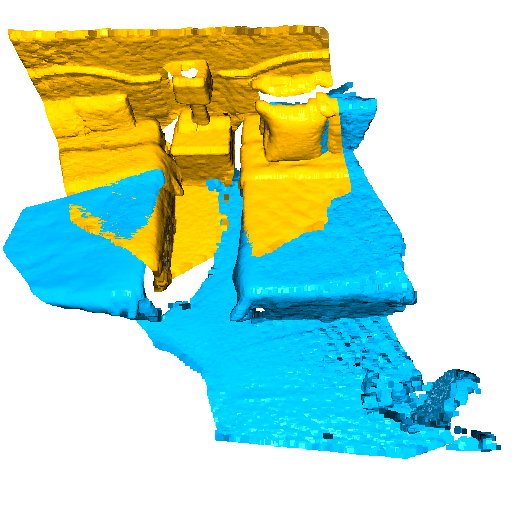}   \\
    \\
    \\
    }                                                                                   \\
    \makecell{
    \includegraphics[width=0.14\textwidth]{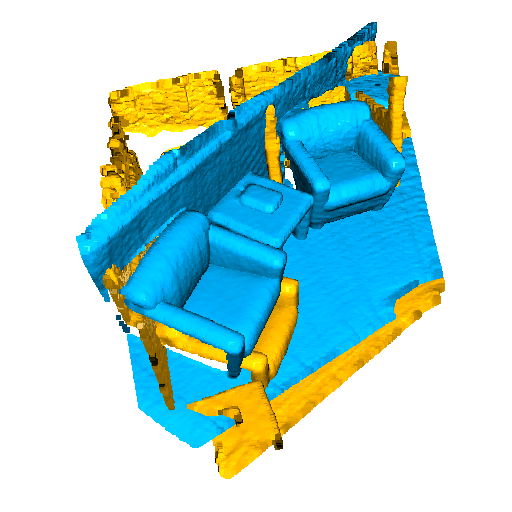}       \\
    Source: \#9249                                                                      \\
      Target: \#5303
    }                                         &
    \makecell{
    \includegraphics[width=0.14\textwidth]{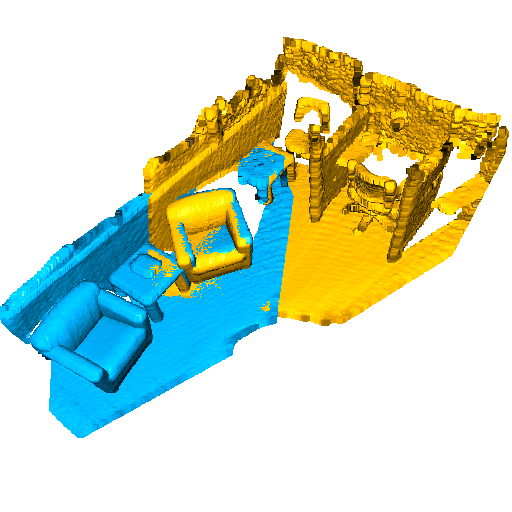}   \\
    RE = 1.17$^{\circ}$                                                                 \\
    TE = 4.90cm                                                                         \\
    }                                         &
    \makecell{
    \includegraphics[width=0.14\textwidth]{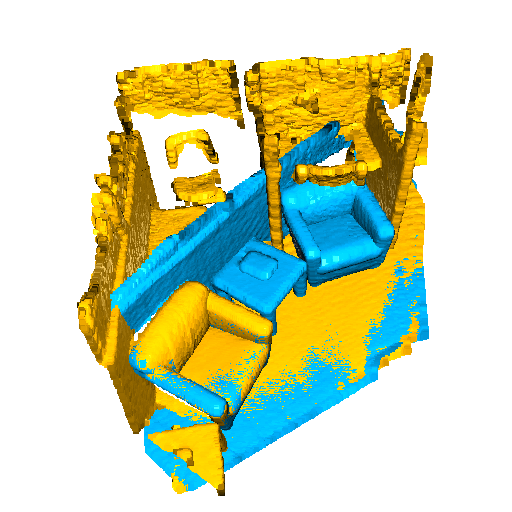}   \\
    \textcolor{red}{RE = 44.36$^{\circ}$}                                               \\
    \textcolor{red}{TE = 78.45cm}                                                       \\
    }                                         &
    \makecell{
    \includegraphics[width=0.14\textwidth]{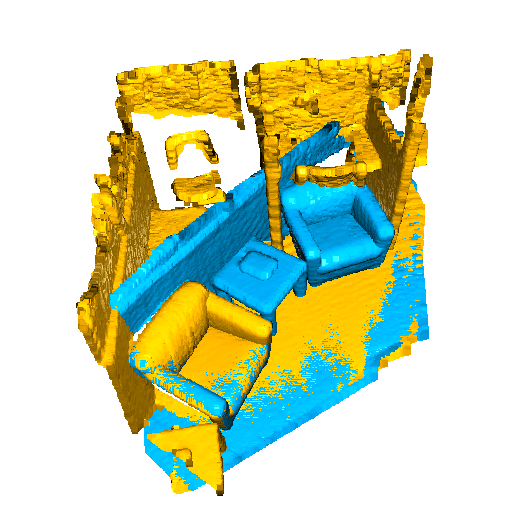}  \\
    \textcolor{red}{RE = 44.28$^{\circ}$}                                               \\
    \textcolor{red}{TE = 78.61cm}                                                       \\
    }                                         &
    \makecell{
    \includegraphics[width=0.14\textwidth]{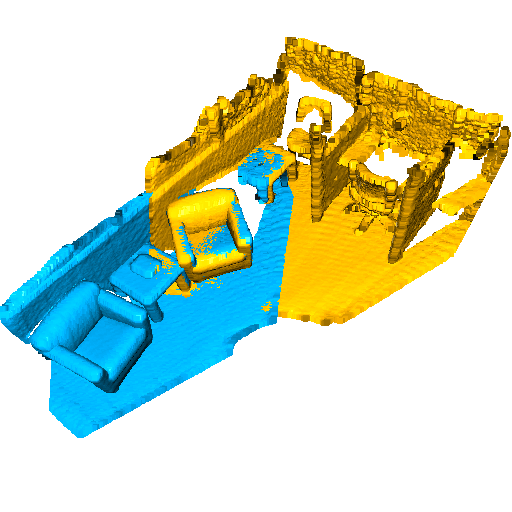}  \\
    RE = 1.60$^{\circ}$                                                                 \\
    TE = 5.12cm                                                                         \\
    }                                         &
    \makecell{
    \includegraphics[width=0.14\textwidth]{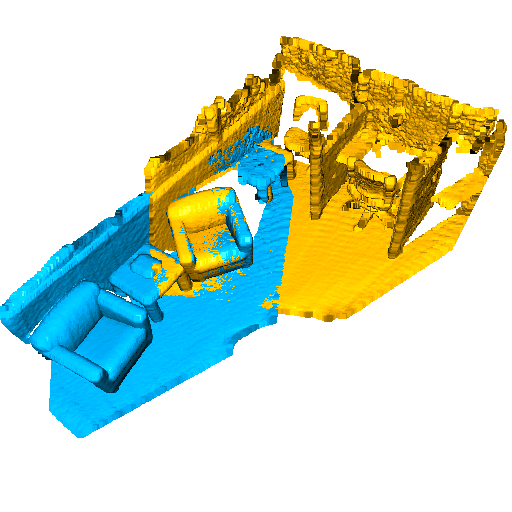}    \\
    \\
    \\
    }                                                                                   \\
    \makecell{
    \includegraphics[width=0.14\textwidth]{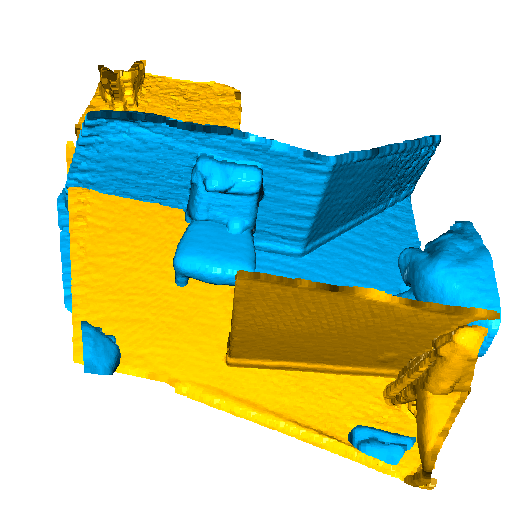}        \\
    Source: \#3643                                                                      \\
      Target: \#5163
    }                                         &
    \makecell{
    \includegraphics[width=0.14\textwidth]{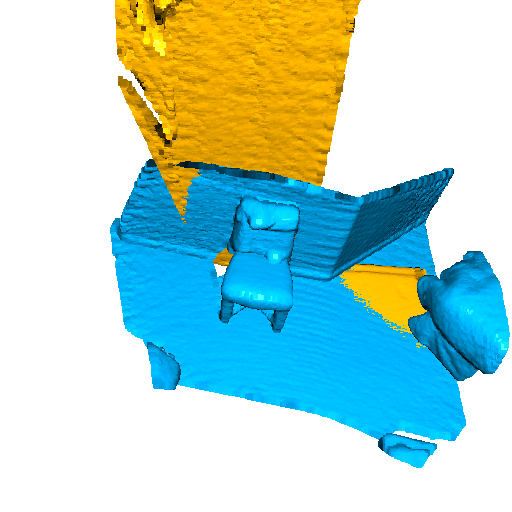}    \\
    \textcolor{red}{RE = 95.18$^{\circ}$}                                               \\
    \textcolor{red}{TE = 293.69cm}                                                      \\
    }                                         &
    \makecell{
    \includegraphics[width=0.14\textwidth]{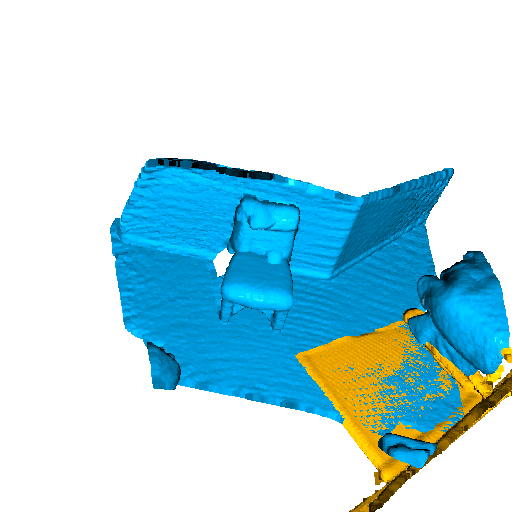}    \\
    \textcolor{red}{RE = 97.17$^{\circ}$}                                               \\
    \textcolor{red}{TE = 152.01cm}                                                      \\
    }                                         &
    \makecell{
    \includegraphics[width=0.14\textwidth]{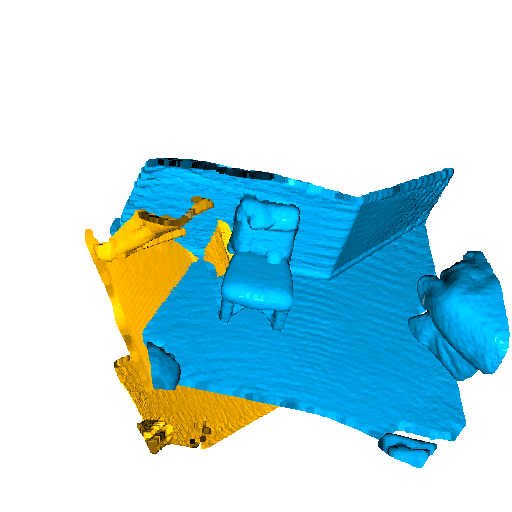}   \\
    \textcolor{red}{RE = 133.85$^{\circ}$}                                              \\
    \textcolor{red}{TE = 149.60cm}                                                      \\
    }                                         &
    \makecell{
    \includegraphics[width=0.14\textwidth]{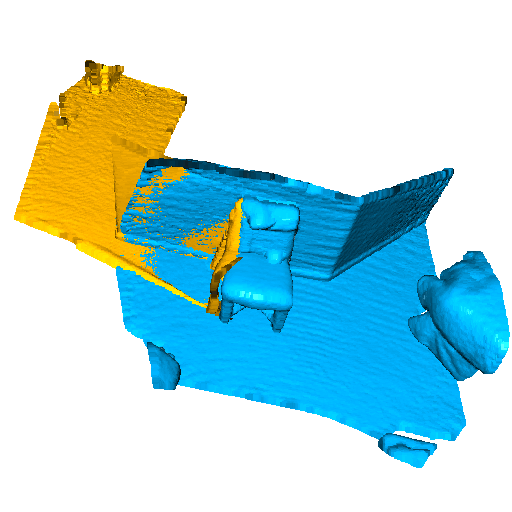}   \\
    RE = 4.84$^{\circ}$                                                                 \\
    TE = 5.38cm                                                                         \\
    }                                         &
    \makecell{
    \includegraphics[width=0.14\textwidth]{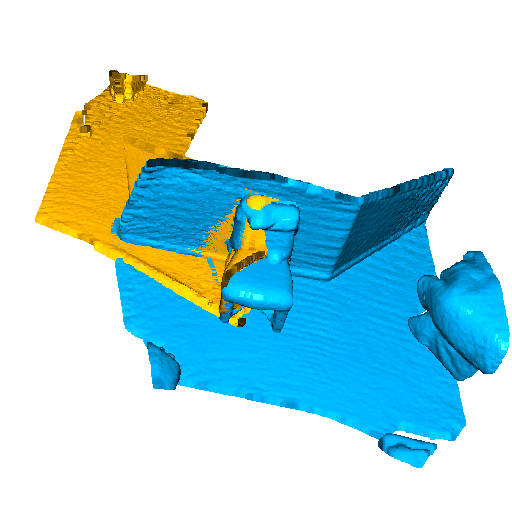}     \\
    \\
    \\
    }
  \end{tabular}
  \caption{Visualization of RANSAC-based TCF~\cite{shiRANSAC2024}, graph-based MAC~\cite{zhang3D2023}, BnB-based HERE~\cite{huangEfficient2024}, and our method on the indoor 3DMatch/3DLoMatch datasets using FPFH descriptor~\cite{rusuFast2009}. ``\#'' is the number of points, and \textcolor{red}{red} text represents the failure case.}
  \label{fig_3dmatch}
\end{figure*}

\subsubsection{3DMatch Dataset}

We downsample the point clouds by 5~cm voxel size following~\cite{choyFully2019} and evaluate the methods on 3DMatch dataset with FPFH~\cite{rusuFast2009} and FCGF~\cite{choyFully2019} descriptors.
Table~\ref{tab:3DMatch} shows the results.

With FPFH~\cite{rusuFast2009} descriptors, our method outperforms all baselines by over 2\% in terms of RR.
Both MAC~\cite{zhang3D2023} and SC2-PCR++~\cite{chenSC2PCR2023} achieve suboptimal RR, but require 7$\times$ and 60$\times$ more computation time than our method, respectively.
The re-implemented TR-DE~\cite{chenDeterministic2022} is the fastest and more efficient than the original paper, but it is less accurate than recent SOTA methods except for TEASER++~\cite{yangTEASER2021} and TEAR~\cite{huangScalable2024}.

\begin{table}[!htbp]
  \caption{Quantitive Results on 3DMatch Dataset.\label{tab:3DMatch}}
  \centering
  \begin{tabular}{c|ccc|c}
    \hline
                                               & \multicolumn{3}{c|}{FPFH (traditional)}    &                                                                          \\
                                               & RR(\%)$\uparrow$                           & RE($^{\circ}$)$\downarrow$ & TE(cm)$\downarrow$ & Time(s)$\downarrow$    \\
    \hline
    RANSAC-1M\cite{fischlerRandom1981}         & 62.66                                      & 3.72                       & 10.74              & 0.87                   \\
    RANSAC-4M\cite{fischlerRandom1981}         & 71.97                                      & 3.73                       & 10.95              & 1.94                   \\
    TCF\cite{shiRANSAC2024}                    & 86.51                                      & 2.66                       & 7.88               & 9.46                   \\
    TEASER++\cite{yangTEASER2021}              & 80.53                                      & 3.38                       & 10.87              & 12.15                  \\
    TEASER++(w/ CC)\cite{yangTEASER2021}       & 75.05                                      & 3.56                       & 10.86              & \textcolor{gray}{0.04} \\
    SC2-PCR++\cite{chenSC2PCR2023}             & \underline{87.31}                          & \textbf{2.12}              & \textbf{6.68}      & 48.33                  \\
    MAC\cite{zhang3D2023}                      & \underline{87.31}                          & 2.37                       & 7.36               & 5.09                   \\
    TEAR\cite{huangScalable2024}               & 61.44                                      & 2.80                       & 8.06               & 0.72                   \\
    HERE\cite{huangEfficient2024}              & 85.21                                      & 2.32                       & \underline{7.05}   & 0.73                   \\
    TR-DE$^{\ast}$\cite{chenDeterministic2022} & 83.49                                      & 2.55                       & 7.75               & \textbf{0.38}          \\
    GMOR (Ours)                                & \textbf{89.46}                             & \underline{2.29}           & 7.13               & \underline{0.62}       \\
    \hline
                                               & \multicolumn{3}{c|}{FCGF (learning-based)} &                                                                          \\
                                               & RR(\%)$\uparrow$                           & RE($^{\circ}$)$\downarrow$ & TE(cm)$\downarrow$ & Time(s)$\downarrow$    \\
    \hline
    RANSAC-1M\cite{fischlerRandom1981}         & 84.53                                      & 3.40                       & 10.63              & \underline{0.30}       \\
    RANSAC-4M\cite{fischlerRandom1981}         & 84.78                                      & 3.42                       & 10.82              & 0.52                   \\
    TCF\cite{shiRANSAC2024}                    & 89.59                                      & 2.25                       & 6.98               & 33.25                  \\
    TEASER++\cite{yangTEASER2021}              & \multicolumn{3}{c|}{Out-of-Memory}         &                                                                          \\
    TEASER++(w/ CC)\cite{yangTEASER2021}       & 85.03                                      & 2.99                       & 9.51               & \textcolor{gray}{0.08} \\
    SC2-PCR++\cite{chenSC2PCR2023}             & \textbf{94.15}                             & 2.04                       & 6.50               & 51.30                  \\
    MAC\cite{zhang3D2023}                      & 92.42                                      & 2.05                       & 6.41               & 19.96                  \\
    TEAR\cite{huangScalable2024}               & 83.36                                      & 2.31                       & 6.83               & 0.95                   \\
    HERE\cite{huangEfficient2024}              & 90.63                                      & \underline{2.02}           & \underline{6.32}   & 0.71                   \\
    TR-DE$^{\ast}$\cite{chenDeterministic2022} & 91.44                                      & 2.15                       & 6.84               & \textbf{0.25}          \\
    GMOR (Ours)                                & \underline{93.22}                          & \textbf{1.90}              & \textbf{6.16}      & 0.50                   \\
    \hline
    \multicolumn{5}{c}{$^{\ast}$TR-DE runs with the reproduced code of original work\cite{chenDeterministic2022}.}
  \end{tabular}
\end{table}

Due to the large number of correspondences generated by FCGF~\cite{choyFully2019}, TEASER++~\cite{yangTEASER2021} requires over 2 hours and runs out-of-memory to compute the maximum clique in some cases.
SC2-PCR++~\cite{chenSC2PCR2023} benefits from its one-to-many feature matching, which increases the likelihood of resampling correct correspondences.
This approach reduces FCGF's~\cite{choyFully2019} generalization error influence but incurs higher computational costs.
Notably, SC2-PCR++~\cite{chenSC2PCR2023} achieves higher performance with learning-based descriptors than with traditional descriptors.
This performance difference occurs because SC2-PCR++~\cite{chenSC2PCR2023} constructs feature graphs using normalized features, losing information from unnormalized FPFH~\cite{rusuFast2009} features during consensus set selection.
Our method achieves suboptimal RR and the lowest RE and TE, with significantly reduced computation time compared to SC2-PCR++~\cite{chenSC2PCR2023} on identical hardware.

\subsubsection{3DLoMatch Dataset}

The 3DLoMatch dataset is a challenging benchmark for registration with low overlap firstly proposed by Predator~\cite{huangPREDATOR2021}.
We evaluate the compared methods using FPFH~\cite{rusuFast2009} and Predator~\cite{huangPREDATOR2021} descriptors.
The pre-trained weights provided by Predator~\cite{huangPREDATOR2021} employ 2.5~cm voxel size downsampling, differing from the 5~cm voxel size used for other descriptors.
Following~\cite{chenSC2PCR2023}, we downsample each point cloud to 5000 points using the saliency score of Predator~\cite{huangPREDATOR2021}.

\begin{table}[!htbp]
  \caption{Quantitive Results on 3DLoMatch Dataset.\label{tab:3DLoMatch}}
  \centering
  \begin{tabular}{c|ccc|c}
    \hline
                                         & \multicolumn{3}{c|}{FPFH (traditional)}        &                                                                          \\
                                         & RR(\%)$\uparrow$                               & RE($^{\circ}$)$\downarrow$ & TE(cm)$\downarrow$ & Time(s)$\downarrow$    \\
    \hline
    RANSAC-1M\cite{fischlerRandom1981}   & 10.67                                          & 5.52                       & 12.50              & \underline{0.65}       \\
    RANSAC-4M\cite{fischlerRandom1981}   & 15.61                                          & 5.50                       & 13.59              & 2.38                   \\
    TCF\cite{shiRANSAC2024}              & 43.96                                          & 4.28                       & 12.63              & 3.76                   \\
    TEASER++\cite{yangTEASER2021}        & 33.86                                          & 4.97                       & 14.79              & 1.86                   \\
    TEASER++(w/ CC)\cite{yangTEASER2021} & 28.80                                          & 5.41                       & 14.82              & \textcolor{gray}{0.01} \\
    SC2-PCR++\cite{chenSC2PCR2023}       & 41.66                                          & \underline{3.85}           & \textbf{10.06}     & 31.19                  \\
    MAC\cite{zhang3D2023}                & \underline{49.13}                              & 4.00                       & 12.22              & 2.12                   \\
    TEAR\cite{huangScalable2024}         & 13.76                                          & 4.33                       & 12.18              & 1.75                   \\
    HERE\cite{huangEfficient2024}        & 38.97                                          & \textbf{3.84}              & \underline{11.52}  & 0.82                   \\
    TR-DE\cite{chenDeterministic2022}    & 43.12                                          & 4.09                       & 12.84              & \textbf{0.47}          \\
    GMOR (Ours)                          & \textbf{50.59}                                 & 3.89                       & 12.01              & \underline{0.65}       \\
    \hline
                                         & \multicolumn{3}{c|}{Predator (learning-based)} &                                                                          \\
                                         & RR(\%)$\uparrow$                               & RE($^{\circ}$)$\downarrow$ & TE(cm)$\downarrow$ & Time(s)$\downarrow$    \\
    \hline
    RANSAC-1M\cite{fischlerRandom1981}   & 65.36                                          & 3.58                       & 10.47              & \underline{0.46}       \\
    RANSAC-4M\cite{fischlerRandom1981}   & 66.59                                          & 3.53                       & 10.28              & 1.35                   \\
    TCF\cite{shiRANSAC2024}              & 66.37                                          & 3.25                       & 10.50              & 36.32                  \\
    TEASER++\cite{yangTEASER2021}        & \multicolumn{3}{c|}{Out-of-Memory}             &                                                                          \\
    TEASER++(w/ CC)\cite{yangTEASER2021} & 61.76                                          & 3.54                       & 10.57              & \textcolor{gray}{0.03} \\
    SC2-PCR++\cite{chenSC2PCR2023}       & \textbf{71.53}                                 & 3.49                       & \underline{9.78}   & 43.46                  \\
    MAC\cite{zhang3D2023}                & 69.17                                          & 3.42                       & 10.47              & 30.24                  \\
    TEAR\cite{huangScalable2024}         & 54.41                                          & 3.24                       & 10.09              & 1.27                   \\
    HERE\cite{huangEfficient2024}        & 64.18                                          & \textbf{2.98}              & \textbf{9.75}      & 1.17                   \\
    TR-DE\cite{chenDeterministic2022}    & 64.85                                          & 3.06                       & 10.30              & \textbf{0.36}          \\
    GMOR (Ours)                          & \underline{69.68}                              & \underline{3.02}           & 10.01              & 0.51                   \\
    \hline
  \end{tabular}
\end{table}

Table~\ref{tab:3DLoMatch} shows the results on the 3DLoMatch dataset.
With FPFH~\cite{rusuFast2009} descriptors, both MAC~\cite{zhang3D2023} and our method surpass other approaches significantly, resulting from their ability to find multiple candidate solutions under low overlap ratios.
Besides, maximal clique-based TEASER++~\cite{yangTEASER2021} and MAC~\cite{zhang3D2023} exhibit faster performance on 3DLoMatch than on 3DMatch.
Because the SC graph of correspondences grows sparser given 3DLoMatch's lower inlier ratios.
In contrast, the performance of SC2-PCR++~\cite{chenSC2PCR2023} underperforms on 3DLoMatch relative to 3DMatch though it is also graph-based method.
Its SOG sampling yields numerous invalid samples due to FPFH's~\cite{rusuFast2009} low inlier ratios.
Our method achieves the highest RR and the suboptimal time consumption.
It registers the input point cloud pairs within a stable time (0.62~s on 3DMatch and 0.65~s on 3DLoMatch dataset), regardless of the overlap ratio.
In addition, the visualization of the registration results using FPFH~\cite{rusuFast2009} descriptor is shown in Fig.~\ref{fig_3dmatch}.
Our method maintains accuracy even in challenging cases with repetitive geometries and low overlap.

The results using Predator~\cite{huangPREDATOR2021} are similar to those of FCGF~\cite{choyFully2019} descriptor.
Downsampling correspondences to 5000 points using Predator's~\cite{huangPREDATOR2021}, saliency scores significantly increases the inlier ratio compared to whole point clouds, as verified in~\cite{huangPREDATOR2021}.
The RR of our method is also the suboptimal with 0.51~s time consumption, comparable to FCGF's~\cite{choyFully2019} 0.50~s.
A more comprehensive ablation study is elaborated in the next section.

\subsubsection{Ablation Study on 3DMatch/3DLoMatch}

\begin{figure}[!htbp]
  \centering
  \subfloat[]{\includegraphics[width=0.49\columnwidth]{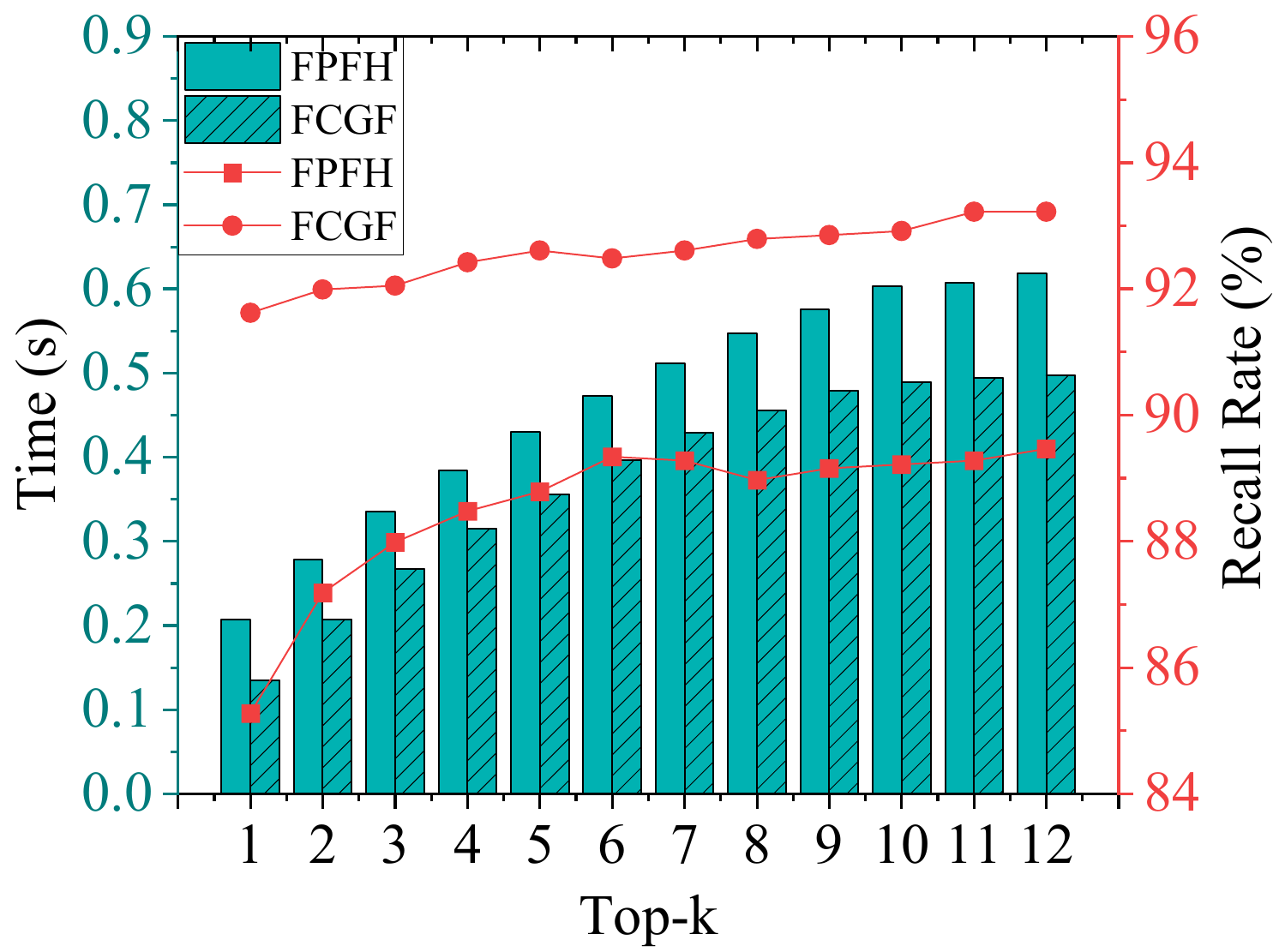}}
  \hfill
  \subfloat[]{\includegraphics[width=0.49\columnwidth]{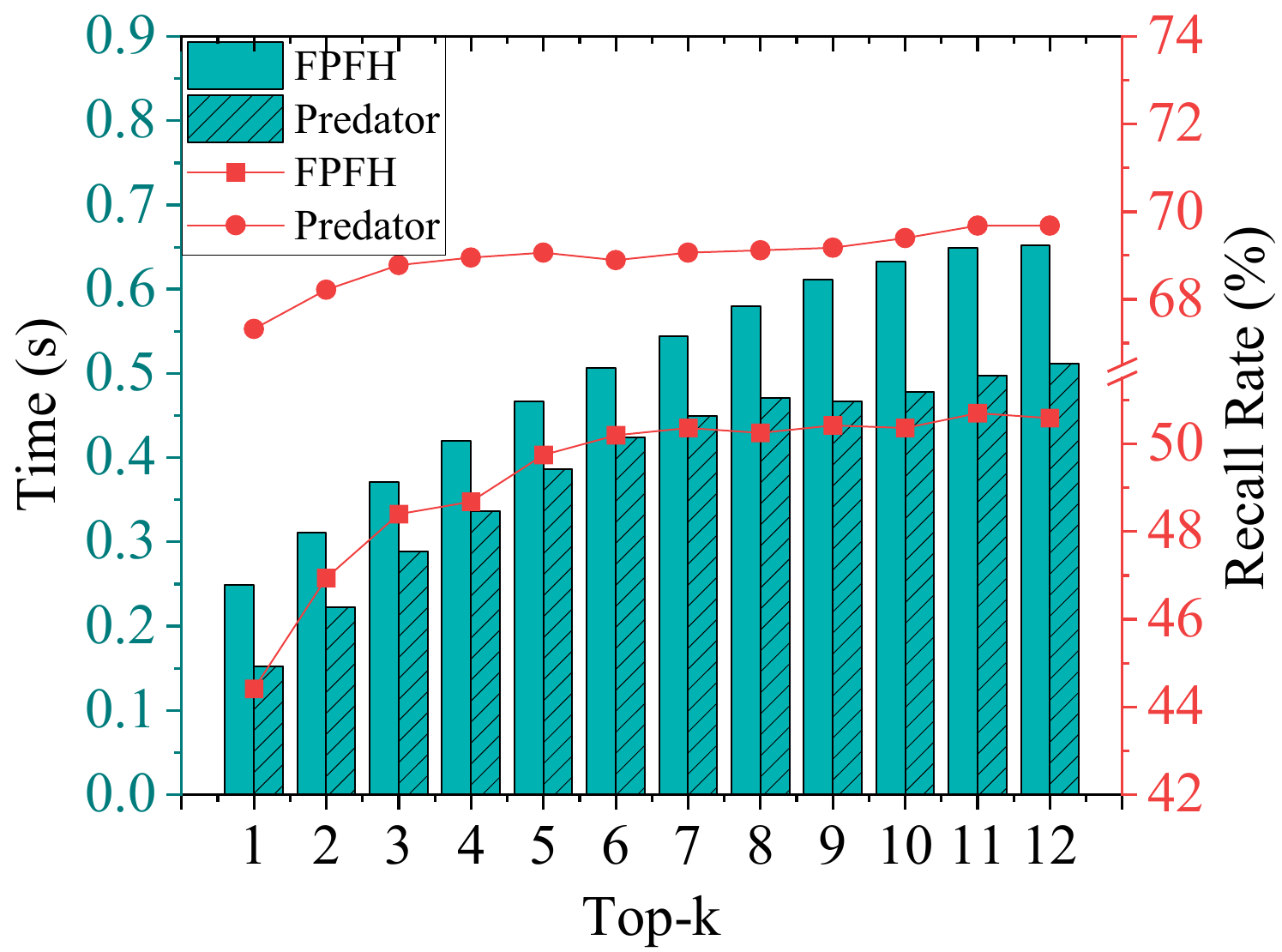}}
  \\
  \subfloat[]{\includegraphics[width=0.49\columnwidth]{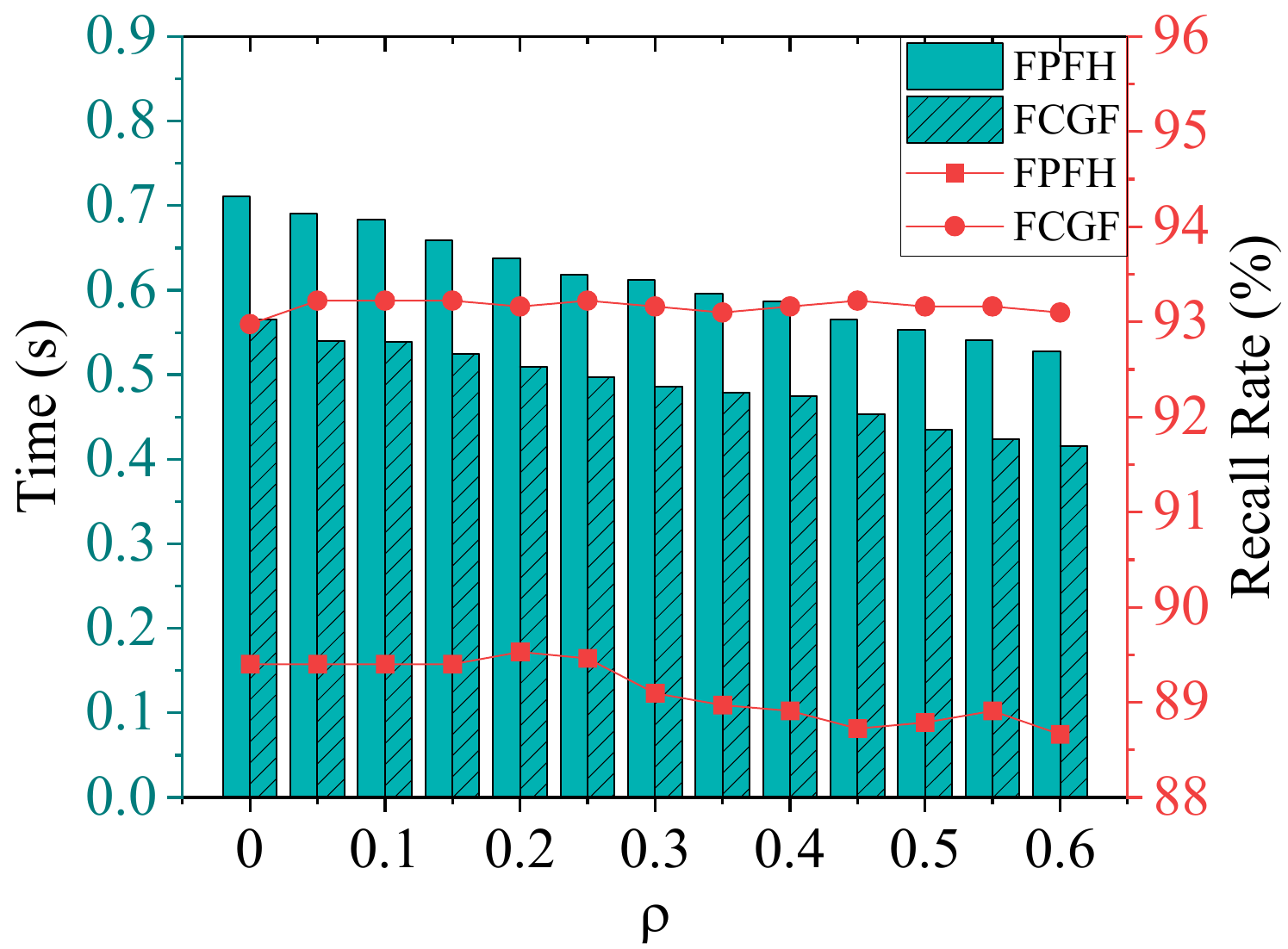}}
  \hfill
  \subfloat[]{\includegraphics[width=0.49\columnwidth]{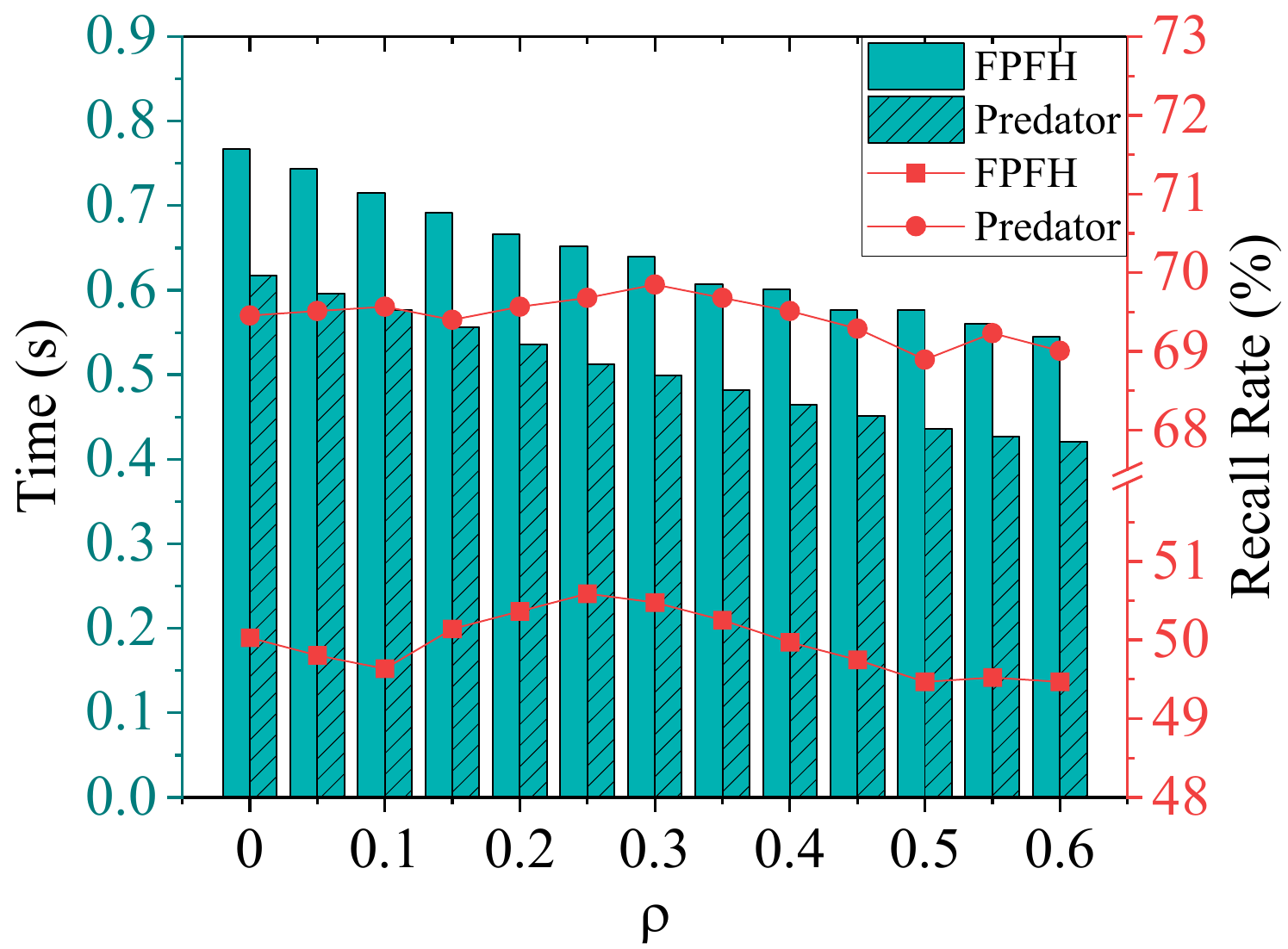}}
  \caption{Ablation study on 3DMatch and 3DLoMatch datasets. (a) 3DMatch with top-$k$ rotation axes; (b) 3DLoMatch with top-$k$ rotation axes; (c) 3DMatch with convergence ratio $\rho$; (d) 3DLoMatch with convergence ratio $\rho$.}
  \label{fig_topkrho}
\end{figure}

We conduct a longitudinal study on 3DMatch/3DLoMatch datasets to analyze top-$k$ rotation axes and convergence ratio $\rho$.
Parameter sweeps vary top-$k$ from 1 to 12 at fixed $\rho=0.25$ and $\rho$ from 0 to 0.6 at fixed $k=12$, as shown in Fig.~\ref{fig_topkrho}.

Both parameters significantly affect computation time.
Increasing $k$ yields sublinear time growth while improving RR.
FPFH stabilizes at $k>6$ on both datasets, FCGF improves until $k=11$, and Predator stabilizes at $k>4$.
This is because the selected k rotation axes between stage I and stage II mitigate the impact of local optima to boost the accuracy.
The RR peaks at specific $\rho$ values: 89.53\% RR with $\rho =$ 0.2 for FPFH~\cite{choyFully2019} on 3DMatch dataset; 50.59\% RR with $\rho =$ 0.25 for FPFH~\cite{choyFully2019} and 69.85\% RR with $\rho =$ 0.3 for Predator~\cite{huangPREDATOR2021} on 3DLoMatch dataset.
Especially for $\rho \in \left[0.05, 0.60 \right]$, FCGF~\cite{choyFully2019} maintains RR between 93.09\% and 93.22\%.
Although intuition suggests higher RR with smaller $\rho$, a larger $\rho$ pre-filters more potential outliers, improving both accuracy and efficiency.

\subsection{Outdoor KITTI Evaluation}

Following~\cite{choyFully2019}, we evaluate compared methods using 555 pairs from KITTI test dataset (sequences 8 to 10).
Traditional FPFH~\cite{rusuFast2009} and learning-based FCGF~\cite{choyFully2019} descriptor are employed.
Particularly, FCGF~\cite{choyFully2019} provides normalized and non-normalized 32-dimensional pre-trained features at 30~cm resolution, and we report the results of normalized version.
To ensure predictable execution time and manageable memory consumption, each point cloud is randomly downsampled to 8000 points for pairwise registration.
Under the Atlanta World assumption~\cite{straubManhattan2018} that is already applied to LiDAR point cloud~\cite{limSingle2022},
we constrain rotation axis search to the cube face with normal of z-axis, and set top-$k$ to 1.
The preprocessing distance factor $d_\mathrm{f}$ of feature matching is set to 0.05 and the number of kNN searched features $k_\mathrm{f}$ is set to 10 for both FPFH~\cite{rusuFast2009} and FCGF~\cite{choyFully2019} descriptors.
The registration is successful if RE $< 5^{\circ}$ and TE $< 2$~m.

\begin{table}[!htbp]
  \caption{Quantitive Results on KITTI Dataset.\label{tab:KITTI}}
  \centering
  \begin{tabular}{c|ccc|c}
    \hline
                                         & \multicolumn{3}{c|}{FPFH (traditional)}    &                                                                          \\
                                         & RR(\%)$\uparrow$                           & RE($^{\circ}$)$\downarrow$ & TE(cm)$\downarrow$ & Time(s)$\downarrow$    \\
    \hline
    RANSAC-1M\cite{fischlerRandom1981}   & 83.24                                      & 0.76                       & 24.26              & \underline{0.62}       \\
    RANSAC-4M\cite{fischlerRandom1981}   & 94.59                                      & 0.55                       & 20.41              & 2.46                   \\
    TCF\cite{shiRANSAC2024}              & 99.82                                      & 0.50                       & 11.10              & 5.19                   \\
    TEASER++\cite{yangTEASER2021}        & 99.64                                      & 0.37                       & 9.55               & 3.33                   \\
    TEASER++(w/ CC)\cite{yangTEASER2021} & 98.56                                      & 0.49                       & 11.35              & \textcolor{gray}{0.03} \\
    SC2-PCR++\cite{chenSC2PCR2023}       & \textbf{100.00}                            & \underline{0.32}           & \underline{7.19}   & 42.27                  \\
    MAC\cite{zhang3D2023}                & 99.28                                      & \underline{0.32}           & 8.16               & 6.43                   \\
    TEAR\cite{huangScalable2024}         & 94.23                                      & 0.40                       & 10.78              & 1.83                   \\
    HERE\cite{huangEfficient2024}        & 99.46                                      & 0.31                       & 7.69               & 1.12                   \\
    TR-DE\cite{chenDeterministic2022}    & 98.02                                      & 0.46                       & 11.20              & 1.61                   \\
    GMOR (Ours)                          & \textbf{100.00}                            & \textbf{0.24}              & \textbf{6.55}      & \textbf{0.19}          \\
    \hline
                                         & \multicolumn{3}{c|}{FCGF (learning-based)} &                                                                          \\
                                         & RR(\%)$\uparrow$                           & RE($^{\circ}$)$\downarrow$ & TE(cm)$\downarrow$ & Time(s)$\downarrow$    \\
    \hline
    RANSAC-1M\cite{fischlerRandom1981}   & 99.10                                      & 0.33                       & \underline{13.99}  & \underline{0.34}       \\
    RANSAC-4M\cite{fischlerRandom1981}   & 99.28                                      & 0.32                       & \textbf{13.83}     & 0.39                   \\
    TCF\cite{shiRANSAC2024}              & 99.28                                      & 0.31                       & 20.52              & 24.71                  \\
    TEASER++\cite{yangTEASER2021}        & \multicolumn{3}{c|}{Out-of-Memory}         &                                                                          \\
    TEASER++(w/ CC)\cite{yangTEASER2021} & 99.10                                      & 0.29                       & 22.08              & \textcolor{gray}{5.36} \\
    SC2-PCR++\cite{chenSC2PCR2023}       & \textbf{99.46}                             & 0.32                       & 20.61              & 205.06                 \\
    MAC\cite{zhang3D2023}                & 98.74                                      & 0.43                       & 17.16              & 5.43                   \\
    TEAR\cite{huangScalable2024}         & 96.04                                      & 0.36                       & 21.27              & 0.41                   \\
    HERE\cite{huangEfficient2024}        & 99.28                                      & \underline{0.27}           & 19.03              & 1.90                   \\
    TR-DE\cite{chenDeterministic2022}    & 99.10                                      & 0.34                       & 20.34              & 0.79                   \\
    GMOR (Ours)                          & \textbf{99.46}                             & \textbf{0.21}              & 16.14              & \textbf{0.12}          \\
    \hline
  \end{tabular}
\end{table}

\begin{figure*}[!htbp]
  \centering
  \footnotesize
  \begin{tabular}{cccccc}
    \normalsize Input pairs                   &
    \normalsize TCF\cite{shiRANSAC2024}       &
    \normalsize MAC\cite{zhang3D2023}         &
    \normalsize HERE\cite{huangEfficient2024} &
    \normalsize GMOR (Ours)                   &
    \normalsize Ground truth                                                  \\
    \makecell{
    \includegraphics[width=0.14\textwidth]{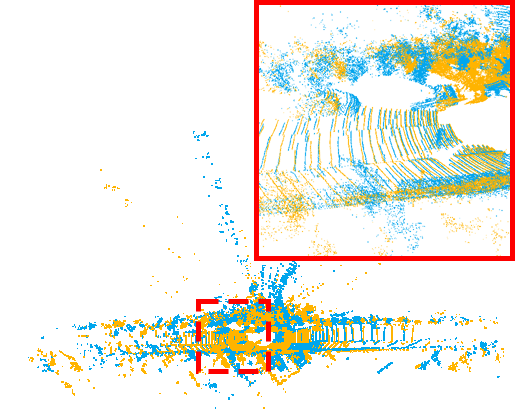}      \\
    Source: \#8000                                                            \\
      Target: \#8000
    }                                         &
    \makecell{
    \includegraphics[width=0.14\textwidth]{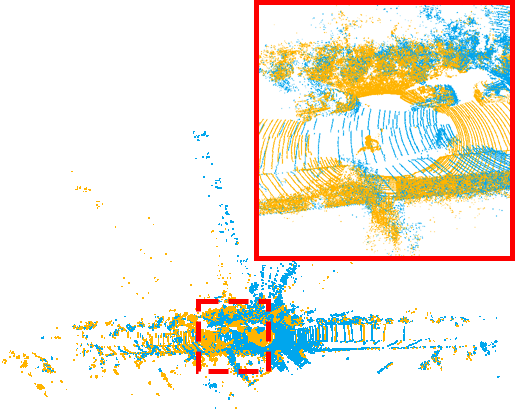}  \\
    RE = 0.14$^{\circ}$                                                       \\
    TE = 0.06m                                                                \\
    }                                         &
    \makecell{
    \includegraphics[width=0.14\textwidth]{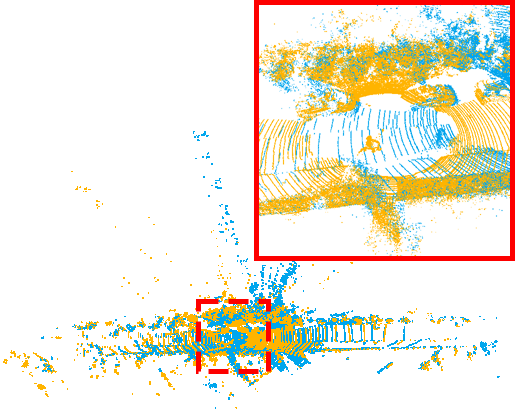}  \\
    RE = 0.68$^{\circ}$                                                       \\
    TE = 0.19m                                                                \\
    }                                         &
    \makecell{
    \includegraphics[width=0.14\textwidth]{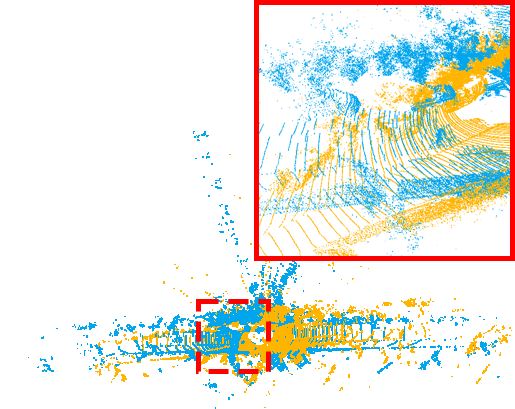} \\
    \textcolor{red}{RE = 8.05$^{\circ}$}                                      \\
    \textcolor{red}{ TE = 13.84m}                                             \\
    }                                         &
    \makecell{
    \includegraphics[width=0.14\textwidth]{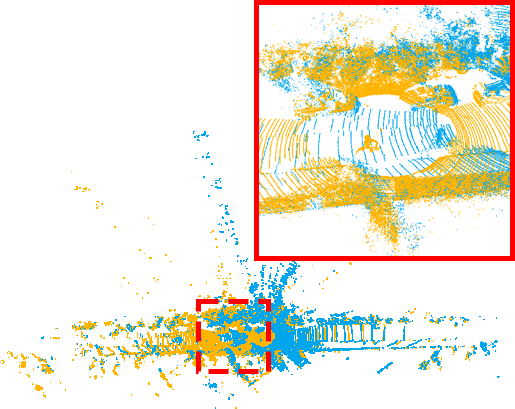} \\
    RE = 0.62$^{\circ}$                                                       \\
    TE = 0.13m                                                                \\
    }                                         &
    \makecell{
    \includegraphics[width=0.14\textwidth]{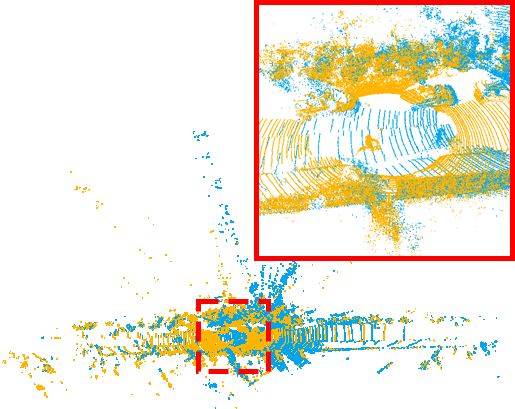}   \\
    \\
    \\
    }                                                                         \\
    \makecell{
    \includegraphics[width=0.14\textwidth]{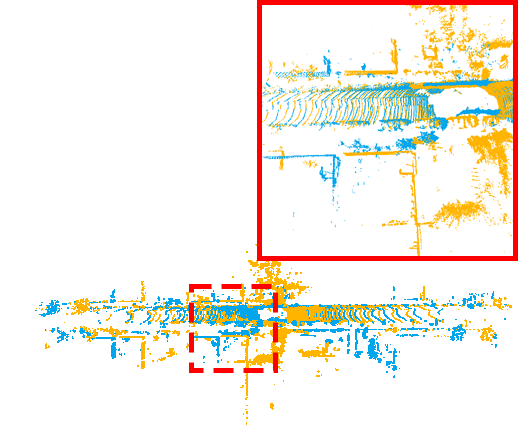}      \\
    Source: \#8000                                                            \\
      Target: \#8000
    }                                         &
    \makecell{
    \includegraphics[width=0.14\textwidth]{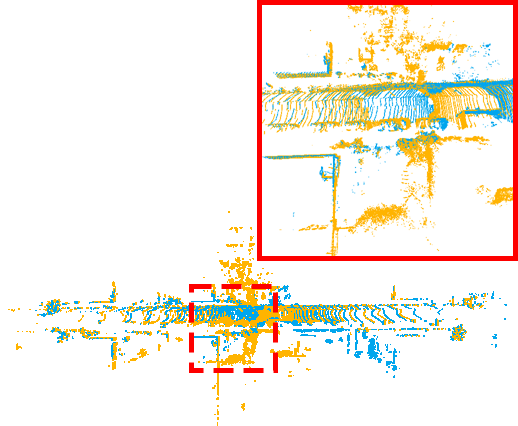}  \\
    RE = 1.15$^{\circ}$                                                       \\
    TE = 0.21m                                                                \\
    }                                         &
    \makecell{
    \includegraphics[width=0.14\textwidth]{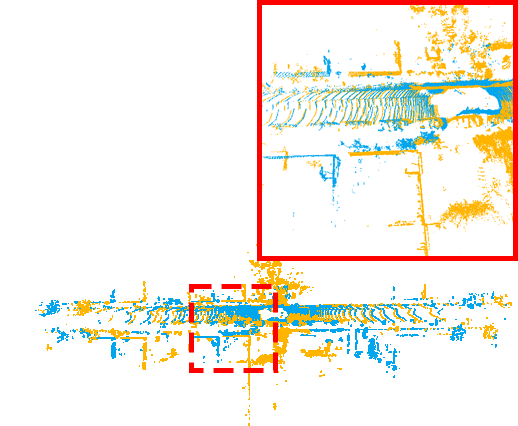}  \\
    RE = 2.31$^{\circ}$                                                       \\
    \textcolor{red}{TE = 10.41m}                                              \\
    }                                         &
    \makecell{
    \includegraphics[width=0.14\textwidth]{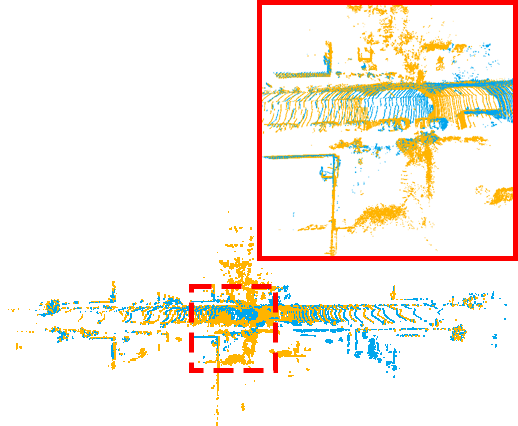} \\
    RE = 3.85$^{\circ}$                                                       \\
    TE = 0.24m                                                                \\
    }                                         &
    \makecell{
    \includegraphics[width=0.14\textwidth]{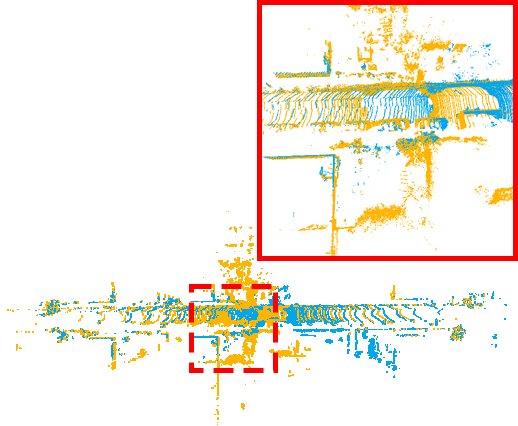} \\
    RE = 0.30$^{\circ}$                                                       \\
    TE = 0.07m                                                                \\
    }                                         &
    \makecell{
    \includegraphics[width=0.14\textwidth]{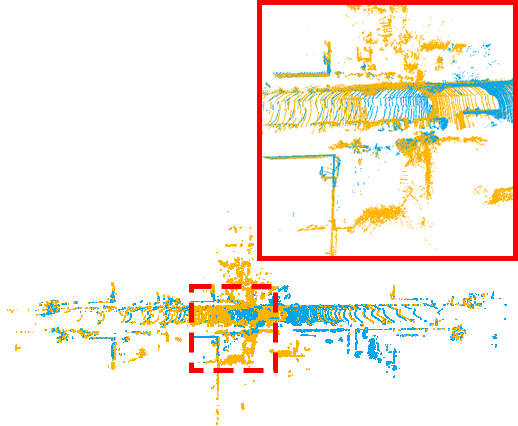}   \\
    \\
    \\
    }                                                                         \\
    \makecell{
    \includegraphics[width=0.14\textwidth]{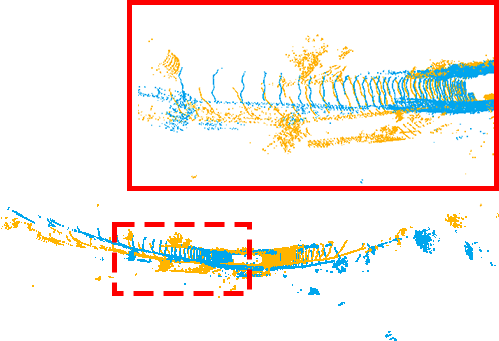}      \\
    Source: \#8000                                                            \\
      Target: \#8000
    }                                         &
    \makecell{
    \includegraphics[width=0.14\textwidth]{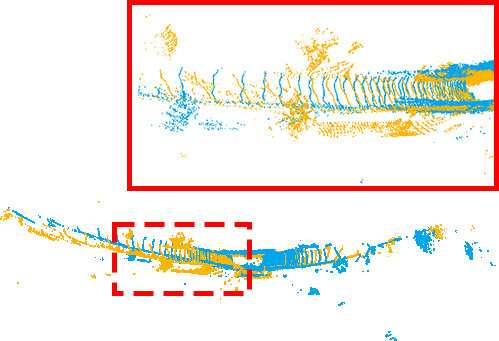}  \\
    \textcolor{red}{RE = 15.91$^{\circ}$}                                     \\
    \textcolor{red}{ TE = 10.66m}                                             \\
    }                                         &
    \makecell{
    \includegraphics[width=0.14\textwidth]{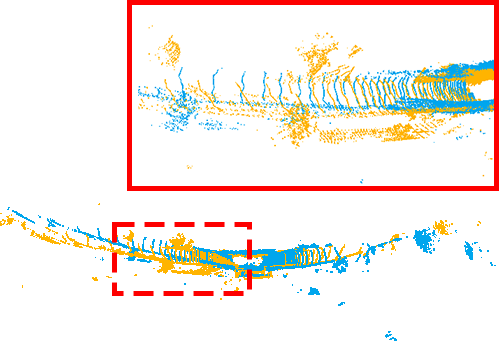}  \\
    \textcolor{red}{RE = 13.19$^{\circ}$}                                     \\
    \textcolor{red}{ TE = 10.80m}                                             \\
    }                                         &
    \makecell{
    \includegraphics[width=0.14\textwidth]{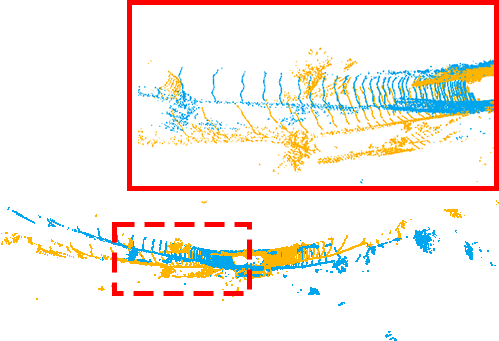} \\
    \textcolor{red}{RE = 9.54$^{\circ}$}                                      \\
    \textcolor{red}{ TE = 10.47m}                                             \\
    }                                         &
    \makecell{
    \includegraphics[width=0.14\textwidth]{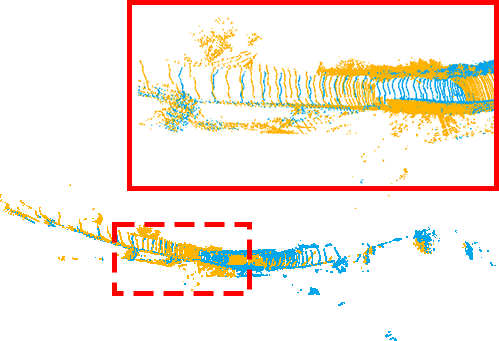} \\
    RE = 0.65$^{\circ}$                                                       \\
    TE = 0.77m                                                                \\
    }                                         &
    \makecell{
    \includegraphics[width=0.14\textwidth]{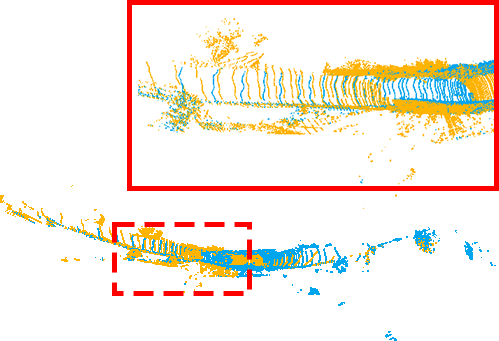}   \\
    \\
    \\
    }                                                                         \\
    \makecell{
    \includegraphics[width=0.14\textwidth]{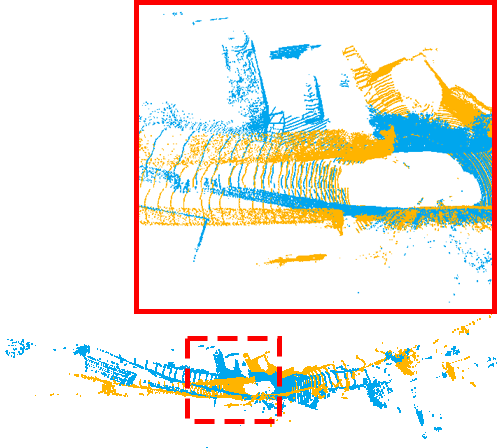}      \\
    Source: \#8000                                                            \\
      Target: \#8000
    }                                         &
    \makecell{
    \includegraphics[width=0.14\textwidth]{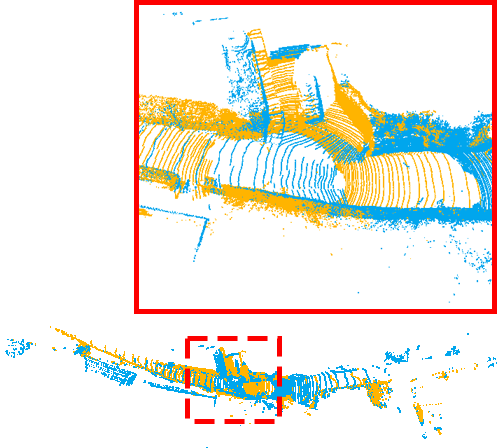}  \\
    RE = 0.80$^{\circ}$                                                       \\
    TE = 0.22m                                                                \\
    }                                         &
    \makecell{
    \includegraphics[width=0.14\textwidth]{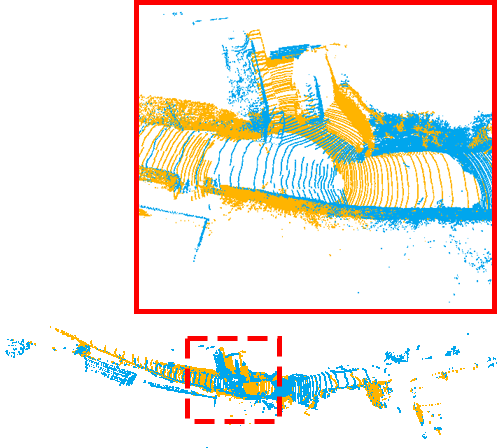}  \\
    RE = 1.24$^{\circ}$                                                       \\
    TE = 0.33m                                                                \\
    }                                         &
    \makecell{
    \includegraphics[width=0.14\textwidth]{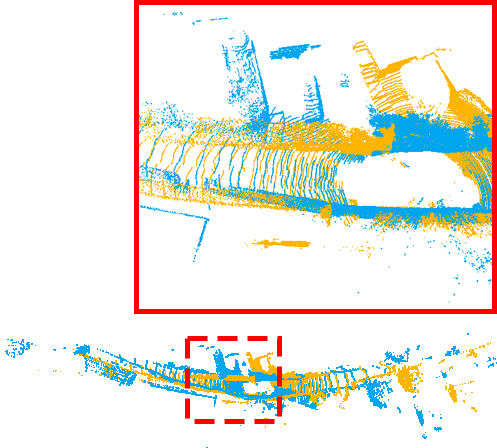} \\
    \textcolor{red}{RE = 9.04$^{\circ}$}                                      \\
    \textcolor{red}{ TE = 9.70m}                                              \\
    }                                         &
    \makecell{
    \includegraphics[width=0.14\textwidth]{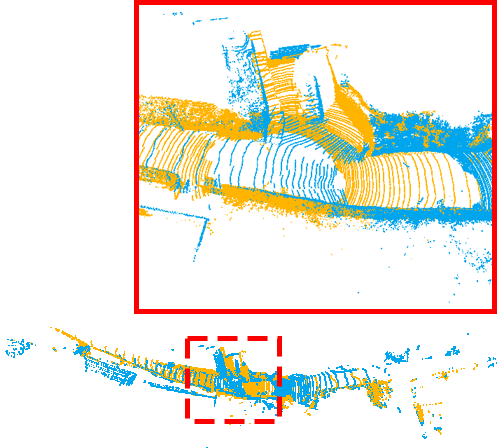} \\
    RE = 0.15$^{\circ}$                                                       \\
    TE = 0.19m                                                                \\
    }                                         &
    \makecell{
    \includegraphics[width=0.14\textwidth]{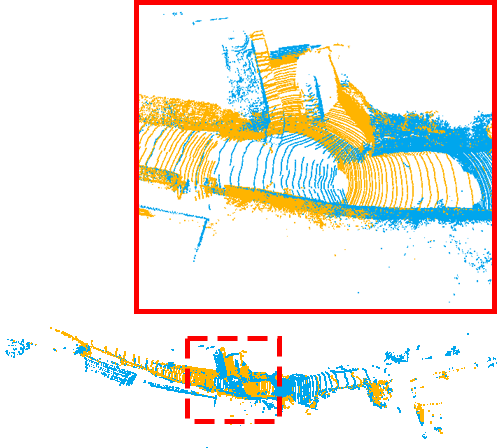}   \\
    \\
    \\
    }
  \end{tabular}
  \caption{Visualization of RANSAC-based TCF~\cite{shiRANSAC2024}, graph-based MAC~\cite{zhang3D2023}, BnB-based HERE~\cite{huangEfficient2024}, and our method on the outdoor KITTI dataset using FPFH descriptor~\cite{rusuFast2009}. ``\#'' is the number of downsampled points, and \textcolor{red}{red} text represents the failure case.}
  \label{fig_kitti}
\end{figure*}

Table~\ref{tab:KITTI} reports the evaluation results.
Our method achieves the highest RR and the fastest execution with both FPFH~\cite{rusuFast2009} and FCGF~\cite{rusuFast2009} descriptors on KITTI test dataset.
This efficiency advantage stems from the single top-$k$ rotation axis constraints and the restricted search range.
SC2-PCR++~\cite{chenSC2PCR2023} achieves competitive RR but incurs the highest time cost due to SOG resampling.
On account of the numerous candidate groups of FCGF~\cite{choyFully2019},
TCF~\cite{shiRANSAC2024} requires much more time than others to reach the 0.99 confidence.
Unexpectedly, RANSAC~\cite{fischlerRandom1981} achieves higher RR with FCGF than with FPFH, contrasting with others' lower RR using FCGF~\cite{choyFully2019}.
This occurs because the RANSAC~\cite{fischlerRandom1981} implemented in Open3D optimizes point cloud overlap, aligning with the triplet loss used in FCGF~\cite{choyFully2019}.
The visualization of registration results using FPFH descriptor~\cite{rusuFast2009} is shown in Fig.~\ref{fig_kitti},
demonstrating that our method achieves accurate registration even in extreme narrow-road LiDAR scenarios.

\subsection{4-DoF LiDAR Registration with Gravity Priors}
\label{sec:4dof_exp}

Under gravity priors, LiDAR registration reduces the registration to 4 DoF with a known rotation axis.
In this section, we evaluate the proposed method using both the 1+3 and 3+1 decomposition strategies, and compare it with representative 4-DoF registration methods on the KITTI dataset.
To ensure a fair comparison, the ground-truth rotation axis is provided to all methods.

\begin{table}[!htbp]
  \caption{4-DoF registration results on KITTI Dataset.\label{tab:KITTI_4DoF}}
  \centering
  \begin{tabular}{c|ccc|c}
    \hline
                                          & \multicolumn{3}{c|}{FPFH + rotation axis prior} &                                                                       \\
                                          & RR(\%)$\uparrow$                                & RE($^{\circ}$)$\downarrow$ & TE(cm)$\downarrow$ & Time(s)$\downarrow$ \\
    \hline
    Quatro~\cite{limQuatro2024}           & 99.64                                           & 0.13                       & 9.45               & 2.10                \\
    BnB~\cite{caiPractical2019}           & 98.92                                           & 0.19                       & 16.54              & 0.17                \\
    FMP+BnB~\cite{caiPractical2019}       & 99.64                                           & 0.27                       & 24.91              & \textbf{0.05}       \\
    Li et al.~\cite{liTransformation2024} & 94.59                                           & 0.24                       & 28.01              & 0.62                \\
    3D-BBS~\cite{aoki3DBBS2024}           & \underline{99.82}                               & 0.62                       & 101.06             & 8.38                \\
    GMOR (1+3)                            & 99.64                                           & \textbf{0.08}              & \underline{7.80}   & \textbf{0.05}       \\
    GMOR (3+1)                            & \textbf{100.00}                                 & \textbf{0.08}              & \textbf{7.18}      & 0.17                \\
    \hline
  \end{tabular}
\end{table}

\begin{figure}[!htbp]
  \centering
  \subfloat[]{\includegraphics[width=0.5\columnwidth]{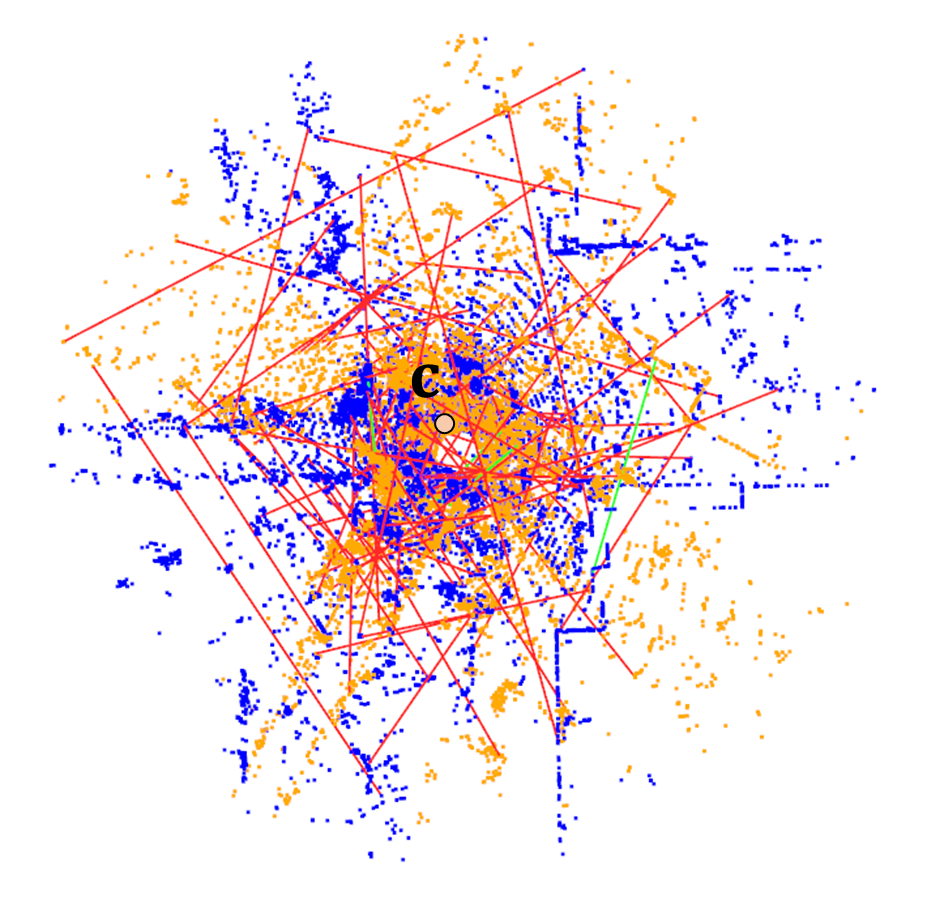}}
  \hfill
  \subfloat[]{\includegraphics[width=0.5\columnwidth]{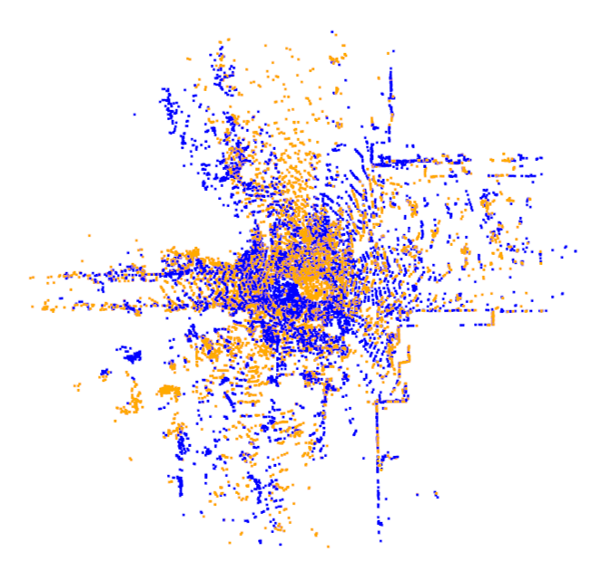}}
  \caption{Initial feature correspondences and registration result under the BEV view of LiDAR point clouds, where the source point cloud is shown in orange and the target in blue.
  (a) Initial feature correspondences, where inliers are shown as green lines and outliers as red lines; $\mathbf{c}$ denotes the ground-truth Chasles' rotation center.
  (b) Registration result obtained by our method.}
  \label{fig_kitti_corrs}
\end{figure}

As shown in Table~\ref{tab:KITTI_4DoF}, the proposed method achieves competitive performance under the 4-DoF setting.
Among the two decomposition strategies, the 3+1 variant consistently yields the most robust results, achieving the highest recall together with the lowest transformation error.
In comparison, the 1+3 decomposition firstly estimates the translation along the rotation axis, remaining efficient but exhibits slightly higher sensitivity to outliers due to the 1D geometric constraint.
As illustrated in Fig.~\ref{fig_kitti_corrs}, our method jointly enforces both rotation and translation on the 2D plane, enabling accurate inlier identification and yielding correct transformation.
These results empirically validate the analysis in Sec.~\ref{sec:4dof}, where resolving the 2D rigid transformation firstly imposes stronger geometric constraints and reduces ambiguity for the subsequent 1D translation estimation.

\begin{figure*}[!htbp]
  \centering
  \subfloat{}{
    \definecolor{color1}{HTML}{429B42}
    \definecolor{color2}{HTML}{7D5CA9}
    \definecolor{color3}{HTML}{FB7C0B}
    \definecolor{color4}{HTML}{706A0E}
    \definecolor{color5}{HTML}{0018BF}
    \definecolor{color6}{HTML}{A30156}
    \definecolor{color7}{HTML}{813D10}
    \definecolor{color8}{HTML}{454545}
    \begin{tikzpicture}[
        legend/.style={draw, thin, minimum size=7pt, inner sep=0pt},
        label/.style={font=\footnotesize}
      ]
      \foreach \name/\color/\x in {
          TCF\cite{shiRANSAC2024}/color1/0.05\textwidth,
          TEASER++\cite{yangTEASER2021}/color2/0.32\textwidth,
          SC2-PCR++\cite{chenSC2PCR2023}/color3/0.59\textwidth,
          MAC\cite{zhang3D2023}/color4/0.86\textwidth
        } {
          \node[legend, draw=\color, rectangle] at (\x,0.5) {\hspace*{2.0em}};
          \node[label, anchor=west] at (\x+0.02\textwidth,0.5) {\name};
        }
      \foreach \name/\color/\x in {
          TEAR\cite{huangScalable2024}/color5/0.05\textwidth,
          HERE\cite{huangEfficient2024}/color6/0.32\textwidth,
          TR-DE\cite{chenDeterministic2022}/color7/0.59\textwidth,
          GMOR (Ours)/color8/0.86\textwidth
        } {
          \node[legend, draw=\color, rectangle] at (\x,0) {\hspace*{2.0em}};
          \node[label, anchor=west] at (\x+0.02\textwidth,0) {\name};
        }
      \draw[thin, draw=black] (0.01\textwidth,0.7) rectangle (0.99\textwidth,-0.2);
    \end{tikzpicture}
  }
  \\
  \setcounter{subfigure}{0}
  \subfloat[]{\includegraphics[width=0.5\columnwidth]{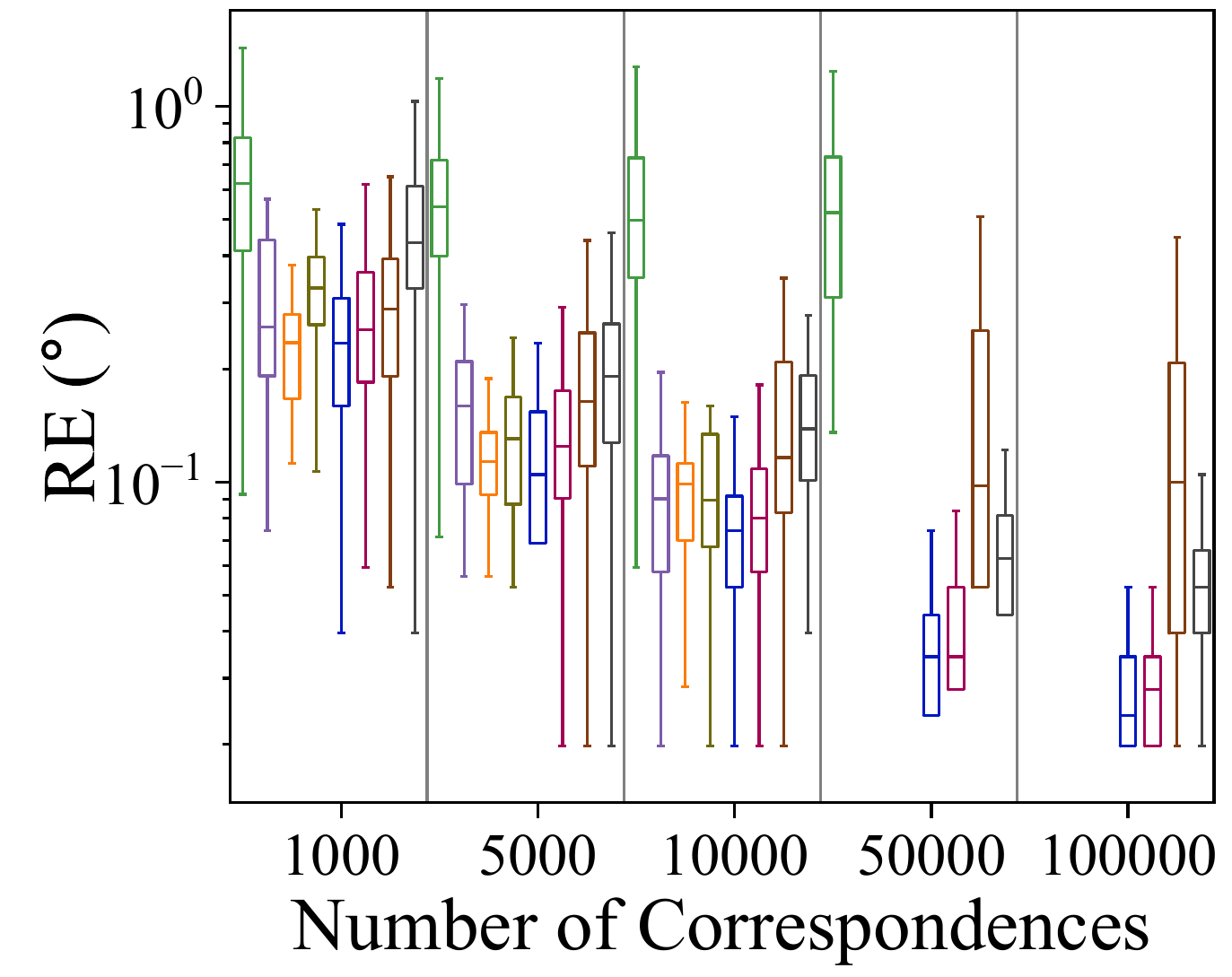}}
  \subfloat[]{\includegraphics[width=0.5\columnwidth]{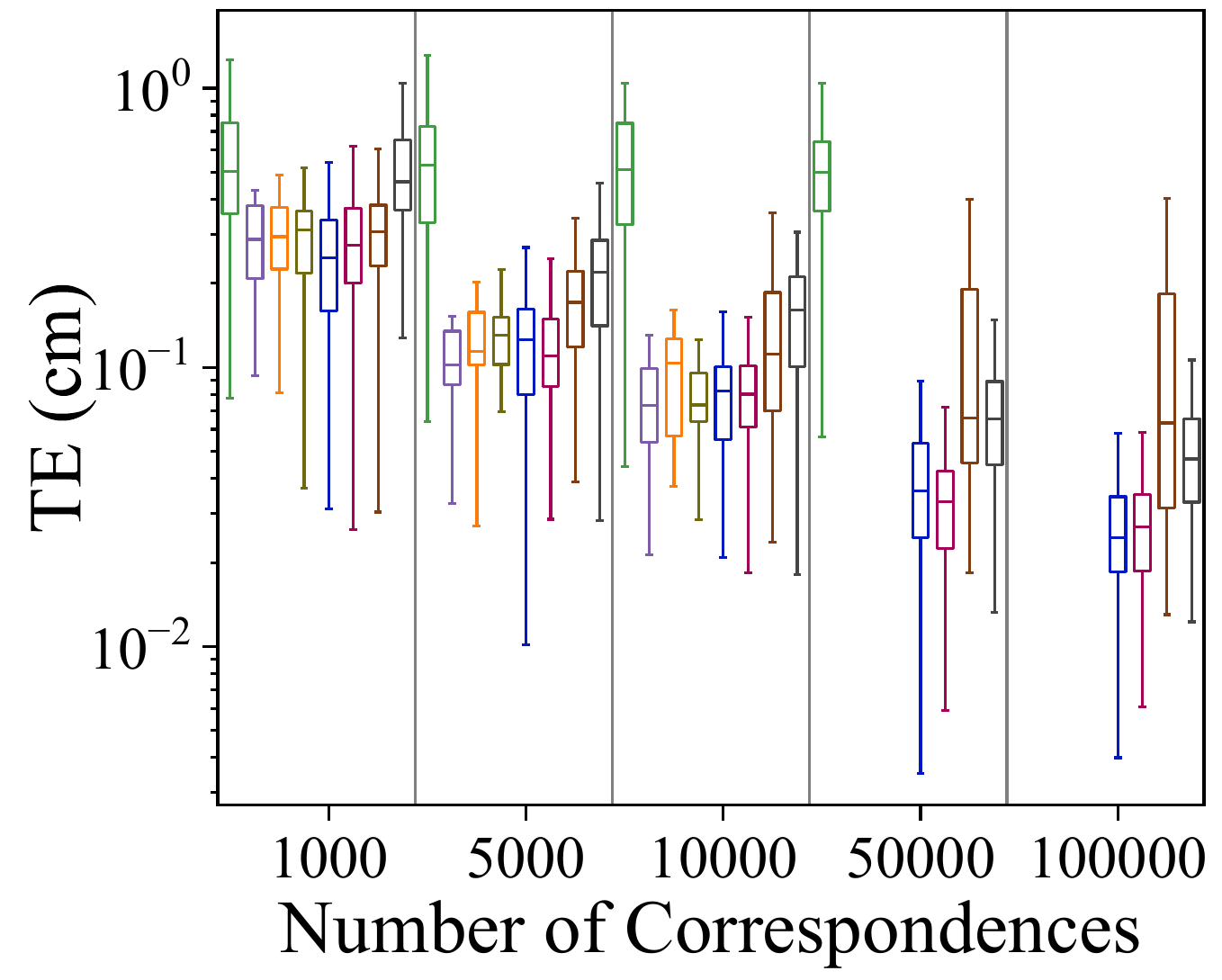}}
  \subfloat[]{\includegraphics[width=0.5\columnwidth]{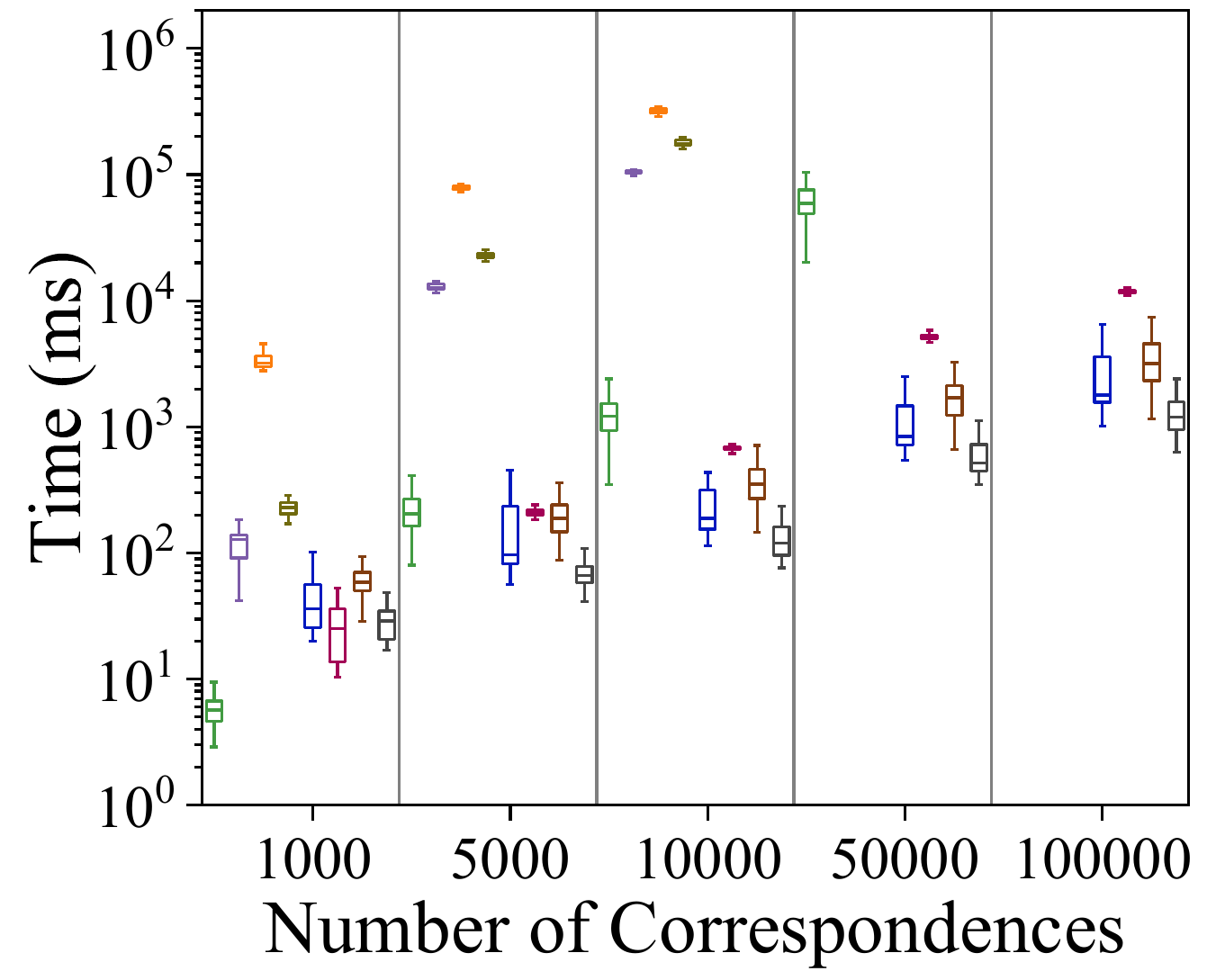}}
  \subfloat[]{\includegraphics[width=0.5\columnwidth]{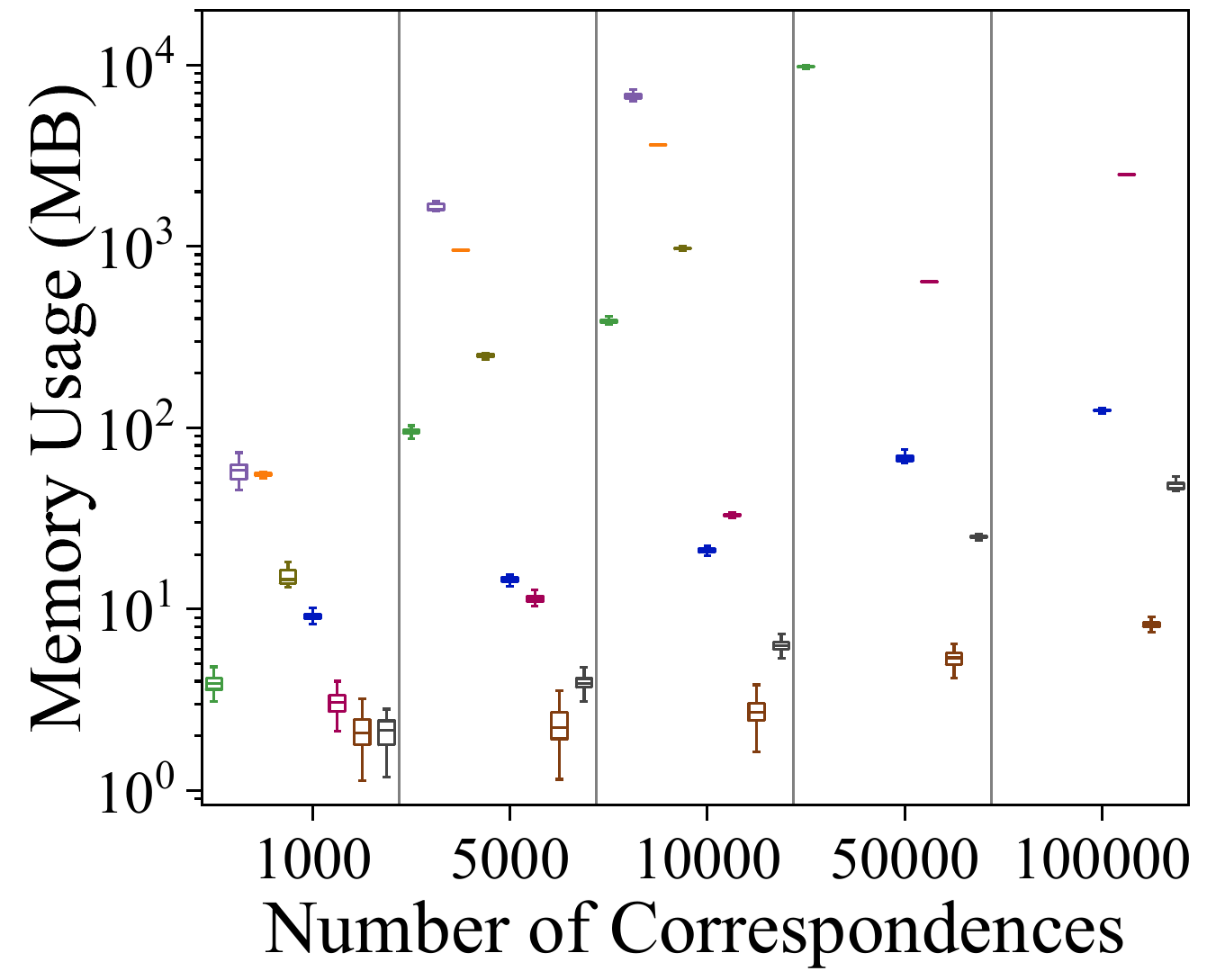}}
  \\
  \subfloat[]{\includegraphics[width=0.5\columnwidth]{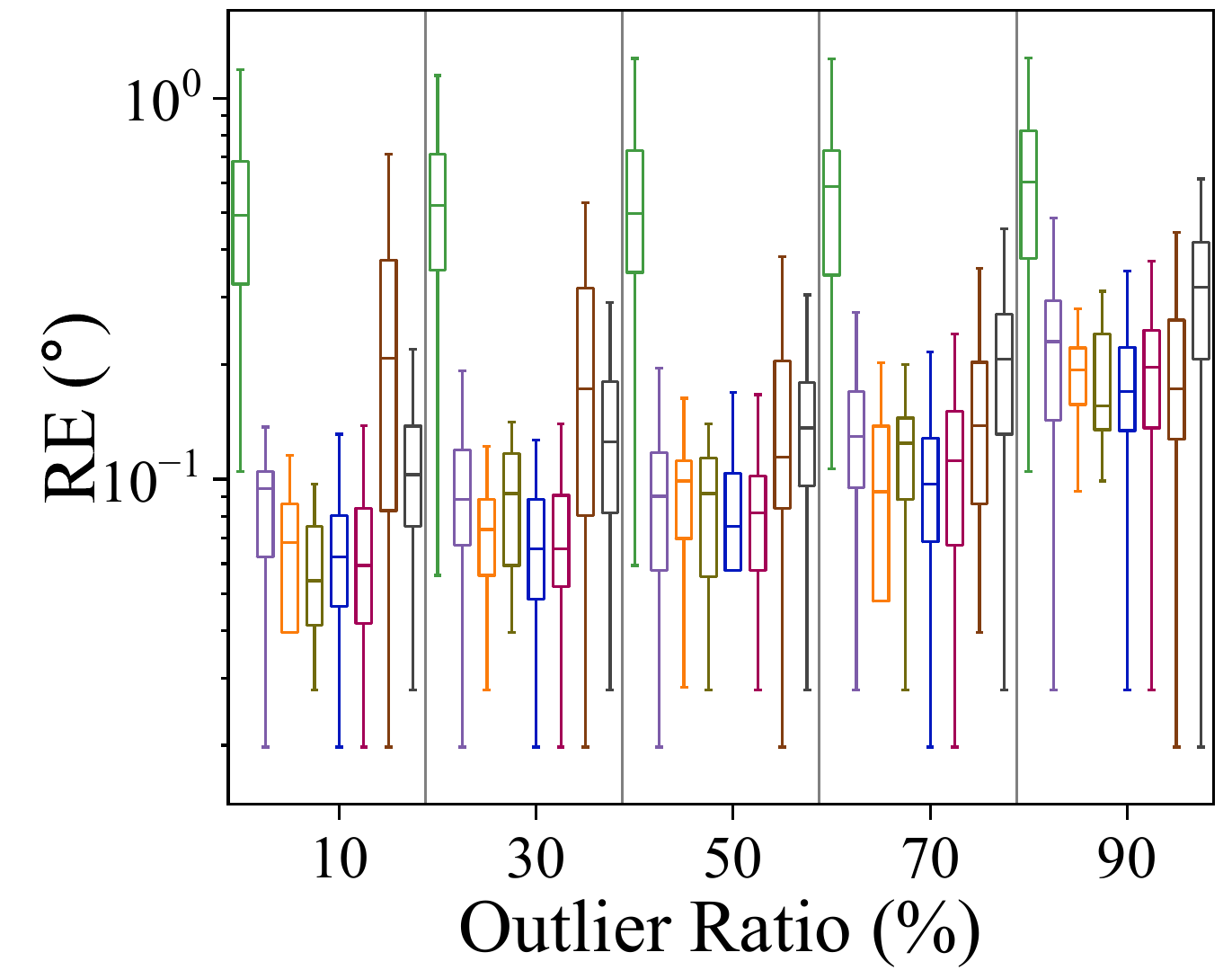}}
  \subfloat[]{\includegraphics[width=0.5\columnwidth]{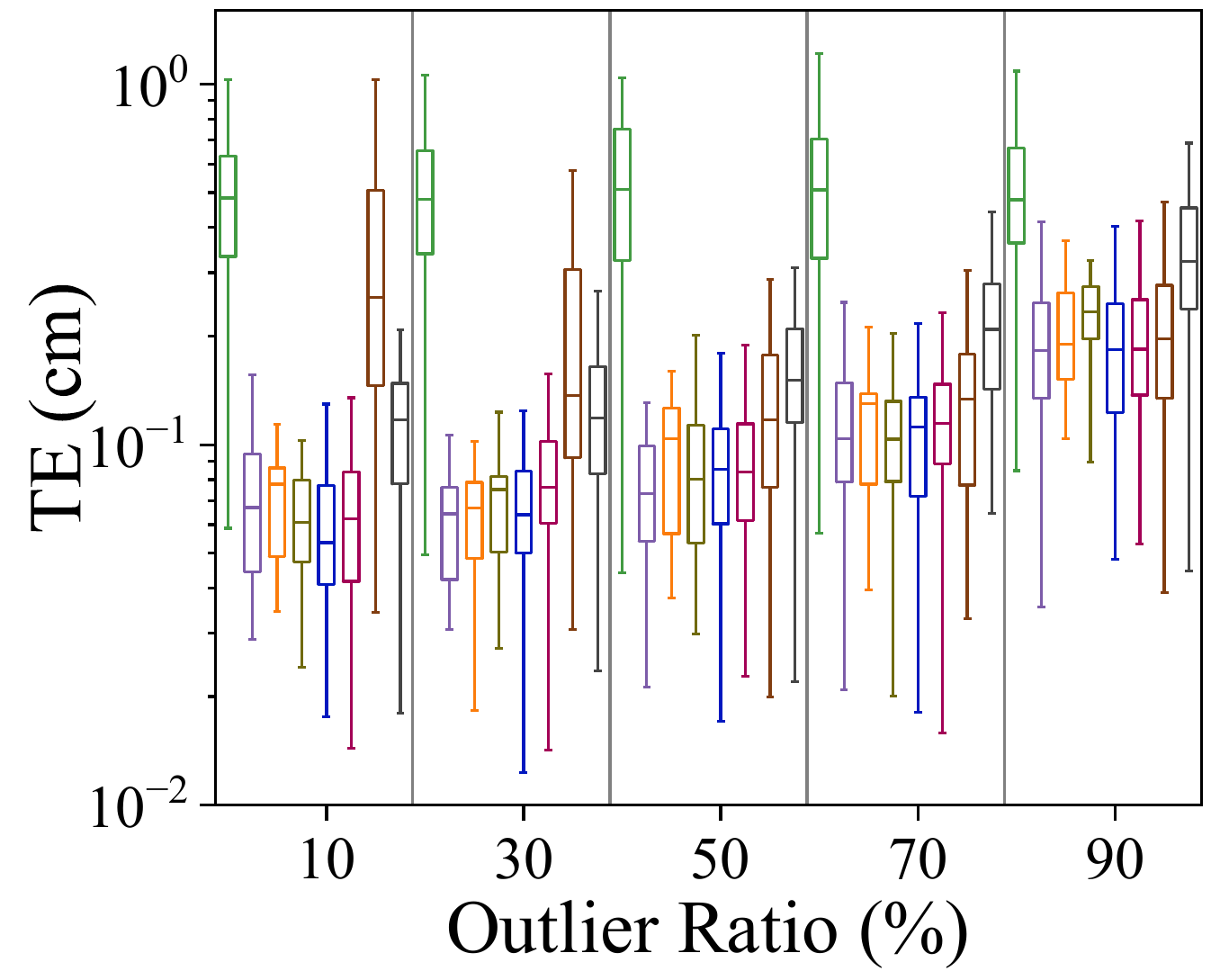}}
  \subfloat[]{\includegraphics[width=0.5\columnwidth]{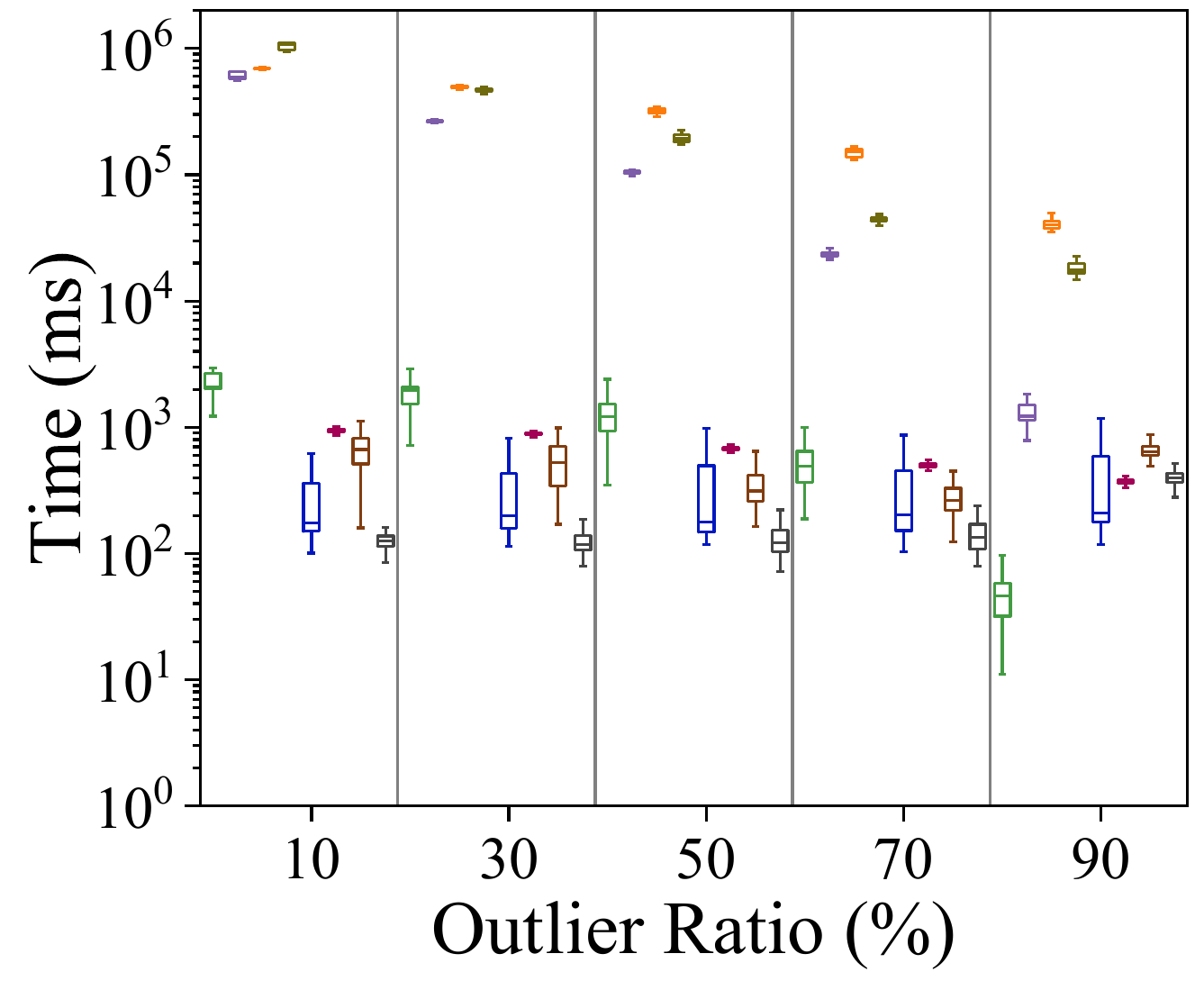}}
  \subfloat[]{\includegraphics[width=0.5\columnwidth]{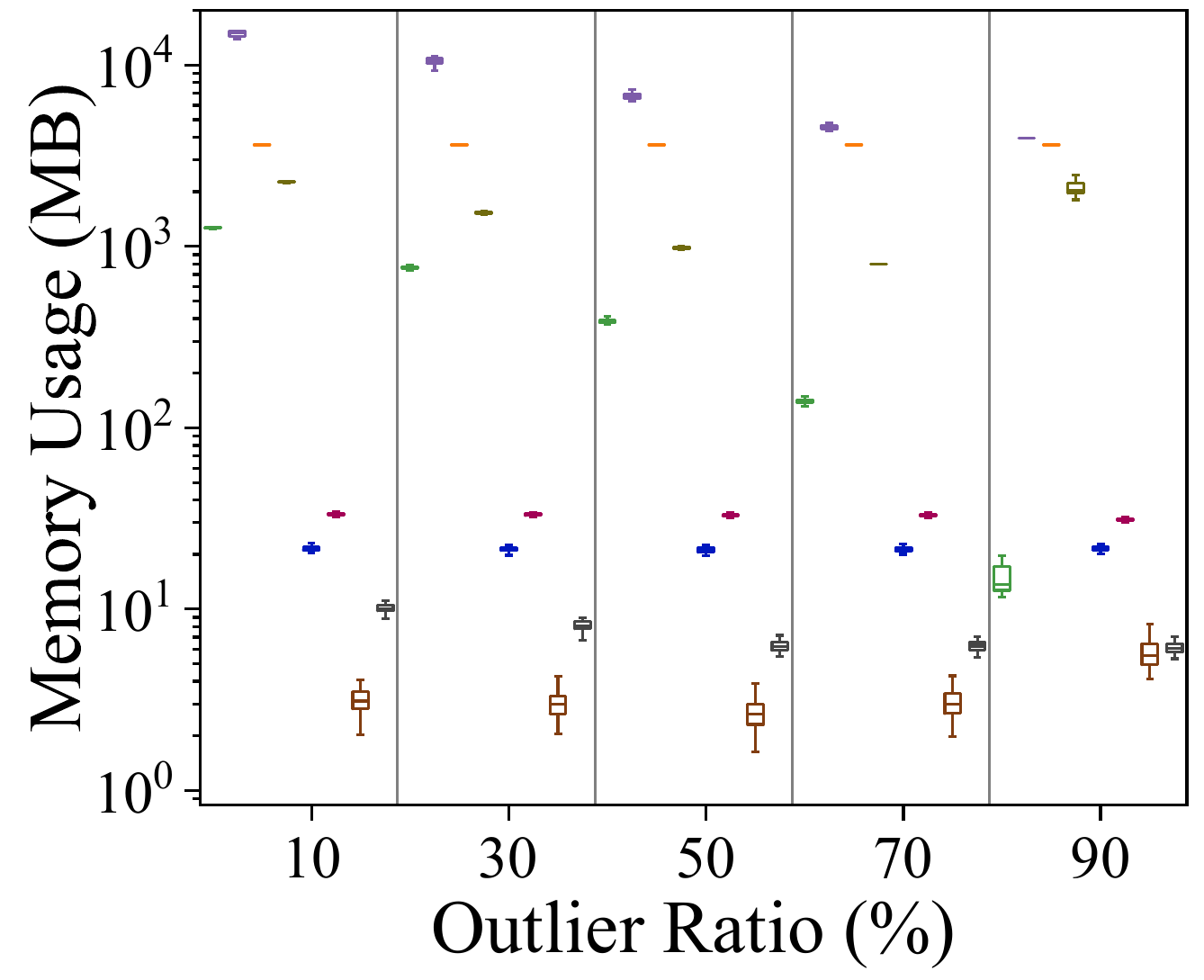}}
  \\
  \subfloat[]{\includegraphics[width=0.5\columnwidth]{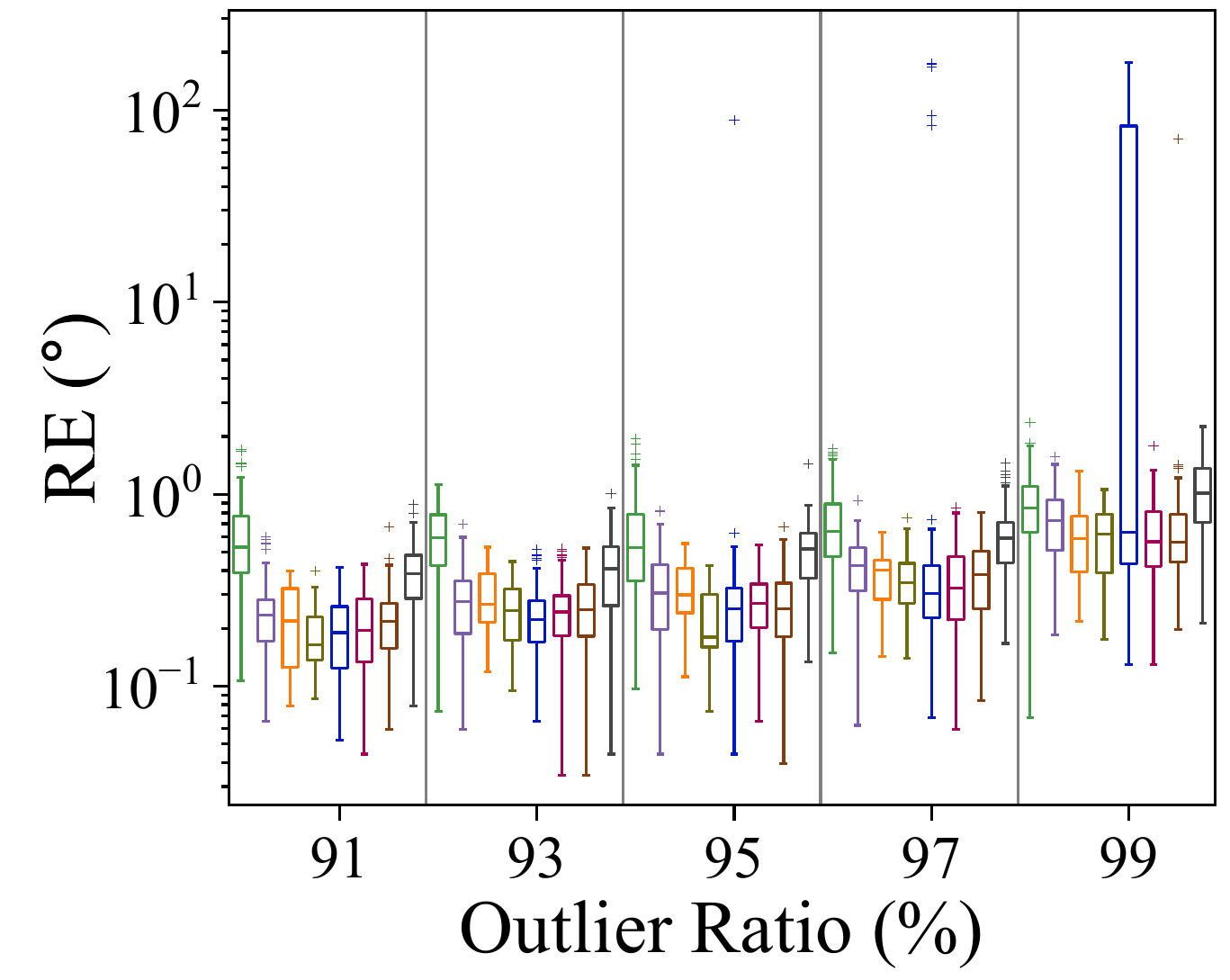}}
  \subfloat[]{\includegraphics[width=0.5\columnwidth]{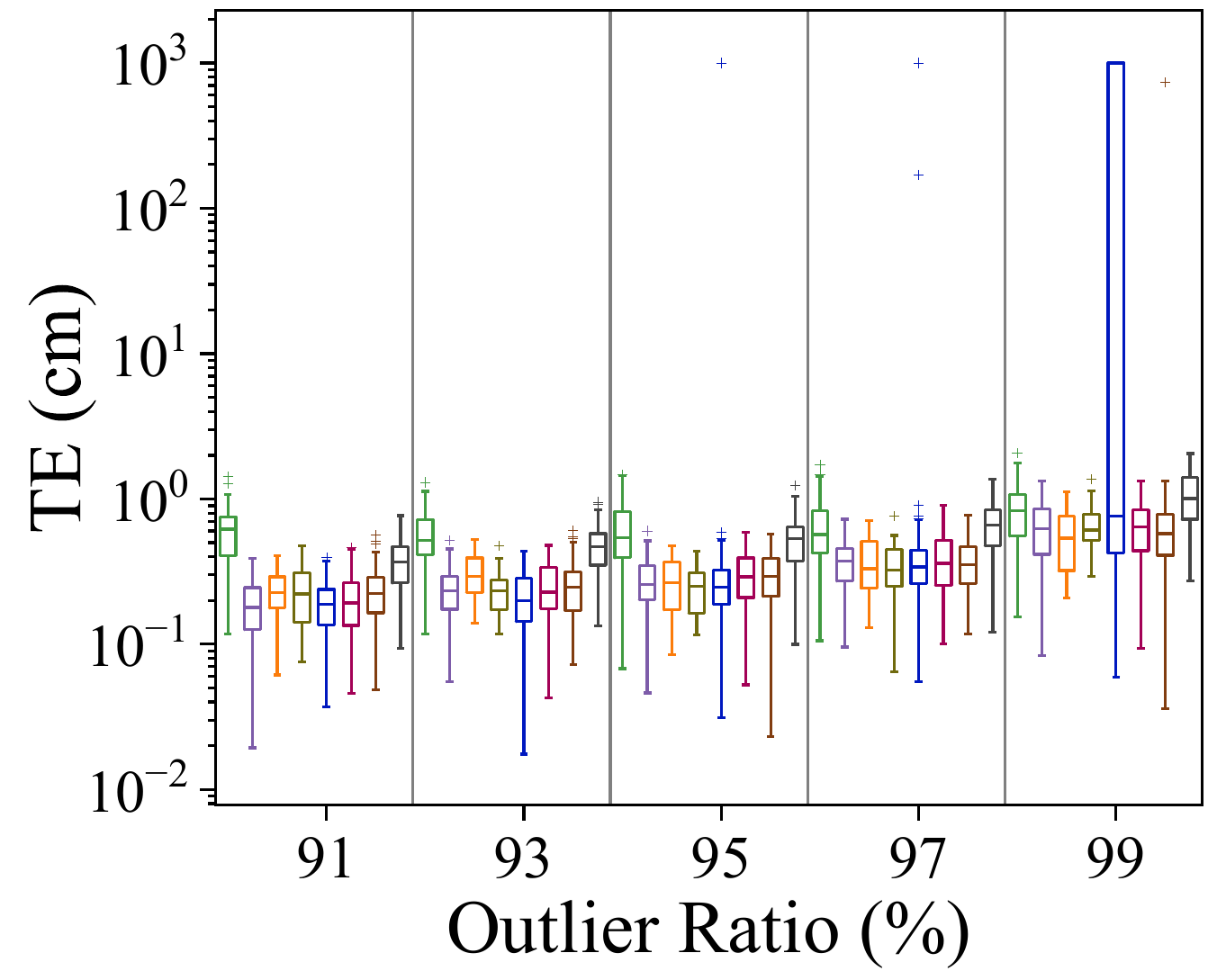}}
  \subfloat[]{\includegraphics[width=0.5\columnwidth]{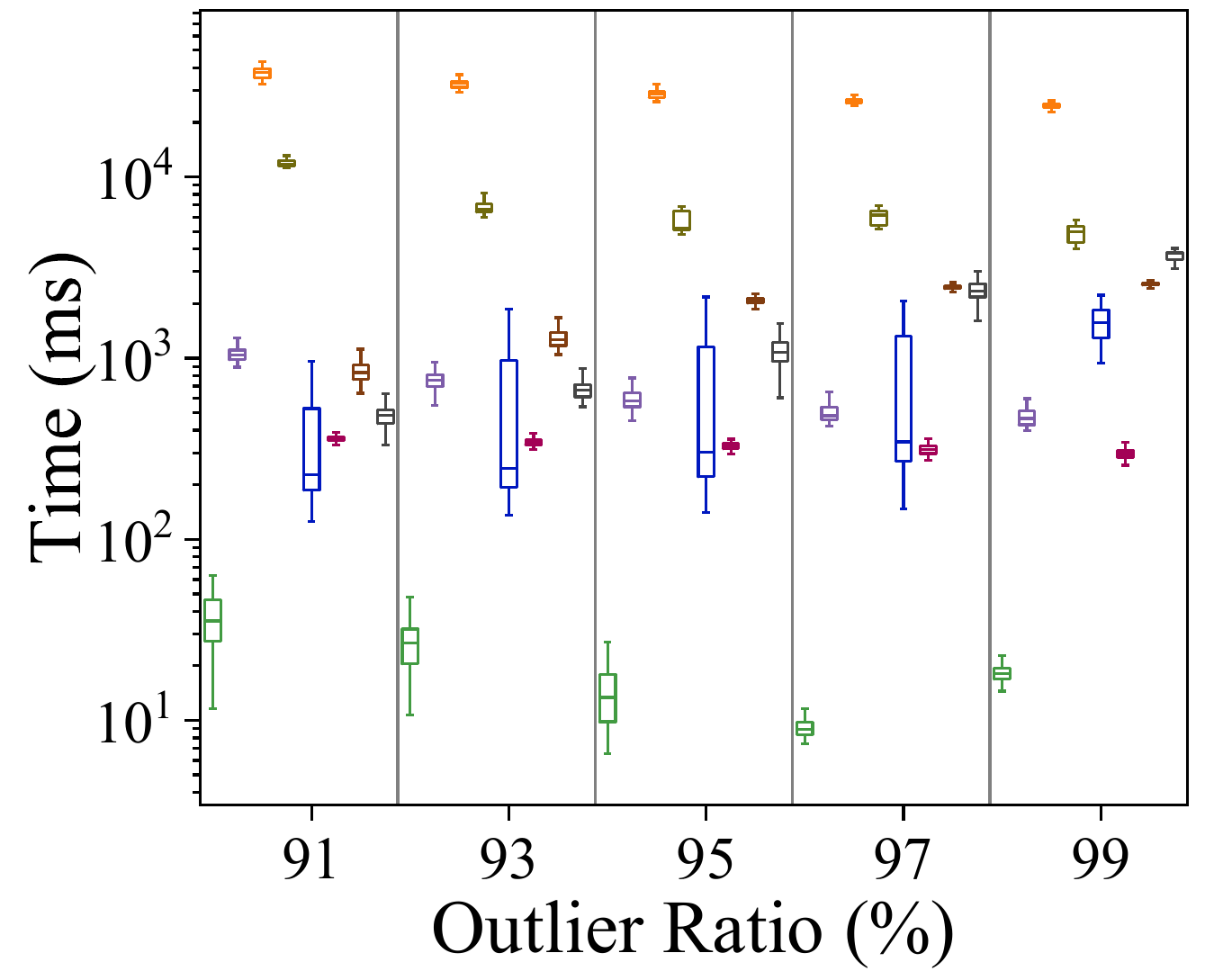}}
  \subfloat[]{\includegraphics[width=0.5\columnwidth]{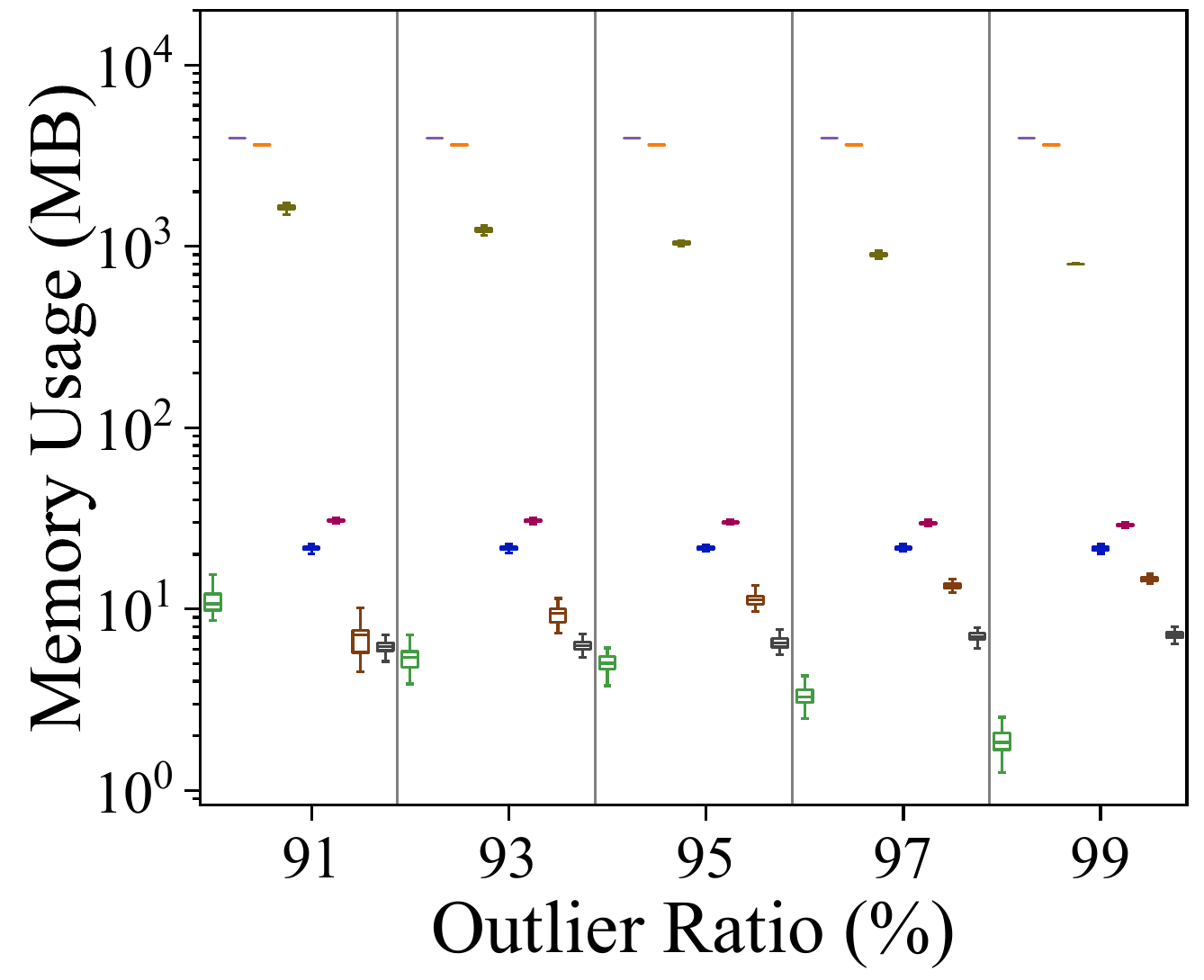}}
  \caption{Performance evaluation of compared SOTA methods with varying numbers of correspondences and outlier ratios on synthetic data.
    (a)-(d) Number of correspondences;
    (e)-(h) Outlier ratio;
    (i)-(l) Extreme outlier ratio.}
  \label{fig_syn}
\end{figure*}

Compared with existing 4-DoF registration methods, the proposed method achieves strong robustness while retaining computational efficiency.
Quatro~\cite{limQuatro2024} exhibits similar robustness with improved runtime efficiency compared to previous 6-DoF TEASER++~\cite{yangTEASER2021} on KITTI dataset.
BnB and FMP+BnB~\cite{caiPractical2019} perform translation-space search with interval stabbing of rotation angle after ISS-based correspondence filtering, while FMP further suppresses outliers, the reduced number of retained correspondences leads to higher efficiency but lower registration accuracy.
Li et al.~\cite{liTransformation2024} employ a three-stage pipeline combining interval stabbing, BnB searches and global voting, which is more sensitive to outlier distributions and results in lower recall.
Finally, the correspondence-free 3D-BBS~\cite{aoki3DBBS2024} achieves high recall at the cost of significantly increased runtime and TE, reflecting the trade-off between robustness and efficiency imposed by voxel map resolution.

\subsection{Performance Evaluation on Synthetic Data}

To validate the robustness, space and time complexity of our method versus SOTA methods,
we evaluate them with varying number of correspondences and outlier ratio.
The Armadillo model from the Stanford 3D Scanning Repository~\cite{curlessvolumetric1996} serves as synthetic data.
This model consists of 172974 points, and we downsample to $N \in \{1000, 5000, \dots, 100000\}$ correspondences.
Point clouds are normalized to a $[-1,1]^3$~m cube before adding zero-mean Gaussian noise and outliers.
For the outlier ratio, we randomly generate points in a 10~m diameter sphere, and replace the target points with a probability equal to outlier ratio.
Noise magnitude and inlier threshold $\xi$ (set to 0.1) correspond to~\eqref{eq_problem}.
We record the estimated RE and TE, elapsed time and memory usage of compared methods across number of correspondences and outlier ratios in synthetic experiments.

\subsubsection{Number of correspondences}

The number of correspondences increases from 1000 to 100000 with 50\% outliers across 100 repetitions.
The reported memory usage represents the memory allocated by the methods, excluding runtime overhead and data loading.
Besides, we modify the source code of HERE~\cite{huangEfficient2024} to dynamically allocate bitmask matrices, replacing static allocation for accurate measurement.

Fig.~\ref{fig_syn}(a)-(d) show the results across varying number of correspondences.
The estimated error decreases for all the methods as the number increases.
TCF~\cite{shiRANSAC2024}, TEASER++~\cite{yangTEASER2021}, MAC~\cite{zhang3D2023} and SC2-PCR++~\cite{chenSC2PCR2023} exhibit significantly higher time consumption than our method with large number of correspondences.
The memory usage increases linearly with the number of correspondences for TEAR~\cite{huangScalable2024}, TR-DE~\cite{chenDeterministic2022} and our method,
but polynomially for TCF~\cite{shiRANSAC2024}, TEASER++~\cite{yangTEASER2021}, MAC~\cite{zhang3D2023} and HERE~\cite{huangEfficient2024}.
TEASER++~\cite{yangTEASER2021}, SC2-PCR++~\cite{chenSC2PCR2023} and MAC~\cite{zhang3D2023} suffer from out-of-memory (OOM) errors beyond 50000 correspondences, while TCF~\cite{shiRANSAC2024} fails at 100000.
Our method maintains memory under 100~MB at 100000 correspondences.

\subsubsection{Outlier ratio}

We  evaluate performance across outlier ratios from 0.1 to 0.9 with 10000 correspondences, and the results are reported in Fig.~\ref{fig_syn}(e)-(h).
TCF~\cite{shiRANSAC2024} consumes much more time and memory when the outlier ratio is low, and it exhibits the highest error since IRLS converges slower than Ordinary Least Squares (OLS) under ideal Gaussian noises in our synthetic environment,
Similarly, the graph-based TEASER++~\cite{yangTEASER2021}, MAC~\cite{zhang3D2023}, SC2-PCR++~\cite{chenSC2PCR2023} incur substantial overhead due to dense graph construction and recursive search.
Our method maintains near-optimal time/memory efficiency while ensuring robustness across outlier ratios.

Furthermore, an independent experiment with extreme outlier ratio from 0.91 to 0.99 is also conducted, and the reported results are shown in Fig.~\ref{fig_syn}(i)-(l).
Not a number (NaN) rotation errors are set to 180$^\circ$, translation errors to 10~m.
Failure cases appear in TEAR~\cite{huangScalable2024} and TR-DE~\cite{chenDeterministic2022}, while other methods succeed.
The time advantage of our method narrows in BnB worst-case scenarios, while the graph-based methods gain efficiency.
Specifically, TCF~\cite{shiRANSAC2024} has the best performance of time and memory with extreme outlier ratio via one-point RANSAC outlier prefiltering.
However, this advantage is limited to ideal scenarios with synthetically added uniformly random outliers and Gaussian noise.
On previous real-world datasets, its computational efficiency degrades significantly.
Our method maintains constant memory usage and achieves the suboptimal performance.
In summary, our method demonstrates consistent robustness, efficiency, and low memory consumption across numbers of correspondences and outlier ratios, ensuring predictable time and memory growth in practical applications.

\section{Conclusion}
In this work, we propose a novel computational geometry-based point cloud registration method named GMOR, which combines BnB search strategy in parameter space and the solution of geometric maximum overlapping formulated as RMQ problem.
To ensure the polynomially bounded time complexity and linear space complexity with respect to the number of correspondences, we decompose the search into two stages: rotation axis search with interval stabbing for 1D RMQ, and rotation angle search with sweep line algorithm for 2D RMQ.
Experimental results on real and synthetic data demonstrate that our approach achieves efficient and robust registration with low computational overhead, showing potential in lightweight robotics and computer vision applications.

\appendices

\section{Derivation of the Reformulated Transformation}
\label{appd:chasles}

\subsection{Decomposition of residual vector}

Considering the residual vector $\mathbf{e}_i$ of the $i$-th correspondence in~\eqref{eq_problem}:
\begin{equation}
    \mathbf{e}_i = \mathbf{Q}_i - (\mathbf{R}\mathbf{P}_i + \mathbf{t}) .
    \label{eq_res_ei}
\end{equation}

Let $\mathbf{r}$ denote the unit rotation axis of $\mathbf{R}$ with $\|\mathbf{r}\| = 1$.
The translation vector $\mathbf{t}$ can be decomposed into components parallel and perpendicular to $\mathbf{r}$:
\begin{equation}
    d = \mathbf{r} \cdot \mathbf{t}, \quad
    \mathbf{t}_\perp = \mathbf{t} - d\mathbf{r}
    = (\mathbf{I} - \mathbf{r}\mathbf{r}^\top) \mathbf{t} .
    \label{eq_t_decomp}
\end{equation}

By Chasles' theorem, any rigid transformation can be reformulated as a screw motion:
a rotation about axis $\mathbf{r}$ passing through a point $\mathbf{C}$ and a translation $d\mathbf{r}$ along the same axis.
Equivalently, the rigid motion of point $\mathbf{C}$ is parallel to $\mathbf{r}$:
\begin{equation}
    \mathbf{R}\mathbf{C} + \mathbf{t} - \mathbf{C} = d\mathbf{r}
    \Longleftrightarrow
    \mathbf{t} + (\mathbf{R} - \mathbf{I}) \mathbf{C} = d\mathbf{r} .
    \label{eq_RC}
\end{equation}

Define $\mathbf{S}_i = \mathbf{Q}_i - \mathbf{P}_i$.
The residual in \eqref{eq_res_ei} can be rewritten as
\begin{equation}
    \begin{aligned}
        \mathbf{e}_i
         & = \mathbf{S}_i - (\mathbf{R} - \mathbf{I}) \mathbf{P}_i - \mathbf{t}                   \\
         & = \mathbf{S}_i - (\mathbf{R} - \mathbf{I}) (\mathbf{P}_i - \mathbf{C})
        - (\mathbf{t} + (\mathbf{R} - \mathbf{I}) \mathbf{C})                                     \\
         & = \mathbf{S}_i - (\mathbf{R} - \mathbf{I}) (\mathbf{P}_i - \mathbf{C}) - d\mathbf{r} .
    \end{aligned}
    \label{eq_res_screw}
\end{equation}

In the same way as $\mathbf{t}$, decompose $\mathbf{S}_i$ into components parallel and perpendicular to $\mathbf{r}$:
\begin{equation}
    \mathbf{S}_i
    = (\mathbf{r} \cdot \mathbf{S}_i)\mathbf{r}
    + (\mathbf{I} - \mathbf{r}\mathbf{r}^\top) \mathbf{S}_i .
\end{equation}

Substituting into~\eqref{eq_res_screw}, we obtain the residual vector $\mathbf{e}_\parallel$ parallel to $\mathbf{r}$ and $\mathbf{e}_\perp$ perpendicular to $\mathbf{r}$:
\begin{equation}
    \mathbf{e}_i
    =
    \underbrace{(\mathbf{r} \cdot \mathbf{S}_i - d)\,\mathbf{r}}_{\mathbf{e}_\parallel}
    +
    \underbrace{
        (\mathbf{I} - \mathbf{r}\mathbf{r}^\top)\mathbf{S}_i
        - (\mathbf{R} - \mathbf{I})(\mathbf{P}_i - \mathbf{C})
    }_{\mathbf{e}_\perp} .
    \label{eq_res_decomp}
\end{equation}

\subsection{Vector perpendicular to rotation axis}

Assume $\mathbf{R}$ represents a rotation of angle $\theta$ about axis $\mathbf{r}$.
Applying Rodrigues' formula,
\begin{equation}
    \mathbf{R} = \mathbf{I}\cos\theta
    + (1 - \cos\theta) \mathbf{r}\mathbf{r}^\top
    + \mathbf{r}_\times \sin\theta,
    \label{eq_rodrigues}
\end{equation}
where $\mathbf{r}_\times$ denotes the skew-symmetric matrix such that
$\mathbf{r}_\times \mathbf{v} = \mathbf{r} \times \mathbf{v}$.
We can get:
\begin{equation}
    \begin{aligned}
        \mathbf{R} \mathbf{r}                  & = \mathbf{r},                                                                                             \\
        \mathbf{R} - \mathbf{r}\mathbf{r}^\top & = \mathbf{R} - \mathbf{R}\mathbf{r}\mathbf{r}^\top = \mathbf{R} (\mathbf{I}- \mathbf{r}\mathbf{r}^\top) .
    \end{aligned}
\end{equation}

$\mathbf{e}_\perp$ in~\eqref{eq_res_decomp} can be rewritten as
\begin{equation}
    \begin{aligned}
        \mathbf{e}_\perp & = (\mathbf{I}-\mathbf{r}\mathbf{r}^\top)\mathbf{S}_i
        -(\mathbf{R}-\mathbf{I})(\mathbf{P}_i-\mathbf{C})                                                                                    \\
                         & = (\mathbf{I} - \mathbf{r}\mathbf{r}^\top) (\mathbf{Q}_i - \mathbf{C} - (\mathbf{P}_i - \mathbf{C}))
        - (\mathbf{R} - \mathbf{I}) (\mathbf{P}_i-\mathbf{C})                                                                                \\
                         & = (\mathbf{I} - \mathbf{r}\mathbf{r}^\top) (\mathbf{Q}_i - \mathbf{C})
        - (\mathbf{R} - \mathbf{r}\mathbf{r}^\top) (\mathbf{P}_i-\mathbf{C})                                                                 \\
                         & = (\mathbf{I} - \mathbf{r}\mathbf{r}^\top) (\mathbf{Q}_i - \mathbf{C})
        - \mathbf{R}(\mathbf{I} - \mathbf{r}\mathbf{r}^\top) (\mathbf{P}_i-\mathbf{C})                                                       \\
                         & = (\mathbf{I} - \mathbf{r}\mathbf{r}^\top) (\mathbf{Q}_i - \mathbf{C} - \mathbf{R} (\mathbf{P}_i - \mathbf{C})) .
    \end{aligned}
\end{equation}

Note that
\begin{equation}
    \begin{aligned}
        \mathbf{r}_\times^2                      & = \mathbf{r}\mathbf{r}^\top - \mathbf{I} ,                                                \\
        (\mathbf{I} - \mathbf{r}\mathbf{r}^\top) & = (\mathbf{I} - \mathbf{r}\mathbf{r}^\top)^\top(\mathbf{I} - \mathbf{r}\mathbf{r}^\top) .
    \end{aligned}
\end{equation}

For any vector $\mathbf{v}$, it holds that
\begin{equation}
    \begin{aligned}
        \left\|\mathbf{r} \times \mathbf{v}\right\|^2 & = \mathbf{v}^\top \mathbf{r}_\times^\top \mathbf{r}_\times \mathbf{v}   \\
                                                      & = \mathbf{v}^\top (-\mathbf{r}_\times^2)\mathbf{v}                      \\
                                                      & = \mathbf{v}^\top (\mathbf{I} - \mathbf{r}\mathbf{r}^\top)\mathbf{v}    \\
                                                      & = \left\|(\mathbf{I} - \mathbf{r}\mathbf{r}^\top)\mathbf{v}\right\|^2 .
    \end{aligned}
\end{equation}

Therefore,

\begin{equation}
    \left\|\mathbf{e}_\perp\right\|^2 = \left\|\mathbf{r} \times (\mathbf{Q}_i - \mathbf{C} - \mathbf{R} (\mathbf{P}_i - \mathbf{C}))\right\|^2
\end{equation}

Meanwhile, both $\mathbf{r}_\times \mathbf{v}$ and $(\mathbf{I} - \mathbf{r}\mathbf{r}^\top)\mathbf{v}$ are perpendicular to $\mathbf{r}$:
\begin{equation}
    \mathbf{r}^\top \mathbf{r}_\times \mathbf{v} = \mathbf{r}^\top (\mathbf{I} - \mathbf{r}\mathbf{r}^\top)\mathbf{v} = 0 .
\end{equation}

Since $\mathbf{e}_\parallel = (\mathbf{r} \cdot \mathbf{S_i} - d)\mathbf{r}$ is parallel to $\mathbf{r}$
and $\mathbf{e}_\perp \perp \mathbf{r}$, we have $\mathbf{e}_\parallel^\top \mathbf{e}_\perp=0$.
Thus, the residual takes the final form
\begin{equation}
    \begin{aligned}
        \left\|\mathbf{e}_i\right\|^2
         & = \left\| \mathbf{e}_\parallel \right\|^2 + \left\| \mathbf{e}_\perp \right\|^2                  \\
         & = \left\| (\mathbf{r} \cdot \mathbf{S}_i - d) \mathbf{r} \right\|^2 + \left\| \mathbf{r} \times(
        \mathbf{Q}_i - \mathbf{C} - \mathbf{R}(\mathbf{P}_i - \mathbf{C})) \right\|^2                       \\
         & = \left\|(\mathbf{r} \cdot \mathbf{S}_i - d)\mathbf{r} + \mathbf{r} \times(
        \mathbf{Q}_i - \mathbf{C} - \mathbf{R}(\mathbf{P}_i - \mathbf{C})) \right\|^2
    \end{aligned}
\end{equation}
which leads directly to~\eqref{eq_chasles} in the main text.

\begin{table*}[!htbp]
    \caption{$(x, y)$ parameterization of hemisphere.\label{tab:par_hs}}
    \centering
    \begin{tabular}{c|cc}
        \hline
        Projection    & Formulation                                                                & Parameter range                                                       \\
        \hline
        Spherical     & $\phi = x, \psi = y$                                                       & $x \in [-\pi, \pi], y \in [0, \frac{\pi}{2}]$                         \\
        Miller        & $\phi = x, \psi = \frac{5}{2}\arctan(\exp(\frac{4}{5}y)) - \frac{5}{8}\pi$ & $x \in [-\pi, \pi], y \in [0, \frac{5}{4}\ln(\tan(\frac{9\pi}{20}))]$ \\
        Stereographic & $\phi = x, \psi = \frac{\pi}{2} - 2\arctan(\frac{y}{2})$                   & $x \in [-\pi, \pi], y \in [0, 2]$                                     \\
        LCEA          & $\phi = x, \psi = \arccos(y)$                                              & $x \in [-\pi, \pi], y \in [0, 1]$                                     \\
        LAEA          & $\phi = x, \psi = \frac{\pi}{2} - 2\arcsin(\frac{y}{2})$                   & $x \in [-\pi, \pi], y \in [0, \sqrt{2}]$                              \\
        Cube (+Z)     & $\mathbf{r} =\frac{(x, y, 1)^\top}{\sqrt{x^2 + y^2 + 1}}$                  & $x, y \in [-1, 1]$                                                    \\
        \hline
    \end{tabular}
\end{table*}

\subsection{Rotation center}

Finally, since the point $\mathbf{C}$ on the rotation axis is not unique
(i.e., $\mathbf{C} + \alpha \mathbf{r}$ represents the same axis for any $\alpha \in \mathbb{R}$),
we arbitrarily choose $\mathbf{r}^\top \mathbf{C} = 0$.
\eqref{eq_RC} is rewritten to

\begin{equation}
    (\mathbf{I} - \mathbf{R})\mathbf{C} = (\mathbf{I} - \mathbf{r}\mathbf{r}^\top) \mathbf{t}.
    \label{eq_RC_rewritten}
\end{equation}

Using Rodrigues' formula in~\eqref{eq_rodrigues}, we have
\begin{equation}
    \begin{aligned}
        \mathbf{I} - \mathbf{R}
         & = (1 - \cos\theta)(\mathbf{I} - \mathbf{r}\mathbf{r}^\top) - \mathbf{r}_\times \sin\theta \\
         & = 2\sin\left(\frac{\theta}{2}\right)
        \left(
        \sin\left(\frac{\theta}{2}\right)(\mathbf{I} - \mathbf{r}\mathbf{r}^\top)
        - \mathbf{r}_\times \cos\left(\frac{\theta}{2}\right)
        \right).
    \end{aligned}
\end{equation}

Substituting into~\eqref{eq_RC_rewritten} together with
$\mathbf{r}^\top \mathbf{C} = 0$ yields
\begin{equation}
    2\sin\left(\frac{\theta}{2}\right)
    \left(
    \mathbf{I}\sin\left(\frac{\theta}{2}\right)
    - \mathbf{r}_\times \cos\left(\frac{\theta}{2}\right)
    \right)\mathbf{C}
    = (\mathbf{I}-\mathbf{r}\mathbf{r}^\top)\mathbf{t}.
\end{equation}

Left-multiplying both sides by $\mathbf{I}\sin\left(\frac{\theta}{2}\right) + \mathbf{r}_\times \cos\left(\frac{\theta}{2}\right)$ and using the identity $\mathbf{r}_\times^2 = \mathbf{r}\mathbf{r}^\top - \mathbf{I}$,
we obtain
\begin{equation}
    2\sin\left(\frac{\theta}{2}\right) \mathbf{C}
    = \left(\sin\left(\frac{\theta}{2}\right) (\mathbf{I} - \mathbf{r}\mathbf{r}^\top)
    + \mathbf{r}_\times \cos\left(\frac{\theta}{2}\right)
    \right) \mathbf{t}.
\end{equation}

Therefore,
\begin{equation}
    \mathbf{C}
    = \frac{1}{2} \left(\mathbf{I} - \mathbf{r}\mathbf{r}^\top
    + \mathbf{r}_\times \cot\left(\frac{\theta}{2}\right)
    \right) \mathbf{t},
\end{equation}
which corresponds to~\eqref{eq_chasles_params_d}.

\section{Comparison of Sphere Projections}
\label{appd:proj}

The unit rotation axis $\mathbf{r} = (r_x, r_y, r_z)^\top$ on the hemisphere ($r_z \in [0, 1]$) is defined as longitude $\phi \in [-\pi, \pi]$ and latitude $\psi \in [0, \frac{\pi}{2}]$:

\begin{equation}
    \begin{aligned}
        r_x & = \cos(\phi) \cos(\psi), \\
        r_y & = \sin(\phi) \cos(\psi), \\
        r_z & = \sin(\psi).
    \end{aligned}
    \label{eq_lonlat}
\end{equation}

Spherical, Miller (used in TR-DE~\cite{chenDeterministic2022}), Stereographic, Lambert Cylindrical Equal-Area (LCEA), and Lambert Azimuthal Equal-Area (LAEA) projections are compared with the proposed cube mapping.
Their formulations follow standard spherical projection models~\cite{snyderMap1987} and are summarized in Table~\ref{tab:par_hs}, where Cube (+Z) denotes the canonical face and the remaining faces are included via coordinate permutations.

Since LiDAR scans in the KITTI dataset~\cite{geigerAre2012} are predominantly acquired on ground vehicles, the distribution of rotation axes is highly biased and lacks sufficient diversity.
To better evaluate the effect of different hemisphere parameterizations under more general 3D rotations, we therefore conduct comparative experiments on the 3DMatch and 3DLoMatch datasets~\cite{zeng3DMatch2017}, which exhibit more diverse relative orientations.

\begin{table}[!htbp]
    \caption{Comparison of hemisphere projections on 3DMatch/3DLoMatch datasets.\label{tab:proj_result}}
    \centering
    \begin{tabular}{c|cc|cc}
        \hline
        Projection    & \multicolumn{2}{c|}{3DMatch} & \multicolumn{2}{c}{3DLoMatch}                                           \\
                      & RR(\%)$\uparrow$             & Time(s)$\downarrow$           & RR(\%)$\uparrow$  & Time(s)$\downarrow$ \\
        \hline
        Spherical     & 88.66                        & 0.61                          & 48.96             & \textbf{0.64}       \\
        Miller        & 88.85                        & \textbf{0.60}                 & 47.61             & 0.66                \\
        Stereographic & \underline{89.34}            & 0.62                          & 48.46             & 0.68                \\
        LCEA          & 89.09                        & 0.76                          & \underline{49.13} & 0.74                \\
        LAEA          & 88.48                        & \textbf{0.60}                 & 48.01             & \underline{0.65}    \\
        Cube          & \textbf{89.46}               & 0.62                          & \textbf{50.59}    & \underline{0.65}    \\
        \hline
    \end{tabular}
\end{table}

As shown in Table~\ref{tab:proj_result}, all hemisphere parameterizations achieve comparable recall on 3DMatch, indicating that the choice of projection has limited impact when sufficient overlap is available.
On the more challenging 3DLoMatch dataset with low overlap, cube mapping achieves the highest recall.
This advantage stems from its more uniform and bounded parameterization of the rotation axis space, which avoids severe distortion near the poles and leads to more balanced subbranch division under sparse correspondences.

\bibliographystyle{IEEEtran}
\bibliography{references}

\end{document}